%% file: main.tex
\documentclass[12pt]{article}
\usepackage{enumerate}
\usepackage{natbib}
\usepackage{multibib}

\newcites{Supp}{Supplementary References}
\usepackage{caption}
\usepackage{subcaption}
\usepackage{url} 

\usepackage{microtype}
\usepackage{graphicx}
\usepackage{multirow}
\usepackage{booktabs} % for professional tables

\usepackage{lineno}
\usepackage{algorithm2e}
\usepackage{float}
\usepackage{algpseudocode}
\RestyleAlgo{ruled}

\usepackage{enumitem}
\usepackage{lineno}
\usepackage{amsmath}
\usepackage{amssymb}
\usepackage{svg}
\usepackage{amsthm}
\usepackage{yhmath}
\usepackage{mathtools}
\DeclarePairedDelimiter\ceil{\lceil}{\rceil}

\theoremstyle{definition}

\newtheorem{definition}{Definition}
\newtheorem{proposition}{Proposition}
\newtheorem*{proposition*}{Proposition}
\newtheorem{lemma}{Lemma}
\newtheorem*{lemma*}{Lemma}
\newtheorem{remark}{Remark}

\newtheorem{theorem}{Theorem}
\newtheorem*{theorem*}{Theorem}
\newtheorem{corollary}{Corollary}
\newtheorem*{corollary*}{Corollary}

\newtheorem*{prf*}{Proof}

\usepackage{hyperref}

\newcommand{\bw}{{\bf w}}
\newcommand{\bX}{{\bf X}}
\newcommand{\bv}{{\bf v}}
\newcommand{\bW}{{\bf W}}
\newcommand{\bV}{{\bf V}}
\newcommand{\pr}{{\partial}}
\newcommand{\var}{\text{Var}}
\newcommand{\cov}{\text{Cov}}
\newcommand{\expect}{{\bf E}}
\newcommand{\bE}{{\bf E}}
\newcommand{\Qal}{Q^{\alpha}}

\newcommand{\zetaeps}{{\zeta^{\epsilon}_t}}

\newcommand{\Qalf}{Q^{\alpha_1}}
\newcommand{\Qals}{Q^{\alpha_2}}
\newcommand{\Zal}{Z^{\alpha}}

\newcommand{\Lhat}{\hat{L}}

\newcommand{\calS}{{\cal S}_t}

\newcommand{\Dt}{D^{(t)}}
\newcommand{\Dto}{D^{(t+1)}}
\newcommand{\Gto}{G^{(t+1)}}
\newcommand{\Lti}{L_i^{(t)}}

\newcommand{\Gt}{G^{(t)}}

\newcommand{\VW}{V_{\text{W}}}
\newcommand{\VD}{V_{\text{D}}}

\newcommand{\diag}{\text{diag}}

\newcommand{\blind}{0}

\date{}

\begin{document}

\def\spacingset#1{\renewcommand{\baselinestretch}%
{#1}\small\normalsize} \spacingset{1}

%%%%%%%%%%%%%%%%%%%%%%%%%%%%%%%%%%%%%%%%%%%%%%%%%%%%%%%%%%%%%%%%%%%%%%%%%%%%%%

\if0\blind
{
 % \title{\bf Enhanced GAN via Parallel Tempering on the Analysis of Gradients' Variance}
   \title{\bf Parallelly Tempered Generative Adversarial Nets: Toward Stabilized Gradients}
  \author{Jinwon Sohn \hspace{.2cm}\\
    Booth School of Business, University of Chicago\\
    and \\
    Qifan Song \\
    Department of Statistics, Purdue University}
  \maketitle
} \fi

\if1\blind
{
  \bigskip
  \bigskip
  \bigskip
  \begin{center}
    {\LARGE\bf Parallelly Tempered Generative Adversarial Nets: Toward Stabilized Gradients}
\end{center}
  \medskip
} \fi

\bigskip
\begin{abstract}

A generative adversarial network (GAN) has been a representative backbone model in generative artificial intelligence (AI) because of its powerful performance in capturing intricate data-generating processes. However, the GAN training is well-known for its notorious training instability, usually characterized by the occurrence of mode collapse. Through the lens of gradients' variance, this work particularly analyzes the training instability and inefficiency in the presence of mode collapse by linking it to multimodality in the target distribution. To ease the raised training issues from severe multimodality, we introduce a novel GAN training framework that leverages a series of tempered distributions produced via convex interpolation. With our newly developed GAN objective function, the generator can learn all the tempered distributions simultaneously, conceptually resonating with the parallel tempering in statistics. Our simulation studies demonstrate the superiority of our approach over existing popular training strategies in both image and tabular data synthesis. We theoretically analyze that such significant improvement can arise from reducing the variance of gradient estimates by using the tempered distributions. Finally, we further develop a variant of the proposed framework aimed at generating fair synthetic data which is one of the growing interests in the field of trustworthy AI.
\thispagestyle{empty}
\end{abstract}

\noindent%
\vfill

\newpage
\pagenumbering{arabic}

\spacingset{1.9} % DON'T change the spacing!
\section{Introduction}
\label{introduction}

The generative adversarial network (GAN) framework has emerged as a powerful and flexible tool for synthetic data generation in various domains. The GAN framework consists of two competing networks $D\in {\cal D}$ (i.e., the critic) and $G\in{\cal G}$ (i.e., the generator) where $\cal D$ and $\cal G$ have neural-net families. Let $X\in {\cal X}= \mathbb{R}^{d_X}$ and $Z\in {\cal Z}= \mathbb{R}^{d_Z}$ be random variables with density $p_X$ and $p_Z$ respectively. We denote by $G:\mathbb{R}^{d_Z} \rightarrow \mathbb{R}^{d_X}$ the generator that aims to transforms the reference variable $Z$ so that $G(Z)\overset{d}{=} X$. The essence of learning (or estimating) $G$ starts from approximating a divergence $d(p_X,p_{G(Z)})$ between two probability distributions by ${\cal D}$, i.e., $d_{\cal D}(p_X,p_{G(Z)})$ between two probability distributions characterized by ${\cal D}$. Then it finds the optimal $G$ that achieves $d_{\cal D}(p_X,p_{G(Z)})=0$. In accordance with the types of ${\cal D}$ and specification of $d_{\cal D}$, diverse probability metrics, such as the Jensen-Shannon divergence \citep[JSD,][]{good:etal:14}, the 1-Wasserstein distance \citep{arjo:etal:17}, $f$-divergence \citep{nowo:etal:16}, and so forth, are available. 

Despite its great potential as a high-quality data synthesizer, GAN has been known to be brutally unstable and easily fall into non-convergence because of its min-max (or adversarial) optimization structure. To resolve this training issue, a plethora of research has discovered better training tricks or network architectures, mostly based on empirical findings and specifically for image data synthesis \citep{jabb:etal:21}. From a more fundamental perspective, \cite{mesc:etal:18} suggested penalizing the average gradient norm of $D$, so that there would be no power to break the equilibrium between $D$ and $G$. \cite{zhou:etal:19} discussed that unstable training may arise from the flow of meaningless gradients from $D$ to $G$, which can be handled by enforcing ${\cal D}$ to be Lipschitz. To see more relevant studies, refer to \cite{roth:etal:17,gulr:etal:17}.

Tempering (or smoothing) $p_X$ has also been studied as an effective strategy to stabilize the GAN training mainly for the information-based divergence.  \cite{arjo:bott:17} theoretically justified that the JSD-based GAN training with some random noise being annealed $\epsilon \rightarrow 0$ during training, i.e., $\min_G d_{\cal D}(p_{X+\epsilon},p_{G(Z)+\epsilon})$, can improve optimization. They show that expanding supports of $X$ and $G(Z)$ by adding $\epsilon$ can remedy the gradient instability induced by the support mismatch between $X$ and $G(Z)$, thus enabling $D$ and $G$ to yield meaningful gradients. Based on this study, \cite{sajj:etal:18} attempted to anneal a functional noise of $X$ created by an auxiliary neural network during the GAN training with JSD. \cite{jenn:fava:19} approached $\min_G d_{\cal D}(p_{X+\epsilon},p_{G(Z)+\epsilon})$ with JSD where a noise $\epsilon$ is learned by an extra network instead of following an annealing schedule. 

While the existing studies benefit from resolving the unstable optimization issue associated with support mismatch when using JSD, this work discovers a stable training mechanism by tempering a multimodal $p_X$ via a convex interpolation scheme in GAN with the scaled 1-Wasserstein distance. Based on this novel discovery of smoothing $p_X$ in GAN training, we eventually devise an efficient and stable GAN framework by learning multiple levels of tempered distributions simultaneously, so it is called parallelly tempered generative adversarial networks (PTGAN). 

Section~\ref{sec:mech} explains the estimation mechanism of the GAN training and discusses a source of training instability, focusing on characterizing the gradients' variance of $D$. To our knowledge, this is the first work to theoretically analyze the GAN mechanism through the lens of the gradients' variance of $D$. Section~\ref{sec:tempered_dist} defines the convex interpolation that creates a tempered density of $p_X$ and analyzes how it alleviates the multimodality of $p_X$. Section~\ref{sec:parallel} specifically designs the proposed framework that incorporates the interpolation scheme without an annealing schedule. While our approach inherently possesses a bias-variance trade-off in updating the critic, our model achieves the nearly optimal minimax rate. Section~\ref{sec:simul} verifies that PTGAN significantly outperforms popular competing models in various benchmark data sets. Moreover, we apply PTGAN to a fair data-generation task to which trustworthy AI has paid great attention recently. In Section~\ref{sec:discussion}, we discuss the importance of GAN training and the merits of this study despite the recent development and success of diffusion-based generative modeling \citep{song:etal:20,ho:etal:20}.

\section{Estimation Mechanism of GAN}
\label{sec:mech}
\subsection{Neural distance}

To define a target distance metric $d_{\cal D}$, let's consider families of fully connected neural networks: ${\cal D}= \{D(x) = w^{\top}_d \kappa_{d-1}(W_{d-1}\kappa_{d-2}(\cdots W_1 x)) :  \bw=(W_1,\dots, W_{d-1}, w_d)\in \bW\}$ and ${\cal G}=\{G(z) = v^{\top}_g \psi_{g-1}(V_{g-1}\psi_{g-2}(\cdots V_1 z)) : \bv=(V_1,\dots, V_{g-1}, v_g)\in \bV\}$, 
where $w_d\in \mathbb{R}^{N^D_d\times 1}$, $v_g\in \mathbb{R}^{N^G_g\times d_X}$, $W_i \in \mathbb{R}^{N^D_{i+1}\times N^D_i}$ for $i=1,\dots,d-1$, $V_j \in \mathbb{R}^{N^G_{j+1}\times N^G_j}$ for $j=1,\dots,g-1$, $N^D_1 = d_X$,  and $N^G_1 = d_Z$; $\kappa_i$ and $\psi_j$ are element-wise non-linear activation functions. For simplicity, the above representations of ${\cal D}$ and ${\cal G}$ don't explicitly include bias nodes, as bias nodes can be induced by augmenting the network input by a constant. The neural distance \citep{aror:etal:17} is defined as follows.
\begin{definition}
\label{def:nd}
    Let $X\sim p_X$ and $Z\sim p_Z$. For a given monotone concave function $\phi:\mathbb{R} \rightarrow \mathbb{R}$, $G\in {\cal G}$, and a network class ${\cal D}$, $d_{\cal D}(p_X,p_{G(Z)})=\sup_{D \in {\cal D}} \{\bE [\phi(D(X))] + \bE [\phi(1-D(G(Z)))]\}$ is called a neural distance between $X$ and $G(Z)$.
\end{definition}

The specification of $\phi$ determines the type of discrepancy. This work considers $\phi(x)=x$, leading to the scaled 1-Wasserstein distance which is approximated by 
\begin{align}
    \label{eq:nd}
    d_{\cal D}(p_X,p_{G(Z)})=\sup_{D \in {\cal D}}  \{\bE[D(X)] - \bE[D(G(Z))]\},
\end{align}
via the Kantorovich-Rubinstein duality. Note that for $d_{\cal D}$ to approximate JSD, one specifies $\phi(x)=\log(x)$ and the sigmoid output of $D$. For \eqref{eq:nd}, the optimal critic function $D^*$ is the maximizer of $L(D,G)=\bE[D(X)]-\bE[D(G(Z))]$ such that $d_{\cal D}(p_{X},p_{G(Z)})=\expect [D^*(X)]-\expect [D^*(G(Z))]$, and the optimal generator $G^*$ is the minimizer of $d_{\cal D}(p_X,p_{G(Z)})$. For theoretical analysis in the remaining sections, the following assumptions are made:
\begin{enumerate}[label={(A\arabic*)}]
    \item Bounded parameter: $\bW = \bigotimes_{i=1}^{d-1}\{W_i \in \mathbb{R}^{N^D_{i+1} \times N^D_{i}} : \lVert W_i\rVert_F \leq M_w(i) \} \bigotimes \{w_d \in \mathbb{R}^{N^D_d \times 1} : \lVert w_d \rVert \leq M_w(d)\}$ and $\bV = \bigotimes_{j=1}^{g-1}\{V_j \in \mathbb{R}^{N^G_{j+1} \times N^G_{j}}: \lVert V_j\rVert_F \leq M_v(j) \} \bigotimes \{v_g \in \mathbb{R}^{N^G_g\times d_X}: \lVert v_g \rVert \leq M_v(g)\}$ with constants $M_w(\cdot)$ and $M_v(\cdot)$. Note $\lVert \cdot\rVert_F$ and $\lVert \cdot\rVert$ denote the Frobenius and the Euclidean norm, respectively. 

    \item Lipschitz activation: $\kappa_i$ and  $\psi_j$ are $K_{\kappa}(i)$- and $K_{\psi}(j)$-Lipschitz functions for all $i,j$, i.e., $\lVert \psi_j(x)-\psi_j(y)\rVert \leq K_{\psi}(j)\lVert x-y\rVert $ for any $x,y \in \mathbb{R}$, and also for $\kappa_i$ as well.
    
    \item Bounded support: ${\cal X}\subset\{\lVert x\rVert \leq B_X, \, x\in \mathbb{R}^{d_X}\}$ and ${\cal Z}\subset\{\lVert z\rVert \leq B_Z, \, z\in \mathbb{R}^{d_Z}\}$.
    
\end{enumerate}
These assumptions can be readily satisfied. For (A1), ad hoc training techniques such as weight clipping or weight normalization \citep{miya:etal:18} can be used. For (A2), popular activation functions such as ReLU, Leaky ReLU (lReLU), Tanh, etc., are 1-Lipschitz. Finally, For (A3), it is common to normalize the input in the deep learning literature, e.g., $-1 \leq X \leq 1$ (a.k.a. min-max normalization), and place a uniform distribution to $Z$.  

\subsection{Adversarial estimation}
\label{sec:adver_est}
\subsubsection{Iterative gradient-based estimation from minibatches}
\label{sec:grad_des}

Let $X_1,\dots,X_n$ and $Z_1,\dots,Z_m$ be i.i.d. samples from $p_X$ and $p_Z$, and define $\bX_{1:n} = \{X_1,\dots,X_n\}$. Since $p_X$ is unknown and $D$, $G$ may be non-linear, we rely on iterative gradient-based updates using the empirical estimator $\hat{L}_b(D,G)=\sum_{i=1}^{n_{b}}D(X_i)/n_b-\sum_{j=1}^{m_{b}} D(G(Z_j))/m_b$ of $L$
where $\{X_1,\dots,X_{n_b}\} \subset \bX_{1:n}$ is a minibatch with $\max\{n_b, m_b\} \ll \min\{n, m\}$. To estimate $D^*$ and $G^*$, we alternate gradient ascent/descent:
\begin{align}
\label{grad_desc}
    \bw^{(t+1)}=\bw^{(t)} + \gamma_D \dfrac{\partial}{\partial \bw}\hat{L}_b(D^{(t)},G^{(t)}),\quad 
    \bv^{(t+1)}=\bv^{(t)} - \gamma_G \dfrac{\partial}{\partial \bv}\hat{L}_b(D^{(t+1)},G^{(t)}),    
\end{align}
where $D^{(t)}$ and $G^{(t)}$ (or $\bw^{(t)}$, $\bv^{(t)}$) are $t$th iterates and $\gamma_D$, $\gamma_G$ are learning rates. This paper expresses $D^{(t)}$ and $\bw^{(t)}$, as well as $G^{(t)}$ and $\bv^{(t)}$ interchangeably when causing no confusion in the context. Also, we set $m_b=n_b$ for simplicity. 

\subsubsection{Estimation dynamics}

To describe the estimation behavior, we denote ${\cal S}$ as the high-density support region of $X$, so that the probability density of $p_X$ out of ${\cal S}$ is reasonably small. In the extreme case, we may directly consider that $\cal X$ is a disconnected compact set (with each disconnected component representing a mode with finite support of $p_X$) and $\cal S=\cal X$. To characterize the distribution modes learned by $\Gt$ (i.e., a subset of ${\cal S}$), we define the recovered support of $p_X$ by $\Gt$ as ${\cal S}_t = {\cal S} \cap \{\Gt(z), z\in \cal Z\}$, and $\calS^c={\cal S}\setminus{\cal S}_t$ the remaining support which is missing by $\Gt$. Unless ${\cal S}$ is fully recovered, we say that $\Gt$  is incomplete for $G^*$. For ease of later discussion, we denote $p_{{\cal S}_t}$ and $p_{{\cal S}^c_t}$ as normalized probability density functions of $p_X$ restricted on ${\cal S}_t$ and ${\cal S}^c_t$ respectively, meaning $p_{\calS}$ represents the recovered distribution components of the target $p_X$ by $\Gt$. 

When $\Gt$ has such incomplete support recovery, $\Dt \rightarrow \Dto$ is necessarily updated by either maximizing $\sum_{i=1}^{n_b} \Dt(X_i)/n_b$ or minimizing $\sum_{i=1}^{m_b} \Dt(\Gt(Z_i))/m_b$. This opposite directional optimization encourages $\Dto$ to assign higher values to $X_{i}$ and lower values to $\Gt(Z_i)$, so $\Dto$ appropriately identifies the discrepancy between the support of $\Gt(Z)$ and $\calS^c$ by its value. The desirable $\Dto$ then guides the direction and size of the generator's gradients for the update of $\Gt \rightarrow \Gto$ such that $\Gto$ accounts for more areas in ${\cal S}$. More specifically, for any $z \in {\cal Z}$ in the minibatch, its contribution to the gradient of $\bv^{(t)}$ with respect to $\Lhat_b$ is $\partial \Lhat_b/\partial \bv=-\left[\nabla_{g} \Dto(g)\right]_{g=\Gt(z)} \cdot \partial \Gt(z)/\partial \bv$. The gradient component $\nabla_{g} \Dto(g)$ instructs the update of $\Gt$ so that $\Gto$ will have a higher $\Dto$ values (likely move toward $\calS^c$ if $\Dto$ successfully characterizes the gap between $\Gt(Z)$ and $\calS^c$). In this regard, accurate estimation for the neural distance (or equivalently $\bw^{(t+1)}$) is of great importance. In an ideal situation, the generator eventually fully recovers ${\cal S}$ after a sufficient number of iterations $T$, i.e., ${\cal S}_T \approx {\cal S}$, reaching out the approximate equilibrium $L(D',G^{(T)}) \lesssim L(D^{(T)},G^{(T)}) \lesssim L(D^{(T)},G')$ for any $D'\in {\cal D}$ and $G'\in {\cal G}$ where $\lesssim$ denotes that $\leq$ approximately holds. Thus, $G^{(T)}$ is regarded as an empirical estimator $\hat{G}^*$ for $G^*$.

Unfortunately, this adversarial process tends to easily forget some of the captured areas in ${\cal S}_{t'}$ at a later $t ( > t')$th iteration, particularly when $p_X$ is multimodal. For a toy example in Figure~\ref{fig:mode_collapse}, the generator revolves around unimodal distributions.  This \emph{mode collapse} behavior is a persistent challenge in the GAN literature. 
\cite{good:16} suggested that the conventional update scheme \eqref{grad_desc} may inadvertently address a max-min problem, which encourages the generator to produce only the most probable modes with the highest values. From the optimal transport perspective, \cite{an:etal:19} showed that, when ${\cal X}$ is non-convex, the neural-net generator $G$ is discontinuous at certain singularities in ${\cal Z}$, so $G$ may just represent one side of the discontinuity. To mitigate mode collapse, various studies have been proposed \citep{zhou:etal:19,kim:etal:23}, but the exact cause of mode collapse still remains unknown.

\begin{figure}[t]
\centering
\includegraphics[width=1.0\textwidth]{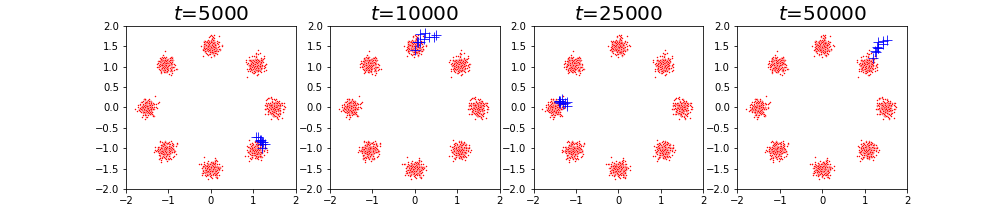}\vspace{-0.2in}
\caption{Mode collapse behavior: target distribution (plain dots) and $G^{(t)}(Z)$ (cross dots). Each unimodal distribution is 1.5 away from the origin and has a variance of 0.01. For ${\cal D}$ and ${\cal G}$, $N^D_{i+1}=N^G_{i+1}=256$ and $\kappa_i(x)=\max\{x,0\}$ (ReLU) for $i=1,\dots,d-1$ (and $g-1$ respectively) where $d=g=5$ and $d_Z=4$. Both $n_b$ and $m_b$ are set to be 100.}
\label{fig:mode_collapse}
\end{figure}

\subsection{Inefficient estimation under severe multimodality}
\label{sec:inef_est}

This section explores how mode collapse worsens GAN training by focusing on the behavior of gradients' variance of $D$ w.r.t. the discrepancy between $\Gt$ and $\calS^c$. Our analysis begins by observing that mode collapse gives rise to multimodality discrepancy between $\Gt$ and $\calS^c$. Keeping in mind this relationship, we first demonstrate that the neural distance is a well-suited metric for assessing multimodality, and then clarify that the degree of multimodality is closely associated with the size of gradients' variance. 

\subsubsection{Assessing the degree of multimodality}
\label{sec:multimodal}
% In statistics, $p_X$ is multimodal if its density has multiple peaks. For instance, a bimodal density can be written as $p_X(x) = \sum_{k=1}^2 p_k(x;\mu_k,\sigma_k^2)/2$, where $\mu_k$ and $\sigma_k^2$ are the mean and variance for $k=1,2$. Assuming equal variances $\sigma = \sigma_1 = \sigma_2$, a common measure of bimodality is the ratio $\lvert \mu_1 - \mu_2\rvert/\sigma$ \citep{ashm:etal:94, wang:etal:09}. Likewise, the neural distance reflects bimodality, as shown in Proposition~\ref{prop:nd_wd}, primarily capturing between-variability under equal variances.

In statistics, $p_X$ is called a multimodal distribution when its density function has more than one peak. For example, the density of $p_X$ having two modalities can be written $p_{X}(x)=\sum_{k=1}^2 p_k(x;\mu_k,\sigma_k)/2$ with mean $\mu_k$ and standard deviation $\sigma_k$ for $k=1,2$. To simplify a discussion, we suppose $\sigma=\sigma_1=\sigma_2$. As a way of measuring bimodality of $p_X$, one can come up with the ratio of between-variability and within-variability like $|\mu_1 - \mu_2|/\sigma$ \citep{ashm:etal:94, wang:etal:09}. Likewise, the neural distance represents the degree of bimodality as shown in the following Proposition~\ref{prop:nd_wd}, but it mainly reflects the between-variability. % when equal variances are given.
% \begin{proposition}
% \label{prop:nd_wd}
% Under (A1-2), the neural distance between probability distributions $p_1(x;\mu_1,\sigma^2)$ and $p_2(x;\mu_2,\sigma^2)$ in a location-scale distribution family is bounded by 
% \begin{align*}
%     \lvert \mu_1-\mu_2\rvert-\omega_{\cal D} \leq d_{\cal D}(p_1,p_2) \leq 
%     \prod_{l=1}^d M_w(l) \prod_{s=1}^{d-1} K_{\kappa}(s)\lvert \mu_1-\mu_2\rvert, 
% \end{align*}
% where $\omega_{\cal D}$ stands for the universal approximation capability of $\cal D$, i.e., the neural networks in $\cal D$ can approximate any 1-Lipschitz function with $\omega_{\cal D}$ error (please refer to SM~\ref{supp:prop1} for a more concrete explanation of $\omega_{\cal D}$). 
% \end{proposition}
\begin{proposition}
\label{prop:nd_wd}
    Under (A1-2), the neural distance between probability distributions $p_1$ and $p_2$ is associated with the 1-Wasserstein distance $d_{W_1}$ as $d_{W_1}(p_1,p_2)-\omega_{\cal D}\leq d_{\cal D}(p_1,p_2)\leq \prod_{s=1}^d M_w(s) \prod_{u=1}^{d-1} K_{\kappa}(u) d_{W_1}(p_1,p_2)$ where $\omega_{\cal D}$ stands for the universal approximation capability of $\cal D$, i.e., the neural networks in $\cal D$ can approximate any 1-Lipschitz function with $\omega_{\cal D}$ error (please refer to SM~\ref{supp:prop1} for a more concrete explanation of $\omega_{\cal D}$).  For the example, $|\mu_1-\mu_2|\leq d_{W_1}(p_1,p_2)\leq \sqrt{2\sigma^2+|\mu_1-\mu_2|^2}$ holds.
\end{proposition}

\noindent 
Note that the between-variability part $|\mu_1-\mu_2|$ becomes equal to the classic bimodality $|\mu_1-\mu_2|/\sigma$ if data is scaled such as $X/\sigma$. Also, the relationship generalizes to a finite-dimensional case $X\sim p_1,Y\sim p_2$ with vector-valued $\mu_1,\mu_2$ {on account of $\lVert \mu_1-\mu_2\rVert \leq d_{W_1}(p_1,p_2)\leq \sqrt{\mathrm{Tr}(\mathrm{Cov}(X))+\mathrm{Tr}(\mathrm{Cov}(Y))+\lVert \mu_1-\mu_2\rVert^2}$}, emphasizing the role of the between-variability. Hence, it is reasonable to posit that the neural distance generally expresses the degree of \emph{multimodality} between any distinguishable distributions $p_1$ and $p_2$. That is, the larger value of $d_{\cal D}(p_1,p_2)$ intuitively implies that $p_X$ is exposed to substantial multimodality. 

%Note that measuring the degree of multimodality for general $p_X$ is yet an open problem to our knowledge.

\subsubsection{Inflation of gradients' variance under mode collapse}
\label{sec:grad_var_inflation}
Such characterization of multimodality by the neural distance is a key step to examine the GAN training mechanism \eqref{grad_desc} in the presence of mode collapse. That is, we say that $d_{\cal D}(p_{\calS^c},p_{\Gt(Z)})$ signifies the degree of multimodality in mode collapse. Here, we investigate the gradients' variance of $\Dt$ and discover that the severe multimodality induced by incomplete $\Gt$ can cause the GAN training process to be statistically inefficient or unstable. This analysis ultimately justifies that the use of tempered distributions of $p_X$ can significantly improve GAN training. 

At first, we derive a lower bound of the gradients' variance of $\Dt$. Due to its highly non-linear structure, we particularly focus on the last weight matrix $w_d^{(t)}$, while $W_l^{(t)}$, for $l=1,\dots,d-1$, are expected to show a similar tendency due to the connectivity of $W_l^{(t)}$ to $w_d^{(t)}$ via backpropagation. As a first step, we design a classification rule in terms of the $t$th loss function. Let $\Lti=\Dt(X_i)-\bE[\Dt(\Gt(Z))]$ for the $i$th entity. Then, as shown in Figure~\ref{fig:value_variance} with the 2-mixture $p_X$ with $\sigma^2=0.01$, $\bE[\Lti|\Lti\geq \epsilon]$ implicitly measures the remaining distance $d_{\cal D}(p_{\calS^c},p_{\Gt(Z)})$ for some small $\epsilon$, which allows us to explain the below lower bound by the neural distance. The details of Figure~\ref{fig:value_variance} appear in SM~\ref{appen:fig}.
\begin{proposition}
\label{prop:grad_lwbd}
Assume $\lVert w_d^{(t)}\rVert>0$. Let $\zetaeps = P[\Lti\leq \epsilon]$,  $\sigma_{\calS}^2=\var[\Lti|\Lti\leq \epsilon]$, $\sigma_{\calS^c}^2=\var[\Lti|\Lti> \epsilon]$, and $\sigma^2_{\Gt}=\var[\Dt(\Gt(Z_j))]$. If $\bE[\Lti| \Lti \leq \epsilon]=0$ for some $\epsilon>0$, the norm of the covariance of $w_d$'s gradient at the $t$th iteration is bounded below by
    \begin{align}
    \label{eqn:grad_lwbd}
            \left\lVert \text{Cov}\left(\dfrac{\partial\hat{L}_b(\Dt,\Gt)}{\partial  w_d}\right)\right\rVert_{2}\geq \dfrac{\zetaeps(1-\zetaeps) \bE[\Lti|\Lti\geq \epsilon]^2 +{\zetaeps\sigma_{\calS}^2 +(1-\zetaeps)\sigma_{\calS^c}^2} +  \sigma^2_{\Gt}}{n_b \lVert  w_d^{(t)} \rVert^2},
    \end{align}
    where $\lVert \cdot\rVert_2$ for the covariance matrix is the induced 2-norm. 
\end{proposition}
\noindent Proposition~\ref{prop:grad_lwbd} implies that the gradient update of $w_d^{(t)}$ becomes more noisy (hence the estimation of the neural distance becomes statistically inefficient) if $\Gt$ and $\calS^c$ are distant (induced by more severe multimodality of $p_X$) and $\Gt$ partially recovers $\cal S$ (a moderate $\zetaeps\in(0,1)$), which eventually leads to unstable and inefficient GAN training. From this perspective, we say that \emph{the GAN training for $p_X$ with severe multimodality is essentially much harder than with less multimodality.} As empirical evidence in Figure~\ref{fig:value_variance}, the case of $\mu_2=3.0$, which has larger multimodality, involves a much larger gradient variance. 

% {\color{red}
% Note that mode collapse refers to the phenomenon in which the generator captures a nonempty subset of $\cal S$ rather than the whole support. If $\zetaeps\approx 0$, i.e., the generator completely fails with ${\cal S}_t \approx\emptyset$,  the bound \eqref{eqn:grad_lwbd} tends to become relatively smaller. This is consistent with the gradient vanishing phenomenon that produces a meaningless generator with a negligible training gradient \citep{ding:etal:22}. }

Next, an upper bound of the gradients' variance of $\Dt$ is derived. To simplify notation, let's denote by $W^{(t)}_{l,r,c}$ the $(r,c)$th entry of $W^{(t)}_l$ for $l=1,\dots,d$ where the third index of $W^{(t)}_{d,r,c}$ is regarded dummy for $w_d^{(t)}$, i.e., $w^{(t)}_{d,r}=W^{(t)}_{d,r,c}$ and by $W_{l,r,\cdot}^{(t)}$ its $r$th row vector. SM~\ref{supp:disc_thm1} shows the specific form of the below constants $C_{\bw}^{(t)}(l)$ and $C_{\kappa,j}^{(t)}(l)$ ($j=1,2,3$) and their implication depending on the type of activation. For example, $C_{\bw}^{(t)}(l)=\lVert W_{l+1,r,\cdot}^{(t)\top} \rVert \prod_{j=l+2}^{d-1} \lVert W_{j}^{(t)} \rVert_F \lVert w_{d}^{(t)} \rVert$ for $l\leq d-3$ reflects the size of $\Dt$ in backpropagation, and it can be positively related to the norm $M_w(\cdot)$. $C_{\kappa,3}^{(t)}(l)$ compares $p_{X}$ and $p_{\Gt(Z)}$ by the covariance between the $(l-1)$th hidden layer and backpropation gradients after it, which vanishes as $X\overset{d}{\approx}\Gt(Z)$.
\begin{theorem}
\label{prop:grad_upbd1}
Under (A1-3), $\left\lvert \pr \Lhat_b(\Dt,\Gt)/\pr W_{l,r,c}\right\rvert/C_{\bw}^{(t)}(l)$ is bounded by
\begin{align}
\label{eqn:grad_upbd1}
 \leq C_{\kappa,1}^{(t)}(l)d_{\cal D}(p_X,p_{\Gt(Z)}) +  
C_{\kappa,2}^{(t)}(l)d_{\kappa}(p_X,p_{\Gt(Z)}) + C_{\kappa,3}^{(t)}(l) +O_p\left(1/{\sqrt{n_b}}\right), %O_p\left(\dfrac{1}{\sqrt{n_b}}\right),
\end{align}
for any $l,r,c$, where the constant $C_{\kappa,j}^{(t)}(l)$ for $j=1,2,3$ relies on the type of activation. The discrepancy $d_{\kappa}$ is the 1-Wasserstein distance $d_{W_1}$ or the total variation $d_{\text{TV}}$, respectively, depending on the Lipschitzness or boundness of the activation's derivative $\kappa'_l(x)$ for all $l$. 
\end{theorem}
\noindent The square of the upper bound in \eqref{eqn:grad_upbd1} becomes the bounds for gradients' variance since $\var[X]\leq \bE[\lvert X \rvert^2]$. Moreover, because $p_X(x)=\zeta_t p_{\calS}(x) + (1-\zeta_t) p_{\calS^c}(x)$ for all $x\in {\cal S}$ where $\zeta_t$ is the proportion of the recovered support, we see $d_{\cal D}(p_X,p_{\Gt(Z)})\leq \zeta_t d_{\cal D}(p_{\calS},p_{\Gt(Z)}) + (1-\zeta_t) d_{\cal D}(p_{\calS^c},p_{\Gt(Z)})$. Therefore, the upper bound \eqref{eqn:grad_upbd1} relates to the degree of mode collapse or multimodality of $p_X$ that is coherently characterized by $d_{\cal D}(p_{\calS^c},p_{\Gt(Z)})$. Together with the lower bound result, we remark the importance of this remaining distance. 
\begin{remark}
    Proposition~\ref{prop:grad_lwbd} and Theorem~\ref{prop:grad_upbd1} highlight that the remaining distance is an essential part of capturing the gradients' variance. For instance, for $w_d$, $C_{\bw}^{(t)}(d)=1$,  $C_{\kappa,1}^{(t)}=1/M_w(d)$, and $C_{\kappa,j}^{(t)}=0$ for $j=2,3$ are set, and thus we observe $\bE[L_i^{(t)}|L_i^{(t)}\geq \epsilon]^2/\lVert w_d^{(t)} \rVert^2$ in \eqref{eqn:grad_lwbd} conceptually corresponds to $d_{\cal D}(p_{\calS^c},p_{\Gt(Z)})/M_w(d)$ in \eqref{eqn:grad_upbd1}.
\end{remark}
Note that under mode collapse, $d_{\cal D}(p_{\calS},p_{\Gt(Z)})$ is generally negligible compared with $d_{\cal D}(p_{\calS^c},p_{\Gt(Z)})$. The same arguments also hold for $d_{\kappa}$ since it is positively related to $d_{\cal D}$ in general. There can be situations where the discrepancy between $p_{\calS}$ and $p_{\Gt(Z)}$ is not ignorable, such as when $\Gt$ generates synthetic samples out of ${\cal S}$. In any case, the incomplete training of $\Gt$ may cause the inefficient estimation of the gradients due to the enlarged remaining distance. 

The derived bounds above provide some insights to stabilize the general GAN training. Basically, \eqref{eqn:grad_lwbd} and \eqref{eqn:grad_upbd1} hint that the norm of the weight matrix should not decay faster than $d_{\cal D}(p_{\calS^c},p_{\Gt(Z)})$ during training to avoid the inflation of gradients' variance. Allowing for large $M_w(l)$ in ${\cal D}$ may settle this issue, but the technical constants could become too large, which may cause unstable training. Hence, it is desirable to have a reasonable size of the weight matrices during training, so that the gradients' variance does not inflate or shrink too much as the depth of $D$ gets deeper. This non-trivial observation justifies why the popular training tricks, such as normalizing weight matrices \citep{miya:etal:18} or imposing a penalty on $D$ \citep{mesc:etal:18, zhou:etal:19}, are practically able to stabilize the adversarial optimization to some extent. Also, our argument supports encouraging a good initialization of $\Gt$ \citep{zhao:etal:23} because in general $d_{\cal D}(p_{{\cal S}_0^c},p_{G^{(0)}(Z)})$ is likely to be huge. \emph{Moreover, the bounds rationalize that the gradient estimates can enjoy variance reduction effects if it is possible to intrinsically lower $d_{\cal D}(p_{\calS^c},p_{\Gt(Z)})$ with the norm of weight matrices controlled during the training.} 

\begin{figure}[t]
\centering
\includegraphics[width=1.0\textwidth]{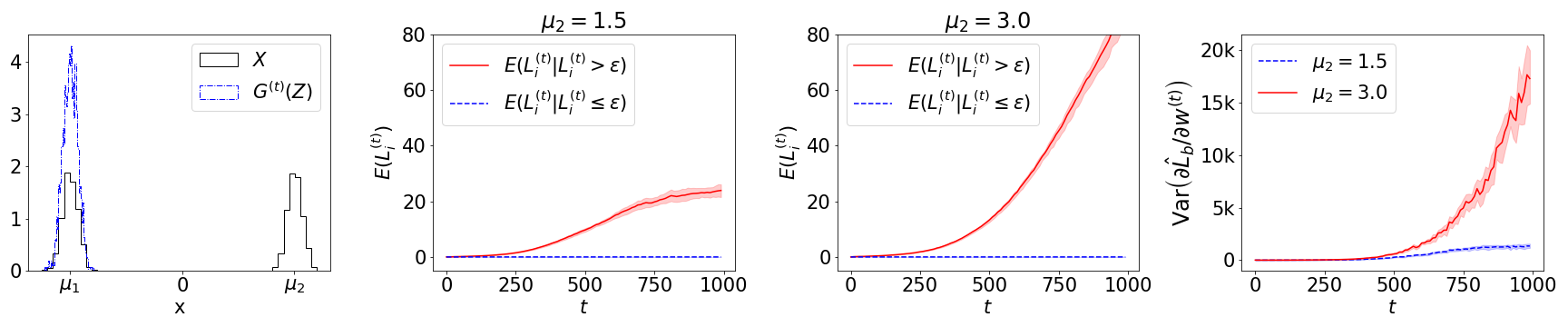}
\vspace{-0.6in}
\caption{Mode collapse induces multimodality. The leftmost depicts mode collapse where $\mu_1=-\mu_2$. In all $t$, $\Dt$ is updated, but $\Gt$ is fixed to generate the left mode. %The same ${\cal D}$ in drawing Figure~\ref{fig:mode_collapse} is used. 
The middle two panels illustrate the values of $\bE[L_i^{(t)}]$ along iterations when $\mu_2=1.5$ and $\mu_2=3.0$ respectively. The rightmost draws the behavior of the gradient variance for each case.
}
\label{fig:value_variance}
\end{figure}

\section{Tempered Distributions via Convex Interpolation}

In the previous section, we discussed how mode collapse leads to inefficient training of $\Dt$. A similar phenomenon, i.e., the local trapping problem, occurs in the Bayesian sampling of multimodal posterior distributions \citep{neal:96, lian:etal:14}. To tackle this problem, an annealing strategy creates a population of sampling targets with changing levels of multimodality. % associated with different temperatures. 
The same idea is also expected to be effective for GAN training because a tempered $p_X$ would involve smaller $d_{\cal D}(p_{\calS^c},p_{\Gt(Z)})$ even though mode collapse occurs. This work creates the tempered distributions by interpolating input variables. %Since the probability density function is generally not accessible in GAN, 

\label{sec:tempered_dist}
\subsection{Convex interpolation between data points}

To temper the unknown $p_X$, we define an auxiliary random variable that represents a tempered (or intermediate) distribution. The $i$th weighted random variable $\Qal_i$ is defined as 
\begin{align}
\label{def:convex}
    Q_i^{\alpha} = \alpha X_{i_1} + (1-\alpha) X_{i_2},
\end{align}
where $X_{i_1}, X_{i_2}$ are two random elements in $\bX_{1:n}$ and $\alpha\sim p_{\alpha}$ on $[0,1]$. The density of $Q_i^{\alpha}$ is denoted by $p_{\Qal}$ whose support is inside ${\cal Q}$ defined as the convex hull of $\mathcal X$ and its size is bounded by $\lVert \Qal\rVert\leq B_X$ as well. Figure~\ref{fig:tempered_dist} illustrates the distribution of $\Qal$ where $X$ follows the 2- and 8-component mixture distribution respectively with $\alpha \sim {\rm Unif}(0,1)$. Evidently, $\Qal$ has a more tempered distribution than $X$ because \emph{the created convex bridge connecting every pair of modes significantly reduces multimodality.} Note that placing such a convex support $\cal Q$ in GAN optimization helps avoid the discontinuity issue of the generator raised by \citep{an:etal:19}. A similar interpolation idea was attempted by the Mixup approach \citep{zhan:etal:17} but used the linear combination between $X_i$ and $G(Z_i)$ for GAN training. SM~\ref{supp:mixup} provides an in-depth discussion advocating our approach for GAN training. %The Mixup strategy only applies to a specific type of $d_{\cal D}$, whereas ours is universally applicable to most probability metrics. 

\subsection{Reduction of multimodality}
\label{sec:reduc_mul}
To see how the smoothing mechanism \eqref{def:convex} reduces multimodality more concretely, we bring the 2-mixture example $p_X(x)=\sum_{k=1}^2 p_k(x;\mu_k,\sigma)/2$. Given $\Qal=\alpha X_1+ (1-\alpha)X_2$ with two i.i.d. copies $X_1,X_2\sim p_X$ and $\alpha \sim {\rm Unif}(0,1)$, we design a 2-component mixture with density $p_{\Qal}(x)=\sum_{k=1}^2 p_k^{\alpha}(x;\mu_k^*,\sigma^*)/2$, where $p_1^{\alpha}$ (and $p_2^{\alpha}$) is the density of $\alpha u + (1-\alpha)X_2\sim p_1^{\alpha}$ (and $\alpha v + (1-\alpha)X_2\sim p_2^{\alpha}$) with $u\sim p_1$ (and $v \sim p_2$) and $\alpha \sim {\rm Unif}(0.5,1)$.
% \jw{\begin{proposition}
% \label{prop:reduction}
% One can show that $\lvert \mu_1^*-\mu_2^*\rvert=3\lvert \mu_1 -\mu_2\rvert/4$ and $\sigma^*=\sqrt{3\sigma^2/4+5(\mu_1-\mu_2)^2/192}$, thus $\Qal$ has decreased bimodality compared to $X$, i.e., $\lvert \mu_1^*-\mu_2^*\rvert /\sigma^* < \lvert \mu_1 - \mu_2 \rvert / \sigma$. Therefore, by Proposition~\ref{prop:nd_wd}, $d_{\cal D}(p_1^{\alpha},p_2^{\alpha})$ has smaller bounds than $d_{\cal D}(p_1,p_2)$ if $\omega_{\cal D}$ and $M_w(l)$ are the same for each distance under (A1-2). Meanwhile, $d_{W_1}(p_1^{\alpha},p_2^{\alpha})=\lvert \mu_1^*-\mu_2^*\rvert$ and $ d_{W_1}(p_1,p_2)=\lvert \mu_1 - \mu_2 \rvert$ for the 1-Wasserstein distance $d_{W_1}$ \citep{chha:etal:23}, so similar pattern can be expected for $d_{\text{TV}}$ \citep{chae:walk:20}. 
% \end{proposition}}
\begin{proposition}
\label{prop:reduction}
Following Proposition~\ref{prop:nd_wd}, the bounds of $d_{\cal D}(p_1^{\alpha},p_{2}^{\alpha})$ and $d_{W_1}(p_1^{\alpha},p_{2}^{\alpha})$ are written in terms of $|\mu_1^*-\mu_2^*|$ and $\sigma^*$ where $\lvert \mu_1^*-\mu_2^*\rvert=3\lvert \mu_1 -\mu_2\rvert/4$ and $\sigma^*=\sqrt{3\sigma^2/4+5|\mu_1-\mu_2|^2/192}$. Thus, if $\omega_{\cal D}$ and $M_w(l)$ are the same for $d_{\cal D}(p_1^{\alpha},p_{2}^{\alpha})$ and $d_{\cal D}(p_1,p_{2})$, $\Qal$ has less bimodality than $X$ in terms of $d_{\cal D}$ and $d_{W_1}$ due to $\lvert \mu_1^*-\mu_2^*\rvert< \lvert \mu_1 - \mu_2 \rvert$ (and also $\lvert \mu_1^*-\mu_2^*\rvert /\sigma^* < \lvert \mu_1 - \mu_2 \rvert / \sigma$). 
\end{proposition}
\noindent Note that there can be other ways to define the mixture components $p_1^{\alpha}$ and $p_2^{\alpha}$, but the same conclusion is induced in general. $d_{\text{TV}}(p_1^{\alpha},p_2^{\alpha})\leq d_{\text{TV}}(p_1,p_2)$ is also expected since divergences are positively correlated with each other in general \citep{chae:walk:20}.

This observation helps differentiate the convex interpolation from adding a random noise $\epsilon$ to $X$ \citep{arjo:bott:17} in terms of the mechanism of easing multimodality. While adding noise increases the within-variability only, the convex interpolation not only expands the support of $X$ but also directly lessens the between-variability by building bridges connecting separate local modes, which significantly contributes to reducing multimodality. The following remark more specifically discusses this property for the toy example. 
\begin{remark}
\label{remark:reduction}
Let's define ${\cal S}_{\epsilon}$ as the support of $X+\epsilon$ for some level of $\epsilon$ and $G^{(t)}_{\epsilon}$ as the generator for $p_{X+\epsilon}$. We consider $p_{X+\epsilon}(x)=\sum_{k=1}^2 p_k^{\epsilon}(x;\mu_k,\sigma_{\epsilon})$ where $\sigma_{\epsilon}^2=\sigma^2 + {\rm Var}(\epsilon)$ and define the missing support ${\cal S}_{t,\epsilon}^c={\cal S}_{\epsilon}\setminus({\cal S}_{\epsilon} \cap \{G^{(t)}_{\epsilon}(z),z\in {\cal Z}\})$. For comparison, let's denote by $G^{(t)}_o$, $G^{(t)}_{\epsilon}$, and $G^{(t)}_{\alpha}$ the $t$th generator for the case of learning $p_X$, $p_{X + \epsilon}$, and $p_{\Qal}$ respectively. Likewise, we set ${\cal S}_{t,\alpha}^c = {\cal  Q}\setminus({\cal  Q}\cap \{G^{(t)}_{\alpha}(z),z\in {\cal Z}\})$ as the missing support in learning $p_{\Qal}$.

Now suppose that $G^{(t)}_o$, $G^{(t)}_{\epsilon}$, and $G^{(t)}_{\alpha}$ recover the right mode, i.e., $p_1=p_{\calS^c}$ and $p_2=p_{G^{(t)}_o}$; $p_1^{\alpha}=p_{{\cal S}_{t,\alpha}^c}$ and $p_2^{\alpha}=p_{G^{(t)}_{\alpha}}$; and $p_1^{\epsilon}=p_{{\cal S}_{t,\epsilon}^c}$ and $p_2^{\epsilon}=p_{G^{(t)}_{\epsilon}}$. Then Proposition \ref{prop:reduction} suggests that $d_{\cal D}(p_{{\cal S}_{t,\alpha}^c},p_{G^{(t)}_{\alpha}(Z)})$ tends to have a smaller remaining distance than $d_{\cal D}(p_{\calS^c}, p_{G^{(t)}_o(Z)})$ and also than $d_{\cal D}(p_{{\cal S}_{t,\epsilon}^c}, p_{G^{(t)}_{\epsilon}(Z)})$ particularly for the case without scaling data. 
The same arguments follow for $d_{W_1}$ as well. Simply adding noise may not effectively decrease the multimodality in view of the neural distance because it just increases the within-variability (e.g., the upper bound in Proposition \ref{prop:nd_wd} increases more when $X$ is not standardized). Thus, training with $\Qal$ is likely to involve smaller gradients' variance than with $X$ or $X+\epsilon$ in general, according to the discussion in Section~\ref{sec:grad_var_inflation}. The same tendency is foreseeable in a higher dimension $X\in\mathbb{R}^{d_X}$ for $d_X>1$ because the between-variability $\lVert \mu_1^*-\mu_2^* \rVert=3\lVert \mu_1-\mu_2 \rVert/4$ decreases.
\end{remark}
We name this particular property as \emph{mode connectivity}, highlighting that \emph{training with $\Qal$ would have a smaller remaining neural distance than training with only $X$ or $X+\epsilon$ when the generator is incomplete}. %Although our argument is based on the simple toy example, 
We believe that the same argument holds even when $p_X$ has more underlying unimodal distributions in higher dimensions because the core concept of creating $\Qal$ underpins whether it connects pairs of disconnected distributions in $p_X$, leading to a decrease of the remaining neural distance. 

\begin{figure}[t]
\centering
\includegraphics[width=1.0\textwidth]{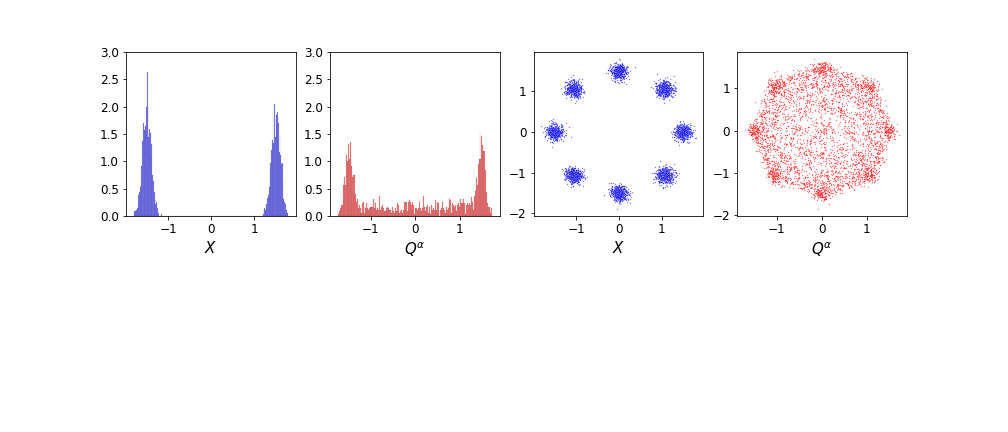}
\vspace{-1.7in}
\caption{Tempered distributions of $p_X$: the first and third panels depict histograms of $X$, with 2- and 8-component mixtures respectively, and the next of each represents $\Qal$.
}
\label{fig:tempered_dist}
\end{figure}

\section{Parallel Estimation}
\label{sec:parallel}

As leveraging the tempered distributions $p_{\Qal}$ for stabilized training and eventually for learning $p_X$ more effectively, we design a novel framework that trains a target generator to learn the joint distribution $p_{\Qal,\alpha}$. In contrast to the previous works relying on annealing strategies \citep{arjo:bott:17,sajj:etal:18} or finding a specific temperature \citep{jenn:fava:19}, our framework does not rely on either of these techniques.

\subsection{Joint optimization with randomized temperature}

%the target distribution $p_X$ rather than $p_{\Qal}$. Depending on applications, one may want to obtain samples of $p_{\alpha X_1 + (1-\alpha)X_2}$ for some fixed $0<\alpha<1$.
% With the randomized temperature $\alpha \sim p_\alpha$, the proposed parallel tempering framework jointly estimates generators for all intermediate distributions in ${\cal P}_{\Qal}$ simultaneously. %To enhance flexibility, we use a mixture noise $\Zal_i = \alpha Z_{i_1} + (1 - \alpha) Z_{i_2}$, with $Z_{i_1}, Z_{i_2} \sim p_Z$.
%The proposed parallel tempering framework is designed to estimate a series of generators for all intermediate distributions in ${\cal P}_{\Qal}$ simultaneously. The temperature parameter $\alpha$ is randomized $\alpha \sim p_{\alpha}$ so that a new neural distance can handle the joint distribution of $\Qal$ and $\alpha$. Accordingly, we have the family of joint densities $p_{\Qal,\alpha}$ denoted by ${\cal P}_{{\cal Q},[0,1]}$. For sufficient flexibility of $G$, a mixture reference noise $\Zal_i=\alpha Z_{i_1} + (1-\alpha) Z_{i_2}$ is used where $Z_{i_1}$ and $Z_{i_2}$ are randomly drawn from $p_Z$. 

A major goal of our GAN framework is to obtain a data generator that can produce synthetic samples following the target distribution $p_X$. At the same time, downstream applications may need to collect samples of $p_{\alpha X_1 + (1-\alpha)X_2}$ for a certain level of $\alpha$. For this purpose, we formulate the neural distance $d_{\cal D}(p_{\Qal,\alpha}, p_{G(Z, \alpha),\alpha}) = \sup_{D \in {\cal D}}\{{\bf E}_{\Qal,\alpha} [D(\Qal, \alpha)] - {\bf E}_{Z,\alpha}[D(G(Z,\alpha),\alpha)]\}$ so that the trained $G$ returns an intermediate distribution of $\alpha X_1 + (1-\alpha) X_2$ for any $\alpha$. In this revised framework, $W_1$ and $V_1$ in ${\cal D}$ and ${\cal G}$ respectively adopt one more input dimension for receiving $\alpha$. Our training framework, therefore, is to solve 
\begin{align}
\label{builtin}
    \min_{G\in {\cal G}} d_{\cal D}(p_{\Qal,\alpha}, p_{G(Z, \alpha),\alpha}).
\end{align}
If $\alpha=1$ or $\alpha=0$, the new optimization \eqref{builtin} reduces to the original problem minimizing \eqref{eq:nd}; the original $\Lhat_b$ can be written as $\hat{L}_b^1(\Dt,\Gt)=\sum_{i=1}^{n_b}\Dt(X_i,1)/n_b-\sum_{j=1}^{m_b} \Dt(\Gt(Z_j,1),1)/m_b$.
The optimization of $D^*$ and $G^*$ for \eqref{builtin} is carried out via the gradient-based update \eqref{grad_desc} w.r.t. $\hat{L}_b^{\alpha}(\Dt,\Gt)=\sum_{i=1}^{n_b}\Dt(Q^{\alpha_i}_i,\alpha_i)/n_b-\sum_{j=1}^{m_b} \Dt(\Gt(Z_j,\alpha_j),\alpha_j)/m_b$ where $Q^{\alpha_i}_i=\alpha_i X_{i_1} + (1-\alpha_i) X_{i_2}$ for randomly chosen $X_{i_1},X_{i_2} \in \bX_{1:n}$, $Z_j$ from $p_{Z}$, and $\alpha_i,\alpha_j$ from $p_{\alpha}$. Because $\bw^{(t)}$ and $\bv^{(t)}$ are updated to reflect the distributions having different levels of smoothness concurrently in solving \eqref{builtin}, we call it a parallel tempering scheme. %\jw{This update scheme remarks that the successfully learned generator synthesizes $p_{\alpha X_{1}+(1-\alpha)X_{2}}$ for any slice $\alpha$ because $\bw^{(t)}$ and $\bv^{(t)}$ are shared across the distributions having different levels of smoothness/multimodality.}

The distributional symmetry of $\Qal$ imposes a constraint on $D$ and $G$ as to the use of $\alpha$. Considering $\alpha X_1 + (1-\alpha)X_2 \overset{d}{=} (1-\alpha) X_1 + \alpha X_2$ holds for any $0\leq \alpha\leq 1$, ensuring $D(\Qal,\alpha)\overset{d}{=}D(Q^{1-\alpha},1-\alpha)$ and $G(Z,\alpha) \overset{d}{=} G(Z,1-\alpha)$ is desirable. By devising a transformation function $t(x)$ symmetric at $0.5$ and plugging it into $D(\Qal,t(\alpha))\overset{d}{=}D(Q^{1-\alpha},t(1-\alpha))$ and $G(Z,t(\alpha))\overset{d}{=} G(Z,t(1-\alpha))$, the constraint can be satisfied. This work adopts $t(x)=-2\lvert x-0.5\rvert+1$ for simulation studies.  

In the ideal case, the perfectly learned $G(Z,\alpha)$ recovers the ground-truth marginal distributions $p_{\Qal,\alpha}$ for all levels of $\alpha$. However, it might not be achievable in practice due to the finite sample size, limited computing resources, and so forth. For instance, if $\alpha \sim {\rm Unif}(0,1)$, the training objective may need far longer iterations to successfully learn $\Qal$ at $\alpha=1$ or $\alpha=0$ because, especially under high multimodality, the sampled $X$ from $p_X$ lies in the boundary region of ${\cal Q}$. For these reasons, we suggest using a mixture-type distribution for $\alpha$ defined as  
\begin{align}
\label{eqn:alpha_dist}
    \alpha \sim r \delta_1(\cdot) + (1-r) p_{\alpha^*}(\cdot),
\end{align}
where $0\leq r\leq 1$, $\alpha^*\sim {\rm Unif}(0,1)$, and $\delta_1$ is a Dirac measure. Such specification of $p_{\alpha}$ naturally encourages the training process to concentrate more on the marginal distribution at $\alpha=1$. Note $r$ stands for the proportion of data instances picked from $p_X$ in minibatches. 
 
\subsection{Reduction of gradients' variance}

\begin{figure}[t]
\centering
\includegraphics[width=1.0\textwidth]{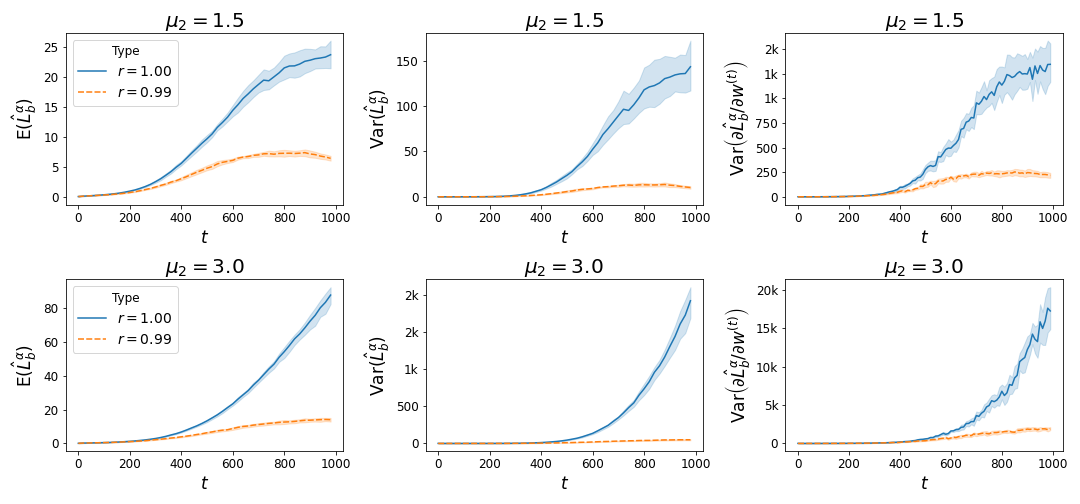}
\vspace{-0.6in}
\caption{Comparison of $\hat{L}^{\alpha}_b(\Dt,\Gt)$ and $\hat{L}_b^1(\Dt,\Gt)$ on the toy example in Figure~\ref{fig:value_variance} with the same structure of ${\cal D}$: in the case of $r=1$,  minibatches consist of original samples. For the case of $r=0.99$, there is only 1\% of interpolated samples in each minibatch.}
\label{fig:ref_choice}
\end{figure}

This section shows that our parallel tempering framework can enjoy improved training stability by reducing the variance of the gradients. Our analysis focuses on comparing the gradients of $\Dt$ when it comes with $\hat{L}_b^{\alpha}(\Dt,\Gt)$ or $\hat{L}_b^1(\Dt,\Gt)$, that is, parallel tempering training versus vanilla training. For fair and explicit comparison, we examine the behavior of one-step update $\bw^{(t+1)}$ for $\Lhat_{b}^{\alpha}$ and $\Lhat_b^1$ respectively, updated from the same configuration of $\bw^{(t)}$. To be specific, we compare $\bw^{(t+1)}|_{p_{\alpha}}=\bw^{(t)}+\gamma_D \partial \hat{L}_b^{\alpha}(\Dt,\Gt)/\partial \bw$ and $\bw^{(t+1)}|_{\delta_1}=\bw^{(t)}+\gamma_D \partial \hat{L}_b^1(\Dt,\Gt) / \partial \bw$.
The use of $\Qal$ introduces some ``bias'' $\bE[\bw^{(t+1)}|_{p_\alpha}]-\bE[\bw^{(t+1)}|_{\delta_1}]$ if the ultimate goal is to train a generative model under $\alpha=1$ (i.e., $\bE[\bw^{(t+1)}|_{\delta_1}]$ is the ``gold standard'' gradient). This bias disappears when $r=1$ but would exist when $0 \leq r <1$. This bias could be negative for learning the target marginal distribution $p_{X,1}$ because the tempering approach essentially learns the joint density of $p_{\Qal,\alpha}$. However, on the bright side, we find that the gradients' variance can substantially decrease with $r<1$ as the following corollary and remark substantiate.
\begin{corollary}
\label{prop:grad_upbd2_2}
Suppose (A1-3) holds. With $C_{\kappa,j}^{(t)}(l)$ for $j=1,2,3$ and $C_{\bw}^{(t)}(l)$ in Theorem~\ref{prop:grad_upbd1}, for any $l,r,c$, $\left\lvert \pr \Lhat_b(\Dt,\Gt)/\pr W_{l,r,c}\right\rvert/C_{\bw}^{(t)}(l)$ is bounded by
\begin{align*}
 \leq C_{\kappa,1}^{(t)}(l)d_{\cal D}(p_{\Qal,\alpha},p_{\Gt(Z,\alpha),\alpha}) +  
C_{\kappa,2}^{(t)}(l)d_{\kappa}(p_{\Qal,\alpha},p_{\Gt(Z,\alpha),\alpha}) + C_{\kappa,3}^{(t)}(l) +O_p\left(1/{\sqrt{n_b}}\right).
\end{align*}
\end{corollary}
\begin{remark}
\label{rem:var_reduction}
    Under \eqref{eqn:alpha_dist}, if $d_{\cal D}(p_{Q^{\alpha^*},\alpha^*},p_{\Gt(Z,\alpha^*),\alpha^*})\leq d_{\cal D}(p_{X,1},p_{\Gt(Z,1),1})$, it is likely that $d_{\kappa}(p_{Q^{\alpha^*},\alpha^*},p_{\Gt(Z,\alpha^*),\alpha^*})\leq d_{\kappa}(p_{X,1},p_{\Gt(Z,1),1})$ due to the positive relationship between $d_{\cal D}$ and $d_{\kappa}$. Thus Corollary~\ref{prop:grad_upbd2_2} suggests that the variance can decrease. In this regard, $r$ should be carefully tuned to effectively balance the bias–variance trade-off in training $D$ and $G$ to maximize the utility of generated data. The divergences at $\alpha=1$ are generally larger than the ones defined on $\alpha^* \sim p_{\alpha^*}$, e.g., during early training course or mode collapse. 
\end{remark}
Figure~\ref{fig:ref_choice} compares the expectation and the (gradients') variance of $\hat{L}_b^{\alpha}(\Dt,\Gt)$ and $\hat{L}_b^1(\Dt,\Gt)$ on the toy example introduced in Figure~\ref{fig:value_variance}.  In Figure~\ref{fig:ref_choice}, we see that the estimates of $d_{\cal D}$, i.e., $\bE[\hat{L}_b^{\alpha}]$, become substantially smaller with $r=0.99$ and in more severe multimodality as well, validating the variance reduction of the gradients suggested by Remark \ref{rem:var_reduction}. Moreover, the (gradients') variance of $\Lhat^{\alpha}_b$ and $\Lhat_b^1$ closely resemble the behavior of $\bE[\Lhat_b^{\alpha}]$ and $\bE[\Lhat_b^1]$, which also supports Corollary~\ref{prop:grad_upbd2_2} that compares the size of gradients' variance by the upper bound. More specific analysis within a linear class of $D$ appears in SM~\ref{supp:ad_var_linear}, which shows a consistent conclusion.

The use of $\Qal$, however, may not always bring the variance reduction effects. There might be the reverse relationship $d_{\cal D}(p_{Q^{\alpha^*},\alpha^*},p_{\Gt(Z,\alpha^*),\alpha^*}) > d_{\cal D}(p_{X,1},p_{\Gt(Z,1),1})$. %, i.e., $\Gt(Z,1)$ is closer to $X$ than $\Gt(Z,\alpha)$ to $\Qal$ for $\alpha < 1$. 
For instance, if $\Gt$ already covers ${\cal S}$ well, the condition may not hold. In this case, actually, the training does not obstinately need the variance reduction because the neural distance $d_{\cal D}(p_{X,1},p_{\Gt(Z,1),1})$ (and also $d_{\kappa}$) is sufficiently minimized yielding the small variance according to Corollary~\ref{prop:grad_upbd2_2}. 

Reducing gradients' variance has been mostly addressed in the optimization literature. \cite{yang:etal:20} shows that solving nonconvex-nonconcave min-max problems has $O\left(\sigma^2_{\text{grad}}/{t}\right)$ convergence rate under the two-sided Polyak-{\L}ojasiewicz condition w.r.t. the loss function where $\sigma^2_{\text{grad}}$ upper bounds gradients' variance. For more details, refer to \cite{yang:etal:20} and references therein. To emphasize the importance of reducing the gradients' variance, this work provides empirical evidence in SM~\ref{supp:noisy_gradient}, displaying the failure of general GAN training when the critic $\Dt$ is exposed to high variance of gradients.

\subsection{Coherent training via regularization}
\label{sec:penalty}

Although training with $p_{\Qal,\alpha}$ offers such statistical benefits, the optimization in \eqref{builtin} may fall into separate GAN training for each $\alpha$, i.e., learning $\Qal|\alpha$ individually. In such cases, there might be no guarantee that the use of $\Qal$ contributes to the GAN training for $p_X$. For example, if training converges at $\alpha_1$ but not at $\alpha_2 \neq \alpha_1$, then $G$ sharing network parameters for all $\alpha$ may sacrifice the equilibrium at $\alpha_1$ to improve convergence at $\alpha_2$.

% Despite such statistical benefits of training with $p_{\Qal,\alpha}$, the optimization \eqref{builtin} may fall into separate GAN training (w.r.t. each $\alpha$). In other words, the optimization problem \eqref{builtin} reduces to learning $\Qal|\alpha$ individually for every slice of $\alpha$. In this case, there might be no guarantee that the use of interpolated variables contributes to the GAN training for $p_X$. For instance, for some specific $\alpha_1$ and $\alpha_2$, imagine that the training at $\alpha_1$ is complete but yet for $\alpha_2\neq \alpha_1$. Then the equilibrium at $\alpha_1$ might be compromised to promote the convergence at $\alpha_2$ because the networks' parameters are shared for all $\alpha$.

% Also, the GAN training of \eqref{builtin} may not simultaneously reach out equilibria at all $\alpha$, potentially leading to adverse effects on training stability. For example, it is possible that the training at $\alpha_1$ is complete but yet for another $\alpha_2\neq \alpha_1$. Then the models could compromise the equilibrium at $\alpha_1$ to promote the convergence at $\alpha_2$ because the neural networks share the parameters for all $\alpha$. 

To prevent such a separate training system, we devise a novel penalty to maximize the potential of the parallel training \eqref{builtin}. In a nutshell, an additional condition to $D$ is imposed such that the learning process across different temperatures is at a similar pace. Given two samples $X_1$ and $X_2$ drawn from two distribution modes of target distributions, thus $\alpha X_1 + (1-\alpha)X_2$ represents a sample from one distribution mode of $p_{Q^{\alpha}}$ for any $\alpha$. The idea is to synchronize the learning pace across different $\alpha$.
As $D(\alpha X_1 + (1-\alpha)X_2, \alpha)$ relates to how good the generator $G(\cdot,\alpha)$ learns the distribution mode represented by $\alpha X_1 + (1-\alpha)X_2$, we regularize the coherency of $D$ values for all $\alpha$'s. That is, 
for $\Qalf=\alpha_1 X_1 + (1-\alpha_1)X_2$ and $\Qals=\alpha_2 X_1 + (1-\alpha_2)X_2$ under the same $(X_1, X_2)$, we hope that $D(\Qalf, \alpha_1)\approx D(\Qals, \alpha_2)$. Through the mean value theorem, we brutally approximate the difference between $D(\Qalf, \alpha_1)$ and $D(\Qals, \alpha_2)$ by $\nabla_{{Q}^{\tilde\alpha}} D({Q}^{\tilde\alpha}, \tilde{\alpha})\cdot(\Qalf - \Qals)$ where $\tilde{\alpha}=\nu \alpha_1 + (1-\nu) \alpha_2$  for some $0\leq \nu\leq 1$, and thus place the following novel penalty when updating $D$,
\begin{align}
\label{penalty}
H=\lambda \bE_{\alpha_1,\alpha_2,\nu}\bE_{\Qalf,\Qals}\left[\left(\nabla_{{Q}^{\tilde\alpha}} D({Q}^{\tilde\alpha}, \tilde{\alpha})\cdot(\Qalf - \Qals)\right)^2\right], 
\end{align}
where $\alpha_1 \sim p_{\alpha}$, $\alpha_2 \sim {\rm Unif}(0,1)$, $\nu\sim{\rm Unif}(0,1)$, and $\lambda$ is a hyperparameter to determine the penalty's impact. %\jw{If $r=1$, \eqref{penalty} reduces to the penalty suggested by \cite{mesc:etal:18}. (incorrect)}

\begin{remark}
Note that  $\nabla_{{Q}^{\tilde\alpha}} D({Q}^{\tilde\alpha}, \tilde{\alpha})\cdot(\Qalf - \Qals)$ can be rewritten as 
 $(\alpha_1-\alpha_2)\nabla_{{Q}^{\tilde\alpha}} D({Q}^{\tilde\alpha}, \tilde{\alpha})\cdot(X_1 - X_2)$. Thus, if $X_1$ and $X_2$ are far away from each other, i.e., two distant distribution modes of $p_X$, the penalty is larger. Intuitively, the proposed coherency penalty is a weighted penalty w.r.t. between-mode distance, so it accommodates the multimodality of $p_X$.
\end{remark}

Additionally, the penalty helps avoid compromising the convergence of other temperatures. Intuitively, the penalty encourages $\bE_{\alpha_1,\alpha_2}\bE_{\Qalf,\Qals}[\lVert D(\Qalf,\alpha_1)-D(\Qals,\alpha_2)\rVert^2]$ to decrease and thus contributes to diminishing $\bE_{\Qal}[\lVert \nabla_{\Qal} D(\Qal, \alpha)\rVert^2]$ for all $\alpha$ simultaneously, so the training at least locally converges to the equilibrium for all $\alpha$ by \cite{mesc:etal:18}. Ideally, there would be no momentum to the escape of equilibria across all $\alpha$. The penalty also naturally helps control the size of weight matrices, so it further contributes to stabilizing the GAN training as discussed in Section~\ref{sec:grad_var_inflation}.

The implementation of our method consists mainly of three steps. In every iteration, the $n_b$ size minibatches of $\Qalf$, $\Qals$, and $Z$ are created respectively with $\alpha_1 \sim p_{\alpha}$, $\alpha_2 \sim {\rm Unif}(0,1)$, which secondly are used to evaluate $\hat{L}_b$ with $\Qalf$ and $Z$; and the penalty \eqref{penalty} with $\Qalf$ and $\Qals$. Then it executes the gradient ascent/descent for the critic and the generator, respectively.  
To see detailed implementations and possible variations, refer to Algorithm~\ref{alg:ptgan} in SM~\ref{supp:algorithm}. This work uses $\lambda=100$ as a default. 

\subsection{Statistical analysis}

In this section, we analyze the proposed distance in \eqref{builtin} and its estimation error within the size-independent sample complexity framework \citep{golo:etal:18,ji:etal:21}. We show that the estimated generator, which globally minimizes the neural distance in \eqref{builtin} for parallel training, achieves nearly min-max optimality. The employed theoretical framework readily adapts deep and wide neural networks by characterizing the sample complexity via the norm of weight matrices. %\om{In contrast to the prior works \citep{zhou:etal:22,metz:22}, our developed theory does not necessitate impractical conditions.  For example, the sample size must increase faster than or equal to the number of parameters. We remark that a lot of deployed neural networks in the real world nowadays have much more parameters than training samples.}

To begin with, the set of i.i.d. samples of $\Qal$ are constructed from $\bX_{1:n}$. Without loss of generality, the sample size $n$ is assumed even, so there is $n_e=n/2$ number of i.i.d. $\Qal$ samples constructed by $Q^{\alpha_i}_i = \alpha_i X_{2i-1} + (1-\alpha_i) X_{2i}$ for all $i=1,\dots,n_e$.  Let's denote by $d_{\cal D}(\hat{p}_{\Qal,\alpha},\hat{p}_{G(Z,\alpha),\alpha})=\sup_{D\in {\cal D}}\{\sum_{i=1}^{n_e} D(Q^{\alpha_i}_i,\alpha_i)/n_e - \sum_{j=1}^{m} D(G(Z_j,\alpha_j),\alpha_j)/m\}$ the empirical neural distance where $\hat{p}$ implies the empirical mass function, and the estimator $\hat{G}^*$ is determined by minimizing $d_{\cal D}(\hat{p}_{\Qal,\alpha},\hat{p}_{G(Z,\alpha),\alpha})$. Note that the following sample complexity analysis does not consider the minibatch scheme. The estimation error of $\hat{G}^*$ can be characterized by the population-level neural distance. By referring to the work of \cite{ji:etal:21}, we specifically write the estimation error as $d_{\cal D}(p_{\Qal,\alpha},p_{\hat{G}^*(Z,\alpha),\alpha})-\inf_{G\in{\cal G}}d_{\cal D}(p_{\Qal,\alpha},p_{G(Z,\alpha),\alpha})$ where $\inf_{G\in{\cal G}}d_{\cal D}(p_{\Qal,\alpha},p_{G(Z,\alpha),\alpha})$ represents the approximation error.
For the simplicity of analysis, the proposed penalty term (\ref{penalty}) is not considered.

%This extra condition is made to present Rademacher complexity of ${\cal D}$ and the composition class induced by $D\circ G$ in terms of the sample size and the characteristics of ${\cal D}$ and ${\cal G}$.

First, we find that the estimation error is bounded by the properties of ${\cal D}$ and the sample size. %Let's denote by ${\cal P}_{{\cal G},\alpha}$ the family of densities $p_{G(Z,\alpha),\alpha}$ for $G \in {\cal G}$. 
We further assume:
\begin{enumerate}[label={(A\arabic*)}]
    \setcounter{enumi}{3}
    \item The activation functions $\kappa_i$ and $\psi_j$ are positive homogeneous for all $i$ and $j$, i.e., $\kappa_i(cx)=c\kappa_i(x)$ and $\psi_j(cx)=c\psi_j(x)$ for any $c \geq 0$ and $x \in \mathbb{R}$.
\end{enumerate}
ReLU and lReLU are representative examples that satisfy this condition.
\begin{theorem}
\label{cor:ae}
Under (A1-4) and $n_e/m\rightarrow 0$, the estimation error is bounded above by 
\begin{align}
\label{eqn:ae}
    d(p_{Q^{\alpha},\alpha}, p_{\hat{G}^*(Z,\alpha),\alpha})-\inf_{G\in{\cal G}}d_{\cal D}(p_{\Qal,\alpha},p_{G(Z,\alpha),\alpha})\leq C_{\text{UB}}\dfrac{\sqrt{B_X^2 + 1}}{\sqrt{n_e}},
\end{align}
where $C_{\text{UB}}=\prod_{l=1}^d M_w(l) \prod_{s=1}^{d-1} K_{\kappa}(s)(4\sqrt{3d} + 2\sqrt{\log(1/\eta)}))$ with the probability $1-2\eta$.
\end{theorem}
\noindent Interestingly, the estimation error may not increase much although the critic uses a deeper network since the error depends on $\sqrt{d}$. The assumption that $m$ scales faster than $n_e$ is mild in the sense that the algorithm obtains i.i.d. samples $Z_i\sim p_Z$ in every iteration. Note (A4) can be eased to $\kappa_i(0)=\psi_j(0)=0$ if $W_i$ and $V_j$ have a bounded maximal 1-norm. Refer to Remark~\ref{remark:dropa4} in SM to see further discussion.

%It is easy to set the sufficiently large number of iterations $T$ such that $m = n_b\times T \gg n_e$.  %With sufficient computing resources, it is easy to set the sufficiently large number of iterations $T$ such that $m = n_b\times T \gg n_e$.  %

The approximation error becomes negligible as the capacity of ${\cal G}$ increases. Denote by $V_{\text{D}}$ and $V_{\text{W}}$ the depth and width of $G\in {\cal G}$ which corresponds to the number of weight matrices and the maximal size of hidden neurons in one layer $\max_{2\leq j\leq g}\{N^G_j\}$, respectively.
\begin{proposition}
\label{prop:ae}
    Suppose $p_X$ is supported within $[0,1]^{d_X}$, $Z\in \mathbb{R}^2$ is absolutely continuous on $\mathbb{R}^2$, ${\cal G}$ uses the ReLU activation function, and (A1) holds. For sufficiently large $\VD$ and $\VW$, the approximation error is then bounded by  
    \begin{align}
    \label{eqn:qpprox}
        \inf_{G \in {\cal G}} d_{\cal D}(p_{\Qal,\alpha},p_{G(Z,\alpha),\alpha}) \leq \prod_{l=1}^d M_w(l) C_{d_X} {(\lceil V_{\text{W}}/2 \rceil^2V_{\text{D}})^{-1/d_X}},
    \end{align}
    where $C_{d_X}$ is a constant that depends on the size of the input dimension $d_X$ only.
\end{proposition}
\noindent The approximation error is primarily influenced by the dimension of the target distribution, but, as \cite{huan;etal;22} justified, $d_X$ appearing in the exponent is reduced to the intrinsic dimension of $p_X$, which is usually smaller than $d_X$. The essence of our proof leverages two small sub-generators $\tilde{G}_1,\tilde{G}_2$, smaller than $G\in {\cal G}$, that approximate $p_X$ and observes that $G$ can approximate the linear interpolation $\alpha\tilde{G}_1 + (1-\alpha)\tilde{G}_2$. Hence, $G$ approximates the distribution of $\Qal$. Refer to SM~\ref{supp:prop:approx} to see the proof in detail.

Lastly, we present the minimax lower bound in the following Theorem~\ref{thm:lbd}. Suppose that ${\cal P}_{{\cal Q},[0,1]}$ is the family of Borel probability measures over the domain ${\cal Q}\times [0,1]$.
\begin{theorem}
\label{thm:lbd}
    Under (A1) and (A3), let $\hat{p}_{n_e}$ be any estimator of the target distribution $p_{\Qal,\alpha}$ constructed based on the $n_e$ size of 
    random samples. Then, 
    \begin{align}
    \label{eqn:lbd}
        & \inf_{\hat{p}_{n_e}} \sup_{p_{\Qal,\alpha} \in {\cal P}_{{\cal Q},[0,1]}}  P\left[d_{\cal D}(p_{\Qal,\alpha},\hat{p}_{n_e}) \geq \dfrac{C_{\text{LB}}}{\sqrt{n_e}} \right] > 0.55, 
    \end{align}
    where $C_{\text{LB}}=\log2|c(C_{X}^2)+c(B_X^2)+c(1-B_X^2)+c(-B_X^2)|/160$ with $C_{X}=\sqrt{B_X^2 + 1}$ and $c(x)=M_w(d)(\kappa_{d-1}(\cdots \kappa_1(M_w(1)x/C_{X})))$.
    % \begin{align*}
    %  C_{\text{LB},1}&= , \\
    % C_{\text{LB},2}&= M_w(d)\left(\kappa_{d-1}\left(\cdots (M_w(1)\dfrac{B_X^2 - B_{\alpha}^2}{\sqrt{B_X^2 + B_{\alpha}^2}}\right)\right), \\
    % C_{\text{LB},3}&= M_w(d)\left(\kappa_{d-1}\left(\cdots (-M_w(1)\dfrac{B_X^2 - B_{\alpha}^2}{\sqrt{B_X^2 + B_{\alpha}^2}}\right)\right), \\
    % C_{\text{LB},4}&= M_w(d)\left(\kappa_{d-1}\left(\cdots (-M_w(1)\sqrt{B_X^2 + B_{\alpha}^2}\right)\right). 
    % \end{align*}
\end{theorem}
\noindent 
Provided that the minimax convergence and approximation results for the original GAN model by \cite{ji:etal:21} and \cite{huan;etal;22} substitute $\sqrt{B_X^2+1}$ in \eqref{eqn:ae} and \eqref{eqn:lbd} for $B_X$, $n_e$ for $n/2$, and $\lceil\VW/2\rceil$ in \eqref{eqn:qpprox} for $\VW$, our parallel tempering structure might involve slightly higher errors. If $p_X$ is relatively simple (e.g., unimodal or mild multimodal), making it easier for GAN training to achieve global optimality, we acknowledge that the original training might be more efficient than ours. However, when $p_X$ is highly noisy and severely multimodal, GAN training is prone to falling into local optima, involving further unstable training, such as mode collapse. Then our parallel tempering technique offers stabilized gradients such that the GAN training reaches optimal equilibrium stably while still achieving the same minimax convergence rate $\sqrt{n}$ for global optimality.

% If $p_X$ is relatively simple (e.g., unimodal or mildly multimodal), standard GAN training may more efficiently reach global optimality. However, for highly noisy or severely multimodal $p_X$, GANs are prone to local optima and unstable behaviors like mode collapse. In such cases, our parallel tempering method stabilizes gradient estimates, enabling convergence to the optimal equilibrium while retaining the same minimax rate $\sqrt{n}$.

\section{Simulation Studies}
\label{sec:simul}
This section handles complex real-world datasets. For simpler targets like the 8-mixture distribution (Figure~\ref{fig:mode_collapse}), our method performs well, with results in SM~\ref{supp:mixup} (Figure~\ref{fig:mixture_comp}). To highlight the inherent improvement from our method, we minimally use extra training tricks rather than aim for state-of-the-art records. Evaluation scores are averaged over 10 independent runs, with standard deviations shown in parentheses. Simulation details, such as architectures, metrics, baselines, optimizers, etc, are provided in SM~\ref{supp:simulation}.

% This section focuses on handling complex real-world data sets. For simpler targets such as the  8-mixture distribution (Figure~\ref{fig:mode_collapse}),  our method easily succeeds, and we present related results in SM~\ref{supp:mixup} (Figure~\ref{fig:mixture_comp}). This work minimally uses extra training tricks to reveal inherent performance improvement of the proposed method rather than struggling to achieve new performance records. The evaluation scores are averaged across 10 independent runs, and their standard deviations are in parentheses. Simulation setups, such as network architectures, the definition of evaluation metrics, competing models, optimizers, etc. are detailed further in SM~\ref{supp:simulation}. 

\subsection{Data generation}
\label{sec:data_gen}

% {\color{red}
% \paragraph{Mixture Distribution} 
% tes the distributions $G$ recovers over the different training iteration. The figure implies that the PTGAN scheme more quickly captures the entire distribution and begins to represent each unimodal component.  
% \begin{figure}[ht!]
% \centering
% \includegraphics[width=1.0\textwidth]{images/ours_mixup.png}\vspace{-0.00in}
% \caption{Plots drawn with lighter colors depict the kernel density plots of generated distributions by $\Gt$ for the target distribution shown in the rightmost column.}
% \label{fig:mixture_comp}
% \end{figure}
% Add one more figure to display distributions at different $\alpha$.
% % }

\paragraph{Image Data Generation} 
We evaluate generative performance on {\bf CIFAR10}, {\bf BloodMnist}, and {\bf CelebA-HQ}. CIFAR10 contains $32\times 32\times 3$ images from 10 classes, while BloodMnist (from MedMNIST \citep{yang:etal:23}) consists of $64\times 64\times 3$ images across 8 blood cell types. CelebA-HQ provides $256\times 256\times 3$ high-quality celebrity images with 40 facial attributes. These classes or attributes induce multimodal $p_X$. We evaluate Inception Score (IS) and Fr\'echet Inception Distance (FID), computed via InceptionV3 pretrained on ImageNet \citep{szeg:etal:16}, and fine-tuned for single-label (BloodMnist) or multi-label (CelebA-HQ) tasks. Higher IS and lower FID indicate better performance.

% First, we check the performance of the generative model in the various image data sets: {\bf CIFAR10}, {\bf BloodMnist}, and {\bf CelebA-HQ}. CIFAR10 comprises $32\times 32\times 3$ images across 10 classes, and BloodMnist available in MedMNIST \citep{yang:etal:23} have $64\times 64\times 3$ images classified by 8 classes of blood cells. CelebA-HQ includes $256\times256\times3$ high-quality images of celebrities with 40 facial attributes. Intuitively, the different classes or attributes shape multimodal $p_X$. For performance evaluation in CIFAR10, the standard computer vision metrics, Inception Score (IS) and Fr\'echet Inception Distance (FID), are measured via the InceptionV3 model pre-trained on ImageNet \citep{szeg:etal:16}. The InceptionV3 model is further fine-tuned via single-label classification for BloodMnist and multi-label learning for CelebA-HQ, respectively. Higher IS and smaller FID indicate superior performance.

For CIFAR10 and BloodMnist, PTGAN is compared to generally applicable decent competitors. The spectral normalization \citep[SN,][]{miya:etal:18} frequently used in powerful models, e.g., StyleGAN-XL \citep{saue:etal:22}, is contrasted. As the strongest penalty-based GAN framework to our knowledge, the Lipschitz GAN \citep{zhou:etal:19}, imposing a maximum penalty (MP) of $D$'s gradient norm, is chosen as a competitor. %Also, the technique of adding a learnable noise $\epsilon$ to $X$ \citep[LN,][]{jenn:fava:19}, i.e., training with $p_{X+\epsilon}$, is considered. 
For fair comparison, the CNN-based structures of $D$ and $G$ used in \cite{miya:etal:18} are employed by all approaches, and $n_b$ is set to 100. PT and CP represent the proposed objective \eqref{builtin} and the coherency penalty \eqref{penalty} respectively. To investigate CP's effects, we test PT with MP and the common gradient penalty \citep[GP,][]{gulr:etal:17} only suitable for the scaled Wasserstein distance. For CIFAR10 and BloodMnist, we choose $r\in \{0.9,0.99\}$ that maximizes the evaluation metrics. Notably, PTGAN (PT+CP) defeats the competitors in the combinations of the two benchmark data sets and GAN metrics (Table~\ref{simul:cifar10}). In particular, PTGAN achieves notable IS/FID scores when coupled with CP. Table~\ref{tab:impact_r} in SM~\ref{supp:image_additional} shows the scores of $r\in \{0.9,0.99,1\}$, showing $r=1$ yields similar performance with MP.

For CelebA-HQ, we modify the CNN-based structures of $D$ and $G$ to adapt to the high-resolution images. $n_b$ is set to 50 for a feasible computation, and $r=0.98$, i.e., there is only one interpolated image in every minibatch. Table~\ref{simul:celeba} compares the Lipschitz GAN only since the Lipschitz GAN is already shown to be stronger than other methods in Table~\ref{simul:cifar10}. To evaluate FID tailored to CelebA-HQ, we fine-tune the InceptionV3 model to predict facial attributes simultaneously through multi-label learning. Since the concept of the IS metric is based on single-label classification, we instead present the original FID only as a reference.

The decent performance of our approach is supported by Figure~\ref{fig:variance_l}, showing significant variance reduction with $r < 1$ in all data sets when training with the ND metric. Figure~\ref{fig:celeba} visually qualifies generated PTGAN images for each case. To accommodate the page limitation, more illustrations are postponed to SM~\ref{supp:image_additional}.

\begin{table}[ht!]
\caption{Summary of IS/FID: %To show the general applicability of PTGAN, 
The GAN models are also trained with other popular metrics: the Jensen-Shannon divergence \citep[JSD,][]{good:etal:14} and the Pearson $\chi^2$-divergence \citep[PD,][]{mao:etal:17}. ND abbreviates the neural distance.}
\label{simul:cifar10}
\vspace{-0.1in}
\centering
\scriptsize
\begin{tabular}{c|c|cc|cc}
\hline
&& \multicolumn{2}{c|}{\textbf{CIFAR10}} & \multicolumn{2}{c}{\textbf{BloodMnist}}  \\ 
\hline
$d_{\cal D}$ & Type & IS ($\uparrow$) & FID ($\downarrow$) & IS ($\uparrow$) & FID ($\downarrow$) \\ 
\hline
\hline
\multirow{4}{*}{JSD} 
& SN   & 6.513 (0.350)   &  34.205 (4.563)  &  4.513 (1.050)  &  2578.92 (7383.72)    \\
% & LN  & 5.562 (0.213) &  60.949 (7.239)   &   &      \\
& MP   & 6.768 (0.081)   &  30.209  (0.550) &  5.103 (0.040) &  47.823 (1.237)    \\
& PT + MP &  6.727 (0.067) &  30.314 (0.663) &  5.102 (0.036)    &  47.131 (2.067)     \\
& PT + CP &  \bf{7.349 (0.110)} &   \bf{24.060 (0.815)} &  \bf{5.252 (0.086)}  &  \bf{41.390 (1.464)}  \\
\hline
\multirow{4}{*}{PD} 
& SN   & 6.611 (0.336)  &  33.959 (5.361) &  4.122 (1.539)   &  7948 (21084.409)  \\
% & LN   & 5.032 (1.983)  & 134.073 (173.329)  &    &      \\
& MP  & 6.850 (0.117)  &  29.563 (0.589) &  4.997 (0.045)   &  49.754 (1.145)    \\
& PT + MP  & 6.779 (0.082)  & 29.932 (0.669)  &  5.015 (0.027)    &  50.387 (1.316)    \\
& PT + CP  & \bf{7.429 (0.084)}  &  \bf{23.280 (0.883)}  &  \bf{5.208 (0.039)}   &  \bf{40.966 (1.201)}    \\
\hline
\multirow{6}{*}{ND}  
& SN   & 5.591 (0.198) & 45.868 (2.148)  &  3.711 (0.257)  &  176.040 (33.278)    \\
% & LN   & 4.513 (0.317) & 80.094 (5.243)  &    &      \\
& MP   & 6.929 (0.123) & 28.777 (1.010)  &   5.006 (0.034)   &  47.951 (1.164)    \\
& GP   & 6.797 (0.106) & 29.814 (0.933) &  4.967 (0.037)  & 51.636 (2.067)      \\
& PT + MP & 6.923 (0.089) &  28.422 (0.961)  &  4.997 (0.025)  &  49.0136 (0.821)    \\
& PT + GP & 6.767 (0.096) & 29.731 (0.661)  &  4.995 (0.052)  &  51.994 (1.114) \\
& PT + CP  & \bf{7.292 (0.090)} & \bf{24.838 (0.866)}  &  \bf{5.071 (0.058)}   &  \bf{41.990 (0.897)} \\
\hline
\end{tabular}
\end{table}

\begin{table}[ht!]
\caption{Summary of FID for \textbf{CelebA-HQ}: FID and MLL-FID are calculated from the pre-trained and the fine-tuned InceptionV3 model via multi-label learning (MLL), respectively.}
\label{simul:celeba}
\vspace{-0.1in}
\centering
\scriptsize
\begin{tabular}{c|c|cc}
\hline
$d_{\cal D}$ & Type & FID ($\downarrow$) & MLL-FID ($\downarrow$)  \\ 
\hline
\hline
\multirow{2}{*}{ND}  
& MP   & 26.859 (0.789) & 23.596 (0.584)  \\
& PT + CP & \bf{24.787 (1.054)} & \bf{20.164 (0.803)} \\
\hline
\end{tabular}
\end{table}

\begin{figure}[ht!]
\centering
\includegraphics[width=1.0\textwidth]{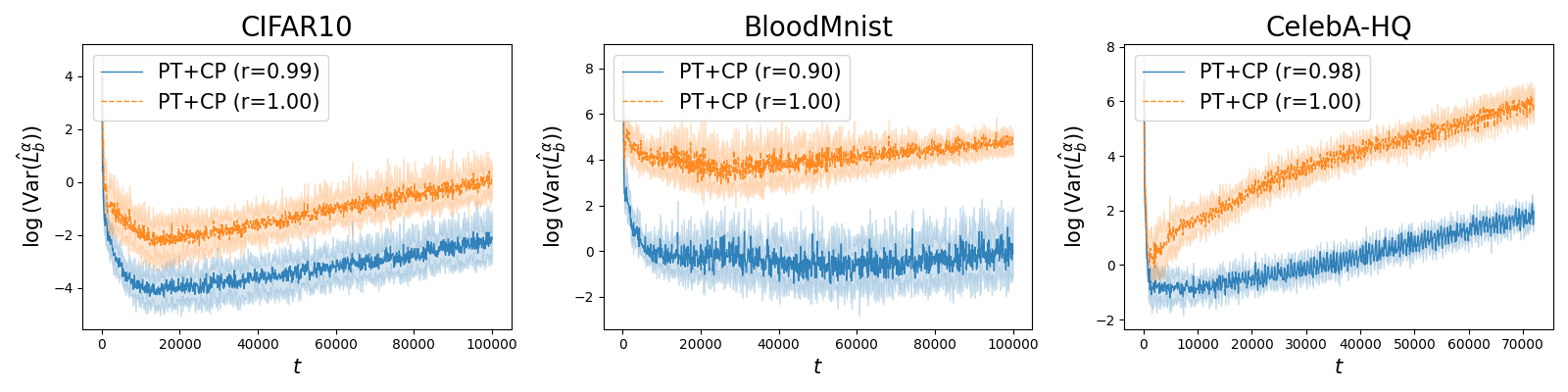}\vspace{-0.10in}
\caption{Variance reduction: the logarithm of $\var[\Lhat_b^{\alpha}]$ on ND over training iterations. Shaded areas indicate one standard deviation from the straight average lines.}
\label{fig:variance_l}
\end{figure}

\begin{figure}[ht!]
\centering
\begin{subfigure}[b]{1.00\textwidth}
    \includegraphics[width=\textwidth]{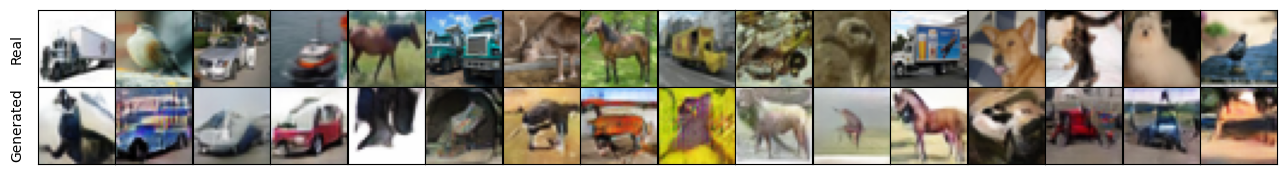}
    \vspace{-0.55in}
    \caption{16 Real/Generated images of CIFAR10}
\end{subfigure}
\begin{subfigure}[b]{1.00\textwidth}
    \includegraphics[width=\textwidth]{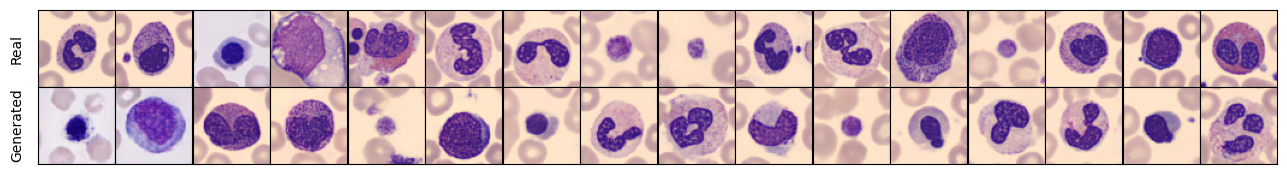}
    \vspace{-0.55in}
    \caption{16 Real/Generated images of BloodMnist}
\end{subfigure}
\begin{subfigure}[b]{1.00\textwidth}
    \includegraphics[width=\textwidth]{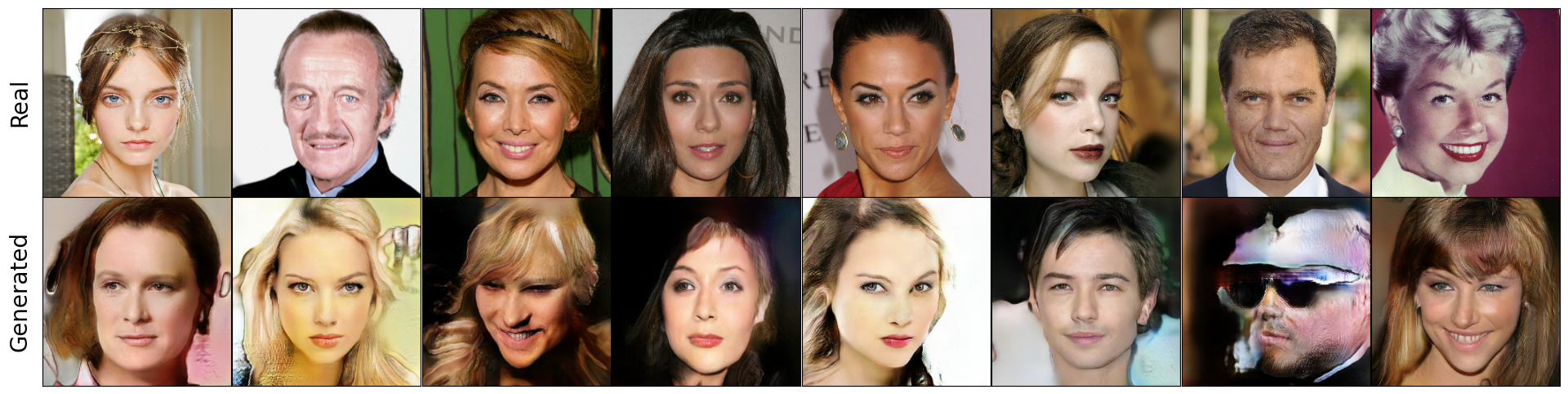}
    \vspace{-0.55in}
    \caption{8 Real/Generated images of CelebA-HQ}
\end{subfigure}
\vspace{-0.5in}
\caption{Real/Generated images are randomly picked from the original images and the generated images of PTGAN from the last iterate of $\Gt$ at $\alpha=1$}
\label{fig:celeba}
\end{figure}

\paragraph{Tabular Data Generation}  PTGAN's performance is assessed in generating tabular data for supervised learning, on three benchmark datasets: {\bf Adult} %\footnote{\label{data:uci}https://archive.ics.uci.edu/ml/datasets/} 
for income prediction, {\bf Credit Card Default} %\footref{data:uci} 
for default prediction, and {\bf Law School Admission} % \footnote{https://www.kaggle.com/datasets/danofer/law-school-admissions-bar-passage} 
for admission prediction. Each dataset is split 90\%/10\% for training ${\mathtt D}_{\text{train}}$ and test  data ${\mathtt D}_{\text{test}}$ respectively. GAN models with fully connected layers for both $D$ and $G$ are trained on ${\mathtt D}_{\text{train}}$ with ND. For PTGAN, $r$ is set to 0.5. We denote by ${\mathtt D}_t$ the output of $\Gt$ with 1k instances. For downstream evaluation at the $t$th iteration, random forest (RF), support vector machine (SVM), and logistic regression (LR) are trained on both ${\mathtt D}_{\text{train}}$ and $\mathtt{D}_t$. Then the area under the ROC curve (AUC) for these models is measured on ${\mathtt D}_{\text{test}}$. These AUC scores are denoted as ${\mathtt S}_{\text{train}}$ and ${\mathtt S}_t$ respectively. We calculate ${\mathtt S}_T=\sum_{t=\ceil*{T/2}+1}^T |{\mathtt S}_{\text{train}} - {\mathtt S}_t|/(T-\ceil*{T/2}+2)$ that implicitly evaluates the quick and accurate convergence of GAN models for the downstream task. Only MP is considered for comparison because of its superiority over other competitors in the previous section. Table~\ref{simul:tabgan} summarizes ${\mathtt S}_T$ from 10 independent runs, indicating PTGAN consistently outperforms MP across all datasets and predictive models. In SM~\ref{appen:tabgan}, Table~\ref{simul:tabgan_full} demonstrates PTGAN's superiority for JSD or PD as well, and Table~\ref{simul:tabgan_mp} summarizes ${\mathtt S}_T$ of MP with different penalty parameters but still defeated. 

\begin{table}[ht!]
\caption{Summary of ${\mathtt S}_T$ scores: Smaller scores are preferred.}
\label{simul:tabgan}
\vspace{-0.1in}
\centering
\scriptsize
\begin{tabular}{c|c|ccc}
\hline
Data & Type & RF ($\downarrow$)& SVM ($\downarrow$)& LR ($\downarrow$) \\ 
\hline
\hline
\multirow{2}{*}{Adult}
& PT + CP & \bfseries 0.022 (0.004) & \bfseries 0.038 (0.007) & \bfseries 0.028 (0.002)  \\
& MP & 0.047 (0.029) & 0.060 (0.025) & 0.050 (0.029)  \\
\hline
\multirow{2}{*}{Law School}
& PT + CP & \bfseries 0.018 (0.007) & \bfseries 0.024 (0.007) & \bfseries 0.006 (0.002) \\
& MP & 0.096 (0.018) & 0.099 (0.024) & 0.069 (0.023)  \\
\hline
\multirow{2}{*}{Credit Card}
& PT + CP & \bfseries 0.062 (0.008) & \bfseries 0.071 (0.018) & \bfseries 0.038 (0.010) \\
& MP & 0.159 (0.040) & 0.168 (0.047) & 0.147 (0.043)  \\
\hline
\end{tabular}
\end{table}

\subsection{Fair data generation}
\label{sec:tab_fairgen}

The intriguing property of PTGAN, learning $p_{\Qal,\alpha}$, can open up new generative modeling tasks. This work focuses on fair data generation, addressing the growing demand for morality control in machine learning. In algorithmic fairness, the goal is to reduce discrimination by decision models $h$ against certain subpopulations. We consider a classification setting with covariates $C \in {\cal C}$, binary sensitive attribute $A \in \{0,1\}$ (e.g., race or gender), and binary outcome $Y \in \{0,1\}$, where $h:{\cal C} \to [0,1]$ predicts $Y$. Fairness requires statistical independence between $h(C)$ and $A$, ensuring $h$ is unaffected by $A$. Discrimination is quantified by $\lvert\bE[\hat{Y}|A=1] - \bE[\hat{Y}|A=0]\rvert$ with $\hat{Y} = 1(h(C) > \tau)$, known as statistical or demographic parity (SP), though enforcing SP often compromises utility such as accuracy. See \cite{baro:etal:17, sohn:etal:23} for more details and recent advances.

%This work primarily focuses on fair data generation that has recently seen a growing demand for morality control in machine learning models. In algorithmic fairness, a key focus aims to minimize discrimination by decision models, denoted as $h$, against minor groups. We pay attention to a classification problem; $C\in {\cal C}$ represents covariates, $A\in\{0,1\}$ is a binary sensitive attribute (e.g., race or gender), $Y\in \{0,1\}$ is a binary outcome, and $h:{\cal C}\rightarrow [0,1]$ is learned to predict $Y$. The goal is to ensure statistical independence between $h(C)$ and $A$, so $h$ returns fair outcomes regardless of $A$. Discrimination is measured by $|\bE(\hat{Y}|A=1)-\bE(\hat{Y}|A=0)|$ where $\hat{Y}=1(h(C)>\tau)$ for some $\tau$. This measurement, called statistical/demographic parity (SP), tends to compromise the model's utility, such as classification accuracy. For more details and recent studies on algorithmic fairness, see \cite{baro:etal:17,sohn:etal:23} and reference therein.

Interestingly, the PTGAN framework can be used to enable $G$ to produce various levels of fair synthetic data while holding the training stability. Let's denote by $X_i^{(j)} = (C_i^{(j)}, j, Y_i^{(j)})$ the tuple of the $j$th group for $j=0,1$, and define $\check{X}^{\alpha}_i = \alpha X_i^{(0)} + (1-\alpha) X_j^{(1)}$ and $\check{X}^{1-\alpha}_i = (1-\alpha) X_i^{(0)} + \alpha X_j^{(1)}$. By creating $Q^{\alpha}$ as an equal mixture of $\check{X}^{\alpha}_i$ and $\check{X}^{1-\alpha}_i$, PTGAN can synthesize fair data sets with $\alpha$ measuring the level of fairness. To clarify the underlying fairness mechanism, let's consider the exemplary situations with the data points $(c_0,0,y_0)$ and $(c_1,1,y_1)$: $\text{(i)}~c_0=c_1, y_0=y_1$, $\text{(ii)}~c_0\neq c_1, y_0=y_1$, $\text{(iii)}~c_0=c_1, y_0\neq y_1$, and $\text{(iv)}~ c_0\neq c_1, y_0\neq y_1$. Case $\text{(iii)}$ explicitly contributes to discrimination, as $A$ directly determines $Y$; $\text{(i)}$ avoids discrimination. $\text{(ii)}$ and $\text{(iv)}$ permit dependency between $C$ and $A$, allowing discrimination through $C$. In this regard, $\Qal$ removes such dependency observed in $\text{(ii)}$, $\text{(iii)}$, and  $\text{(iv)}$ for $0<\alpha <1$ while preserving the relationship between $Y$ and $C$ in $\text{(i)}$.

\begin{figure}[t]
\centering
\includegraphics[width=1.0\textwidth]{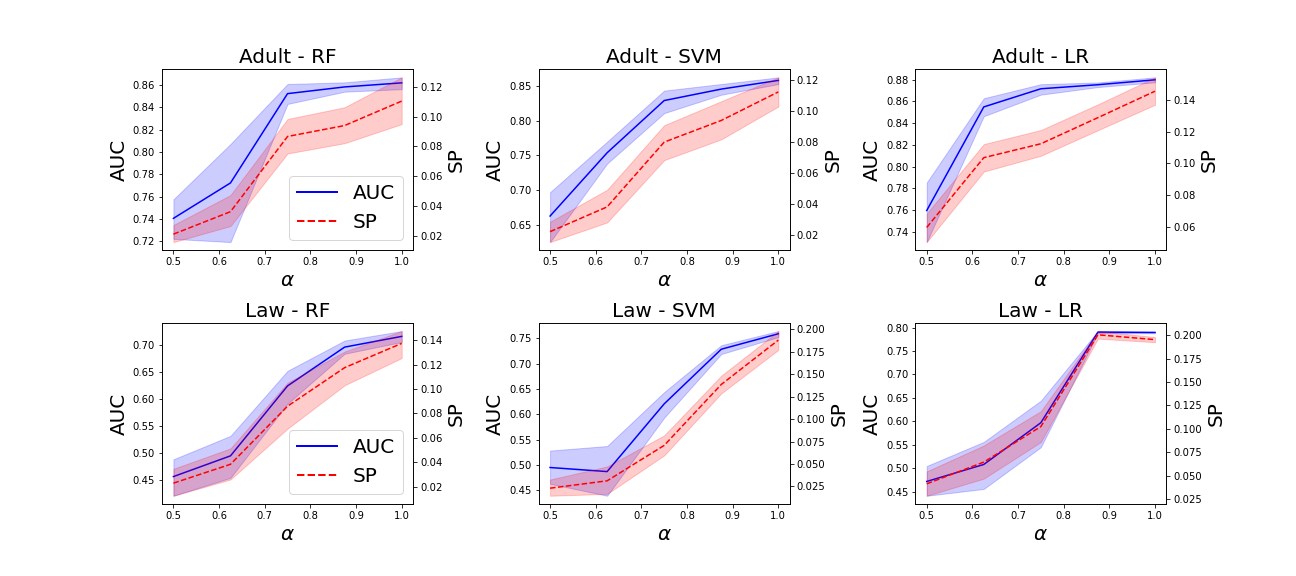}
\vspace{-0.6in}
\caption{The trade-off curves between AUC and SP for the downstream classifiers.}
\label{fig:uftradeoff}
\end{figure}

Our FairPTGAN framework is verified by comparing the behavior of trade-offs between utility and fairness to a recent fair generative model \citep[FairWGANGP,][]{raja:etal:22} and a seminar preprocessing method \citep[GeoRepair,][]{feld:etal:15}. FairWGANGP places a penalty term $\lambda_f \lvert\bE[\tilde{Y}|\tilde{A}=1]-\bE[\tilde{Y}|\tilde{A}=0]\rvert$ when updating $\Gt$ where $(\tilde{C},\tilde{A},\tilde{Y}) \sim \Gt(Z)$, so that $\Gt$ produces societally unbiased synthetic data. %$\lambda_f=10$ by referring to \cite{raja:etal:22}. 
GeoRepair solves a Wasserstein-median problem between  $C|A=1$ and $C|A=0$ with a weighting parameter $0\leq \lambda_p \leq 1$. For PTGAN, $r=0.2$ is set to encourage the generator to learn the intermediate (fair) distributions more effectively. For evaluation, the Pareto frontiers of AUC and SP are adopted as in \cite{sohn:etal:23}, where the cutting-off parameter $\tau$ is chosen to maximize AUC. Following the same evaluation procedure in tabular data generation, the Pareto frontiers are found by evaluating AUC and SP of the downstream models on the remaining 10\% test data. GeoRepair is applied to the FairPTGAN model with $\alpha=1$.

FairPTGAN is computationally efficient and achieves favorable Pareto frontiers. As shown in Table~\ref{tab:fair_comp}, it attains lower statistical parity (SP) than competitors at certain utility levels. Our joint learning structure in \eqref{builtin} enables the generator to produce datasets with different fairness levels by simply varying $\alpha$ in $\Gt(Z, \alpha)$. In contrast, FairWGANGP and GeoRepair require retraining or repeated processing when their fairness parameters $\lambda_f$, $\lambda_p$ change. While GeoRepair is model-free, its computational cost gets brutally expensive as the number of entities and variables in the data increases. Additionally, GeoRepair ignores the multivariate structure of $C$, which may lead to sacrificing too much utility.

% FairPTGAN proves computationally efficient and achieves favorable Pareto frontiers. Table~\ref{tab:fair_comp} indicates that ours have smaller statistical parity (SP) values than competitors for certain utility thresholds. We remark that the joint learning structure in \eqref{builtin} allows FairPTGAN's generator to produce synthetic datasets with different fairness levels just by varying $\alpha$ of $\Gt(\Zal,\alpha)$ (i.e., achieving a different balance for the fairness-utility trade-off). In contrast, FairWGANGP requires retraining when $\lambda_f$ changes. GeoRepair also has to repeat the processing whenever $\lambda_p$ changes as well. Although GeoRepair may take some advantages of being model-free, its computational cost gets brutally expensive as the number of entities and variables in data increases. Additionally, GeoRepair ignores the multivariate structure of $C$, which may lead to sacrificing too much utility.

\begin{table}[t]
\caption{Averages of the 10 smallest SP scores whose AUCs are greater than the thresholds ($\geq 0.85$ for Adult and $\geq 0.65$ for Law School) are reported. Table~\ref{tab:fair_comp2} in SM~\ref{appen:tab_fairgan} presents consistent results with different thresholds. }
\label{tab:fair_comp}
\centering
\scriptsize
\begin{tabular}{c|c|ccc}
\hline
Data & Model & RF ($\downarrow$)& SVM ($\downarrow$)& LR ($\downarrow$) \\ 
\hline
\hline
\multirow{3}{*}{Adult}  
& \multirow{1}{*}{FairPTGAN} & \bf{0.064 (0.006)} & \bf{0.077 (0.014)} & \bf{0.084 (0.006)} \\
& \multirow{1}{*}{FairWGANGP}& 0.083 (0.010) & 0.088 (0.010) & 0.095 (0.005) \\
&\multirow{1}{*}{GeoRepair}  & 0.082 (0.012) & 0.089 (0.009) & 0.106 (0.009) \\
\hline
\multirow{3}{*}{Law School.} 
& \multirow{1}{*}{FairPTGAN}  & \bf{0.054 (0.014)} & \bf{0.056 (0.006)} & \bf{0.079 (0.020)} \\
& \multirow{1}{*}{FairWGANGP} & 0.105 (0.006) & 0.115 (0.007) & 0.175 (0.003) \\
&\multirow{1}{*}{GeoRepair}   & 0.102 (0.011) & 0.129 (0.013) & 0.187 (0.003) \\
\hline
\end{tabular}
%\end{small}
\end{table}

% \begin{table}[ht!]
% \caption{Averages of the 10 smallest SP scores whose AUCs are greater than the thresholds (Adult: $\geq 0.85$, Law: $\geq 0.65$). Standard deviations are in the parentheses next to the averages. Baseline indicates SP on real data sets, and the scores are calculated from 10 independent runs.}
% \label{tab:fair_comp}
% \centering
% %\begin{small}
% \begin{tabular}{c|c|ccc}
% \hline
% Data & Approach & RandomForest & SVM & Logistic Reg. \\ 
% \hline
% \hline
% \multirow{4}{*}{Adult} 
% & \multirow{1}{*}{Baseline}   & 0.140 (0.021) & 0.136 (0.017) & 0.172 (0.019) \\ 
% \cline{2-5} 
% & \multirow{1}{*}{FairPTGAN} & \bf{0.064 (0.006)} & \bf{0.077 (0.014)} & \bf{0.084 (0.006)} \\
% & \multirow{1}{*}{FairWGANGP}& 0.083 (0.010) & 0.088 (0.010) & 0.095 (0.005) \\
% &\multirow{1}{*}{GeoRepair}  & 0.080 (0.009) & 0.084 (0.010) & 0.098 (0.012) \\
% \hline
% \multirow{4}{*}{Law}
% & \multirow{1}{*}{Baseline}   & 0.170 (0.017) & 0.210 (0.010) & 0.221 (0.017) \\ 
% \cline{2-5} 
% & \multirow{1}{*}{FairPTGAN}  & \bf{0.054 (0.014)} & \bf{0.056 (0.006)} & \bf{0.079 (0.020)} \\
% & \multirow{1}{*}{FairWGANGP} & 0.105 (0.006) & 0.115 (0.007) & 0.175 (0.003) \\
% &\multirow{1}{*}{GeoRepair}   & 0.088 (0.006) & 0.112 (0.012) & 0.182 (0.003) \\
% \hline
% \end{tabular}
% %\end{small}
% \end{table}

% \begin{figure}[t]
% \centering
% \includegraphics[width=1.0\textwidth]{images/util_fair_tradeoff_comp.png}\vspace{-0.00in}
% \caption{utility-fairness trade-off}
% \label{fig:uftradeoff_comp}
% \end{figure}

\section{Discussion}
\label{sec:discussion}

Recent work in generative modeling has focused on diffusion models, which often outperform GANs in various applications \citep{ho:etal:20, song:etal:23}. The main reason GANs lag behind is the long-standing challenge of balancing $D$ and $G$ having large and complex network architectures \citep{saue:etal:25}. Still, GANs offer key advantages such as fast sampling and flexible applicability. For instance, \cite{saue:etal:25} proposed a hybrid model that replaces the GAN generator with a diffusion model to speed up the sampling procedure. Leveraging the GAN framework, \cite{wang;etal;22} developed a Bayesian sampler for posterior inference, and \cite{zhou:etal:22} proposed a generative sampler for a conditional density estimation in a regression setting. In this context, our PTGAN framework, which stabilizes the variance of gradients in GAN training, can provide a promising direction for further advancing various generative models.

% The recent literature on generative modeling has focused on diffusion models since they have achieved remarkable performance even better than GAN models in various applications \citep{ho:etal:20,song:etal:23}. The main reason why GAN lags is due to the long-standing challenge of balancing $D$ and $G$ having large and complex network architectures \citep{saue:etal:25}. Nevertheless, the GAN's inherent benefits, i.e., fast sampling in the inference phase and methodological simplicity, remain greatly attractive. For example, \cite{saue:etal:25} proposed a hybrid structure - replacing the GAN's generator with a diffusion model in essence - to improve the inference speed of a diffusion model while preserving the quality of generated samples. \cite{wang;etal;22} developed a GAN-based Bayesian sampler for posterior analysis. In this context, we believe that our PTGAN framework, which improves GAN training by stabilizing the variance of gradients, makes potential room for both GAN and diffusion models to advance their better generation performance.

This work can be extended in several directions. First, the convex interpolation scheme could be replaced with advanced data augmentation techniques based on the interpolation structure \citep{shen:etal:24}. Interpolating more than three samples may also help capture a wider range of subpopulations to enhance fairness. Applying the parallel tempering framework to other generative models, such as a restricted Boltzmann machine or diffusion model, could further promote diversity in synthetic data. Refer to the extra discussion in SM~\ref{supp:extension} to see a possible extension. As noted in \cite{sada:etal:23}, diffusion models may face diversity issues, particularly with limited data or in conditional settings.

\bibliographystyle{apalike}
\bibliography{main}

\clearpage
\pagenumbering{arabic}

\include{supp.tex}

\bibliographystyleSupp{apalike}
\bibliographySupp{supp}

\end{document}

%% file: supp.tex
\appendix
\bigskip
\begin{center}
{\large\bf Supplementary Material (SM)}
\end{center}

\begin{center}
    {\bf Title: Parallelly Tempered Generative Adversarial Nets: \\ Toward Stabilized Gradients}
\end{center}

% \spacingset{1.9} % DON'T change the spacing!
\section{Additional Discussion}

\subsection{Training stability affected by gradients variance on $\Dt$}
\label{supp:noisy_gradient}
This section, we aims to empirically justify the importance of gradient variance of $\Dt$ by showing that a large gradient variance of $\Dt$ during the training can lead unstable GAN training. We choose the 8-mixture toy example (Figure~\ref{fig:mode_collapse}) and explicitly control the gradient variance by injecting additional Gaussian noise ${\rm N}(0,\sigma^2)$ to each gradient of $\Dt$ during training (Figure~\ref{fig:noisy_gradients}).
To heighten the effect of increasing gradient variance, this experiment is conducted for 
the Lipschitz GAN \citepSupp{zhou:etal:22}. Note that the Lipschitz GAN pursues more stable training than the original GAN by encouraging the critic function $\Dt$ to be Lipschitz via adding a gradient penalty, and it shows successful performance on the toy example when no additional noise is injected to the gradient. 

The experiment runs the GAN model 10 times independently. For each run, we evaluate the 1-Wasserstein distance\footnote{We use the Python library ``ot" to calculate the 1-Wasserstein distance approximately.} between the random samples of $p_X$ and $p_{\Gt(Z)}$ in training iterations. As shown in the figure, the GAN training with $\sigma=0.01$ shows a significantly fluctuating performance among 10 independent runs, which implies that GAN training with high variance of gradients on the critic can undergo difficult optimization for the generator. This experiment intentionally magnifies the variance of the $\Dt$'s gradients to observe that GAN training with high variance of gradients can fail. In practice, we hypothesize that there may be various sources causing high variance of gradients, such as mode collapse, misspecified hyperparameters, overfitting, etc.  
\begin{figure}[ht!]
    \centering
    \includegraphics[width=1.0\linewidth]{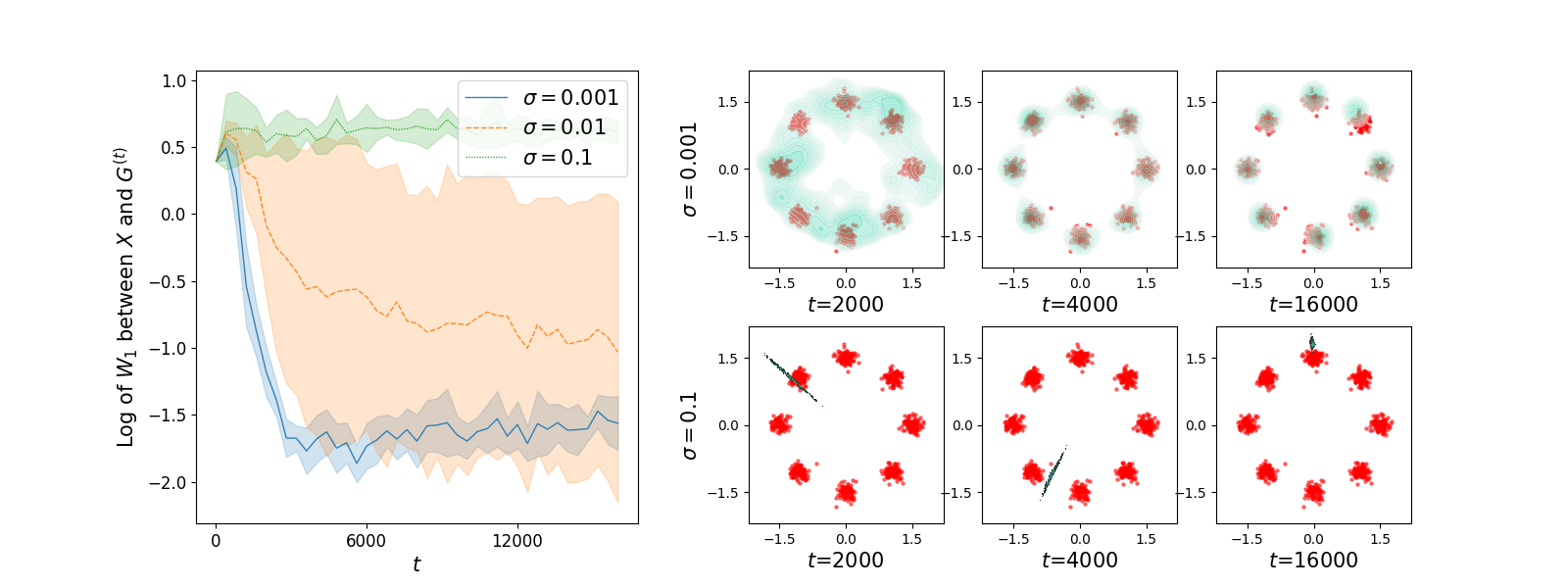}
    \caption{The left plot shows the logarithm of the 1-Wasserstein distance between $p_X$ and $p_{\Gt(Z)}$ of the Lipschitz GAN. Shaded areas represent one standard deviation from the average lines at every iteration. Note $\sigma$ stands for the size of the standard deviation of Gaussian noise added to the gradients of $\Dt$ during the training.}
    \label{fig:noisy_gradients}
\end{figure}

\subsection{Theorem~\ref{prop:grad_upbd1}}
\label{supp:disc_thm1}
First of all, the main theorem is restated for readers' convenience. After introducing the specific forms of the technical constants depending on the weight matrix's location $l$ and the type of activation function, we provide an intuitive explanation of the technical constants, particularly for the ReLU case in Remark~\ref{rem:interpre_relu} to appear later. The proof appears in Section~\ref{supp:prop_thm1}.
\begin{theorem*}
Under (A1-3), $\left\lvert \pr \Lhat_b(\Dt,\Gt)/\pr W_{l,r,c}\right\rvert/C_{\bw}^{(t)}(l)$ is bounded by
\begin{align*}
 \leq C_{\kappa,1}^{(t)}(l)d_{\cal D}(p_X,p_{\Gt(Z)}) +  
C_{\kappa,2}^{(t)}(l)d_{\kappa}(p_X,p_{\Gt(Z)}) + C_{\kappa,3}^{(t)}(l) +O_p\left(1/{\sqrt{n_b}}\right), %O_p\left(\dfrac{1}{\sqrt{n_b}}\right),
\end{align*}
for any $l,r,c$, where the constant $C_{\kappa,j}^{(t)}(l)$ for $j=1,2,3$ relies on the type of activation. The discrepancy $d_{\kappa}$ is the 1-Wasserstein distance $d_{W_1}$ or the total variation $d_{\text{TV}}$, respectively, depending on the Lipschitzness or boundness of the activation's derivative $\kappa'_l(x)$ for all $l$.
\end{theorem*}
To account for the constants more specifically, we first present the notations. 
\paragraph{Notation} Let $W_{l,r}$ (or $W_{l,r,\cdot}$) be the $r$th row vector of $W_l\in \mathbb{R}^{N^D_{l+1}\times N^D_l}$. Likewise, the $c$th column vector of $W_l$ is denoted by $W_{l,\cdot,c}$ or $W^{\top}_{l,c}$. Note $W_{l,r,c}$ is the $(r,c)$th parameter of $W_l$. $B_l(x)=W_lA_{l-1}(x)$ is the $l$th pre-activation layer and $A_{l-1}$ is the $(l-1)$th post-activation layer, i.e., $A_l(x)=\kappa_{l}(B_l(x))$, where $W_l\in \mathbb{R}^{N_{l+1}^D\times N_{l}^D}$, $A_{l-1}(x)\in \mathbb{R}^{N_{l}^D}$, and $B_l(x),A_l(x)\in \mathbb{R}^{N_{l+1}^D}$. For instance, $D(x)=w_{d}^{\top}A_{d-1}(x)$. Note the activation function applies element-wisely. We denote by $({\bf x})_r$ the $r$th component of a generic vector ${\bf x}$, i.e., $B_{l,r}(x)$ and $A_{l-1,c}(x)$ are the $r$th and $c$th pre/post-activation nodes, respectively. The derivative of $\kappa_l(x)$ is denoted by $\kappa_l'(x)$. We also define a technical term $p_{l,r,k_{l\sim d}}(x)=\prod_{j=l}^{d-1}\kappa_{j}'(B_{j,k_j}(x))$ with $k_{l}=r$ and for some index $k_{l+1},\dots,k_{d-1}$, i.e., $k_{l\sim d}=(r,k_{l+1},\dots,k_{d-1})$  and if $l=d-1$, $p_{l,r,k_{l\sim d}}(x)=\kappa'_{d-1}(B_{d-1,r}(x))$. We use these notations with the superscript $(t)$ when they are based on $t$th iterates $\Dt$ and $\Gt$, e.g., $A_{l-1,c}^{(t)}(x)$ is the $c$th post-activation node in the $(l-1)$th hidden layer of $\Dt(x)$. 

We provide the form of such constants across different $l$ and the type of activation. First of all, the $\Dt$'s capacity constant $C_{\bw}^{(t)}(l)$ during backpropagation appears as follows:
\begin{itemize}
    \item If $l=d$, $C_{\bw}^{(t)}(d)=1$; 
    \item If $l=d-1$, $C_{\bw}^{(t)}(d-1)=|w_{d,r}^{(t)}|$; 
    \item If $l=d-2$, $C_{\bw}^{(t)}(d-2)=\lVert W_{d-1,r,\cdot}^{(t)\top} \rVert \lVert w_d^{(t)}\rVert$; 
    \item If $l\leq d-3$, $C_{\bw}^{(t)}(l)=\lVert W_{l+1,r,\cdot}^{(t)\top} \rVert \prod_{j=l+2}^{d-1} \lVert W_{j}^{(t)} \rVert_F \lVert w^{(t)}_{d} \rVert$.
\end{itemize}
Secondly for $C_{\kappa,j}^{(t)}(l)$ ($j=1,2,3$), if $l=d$, regardless of the type of activation function, 
\begin{align*}
    C_{\kappa,1}^{(t)}(d)=1/M_w(d),\quad  \text{and} \quad C_{\kappa,2}^{(t)}(d)=C_{\kappa,3}^{(t)}(d)=0. 
\end{align*}
If $l\leq d-1$, then 
\begin{itemize}
    \item for the identity activation, $C_{\kappa,1}^{(t)}(l)=1/\prod_{j=l}^dM_w(j)$, $C_{\kappa,2}^{(t)}(l)=0$, and $C_{\kappa,3}^{(t)}(l)=0$;
    \item for any nonlinear activation functions, 
    \begin{align*}
        C_{\kappa,3}^{(t)}(l)&=\max_{k_{l\sim d}} |\text{Cov}(p_{l,r,k_{l\sim d}}^{(t)}(X_i),A_{l-1,c}^{(t)}(X_i))-\text{Cov}(p_{l,r,k_{l\sim d}}^{(t)}(G^{(t)}(Z_i)),A_{l-1,c}^{(t)}(G^{(t)}(Z_i))|.
    \end{align*}
    \item for the ReLU activation, 
    \begin{align*}
            C_{\kappa,1}^{(t)}(l)&=\dfrac{\max_{k_{l\sim d}}\lvert \bE[p_{l,r,k_{l\sim d}}^{(t)}(X_i)] + \bE[p_{l,r,k_{l\sim d}}^{(t)}(G^{(t)}(Z_i))] \rvert}{2\prod_{j=l}^d M_w(j)}, \\ 
            C_{\kappa,2}^{(t)}(l)&=\dfrac{\lvert \bE[A_{l-1,c}^{(t)}(X_i)]+\bE[A_{l-1,c}^{(t)}(G^{(t)}(Z_i))] \rvert}{2}; 
    \end{align*}
    \item for differentiable and non-decreasing activation satisfying $\kappa'_l(x)\geq C_{\kappa'}(l) > 0$ for all x, 
    \begin{align*}
            C_{\kappa,1}^{(t)}(l)&=\dfrac{\max_{k_{l\sim d}}\lvert \bE[p_{l,r,k_{l\sim d}}^{(t)}(X_i)] + \bE[p_{l,r,k_{l\sim d}}^{(t)}(G^{(t)}(Z_i))] \rvert}{2\prod_{j=l}^d M_w(j)C_{\kappa'}(j)}, \\ 
            C_{\kappa,2}^{(t)}(l)&=\dfrac{\lvert \bE[A_{l-1,c}^{(t)}(X_i)]+\bE[A_{l-1,c}^{(t)}(G^{(t)}(Z_i))] \rvert}{2}\prod_{j=l}^{d-1}K_{\kappa}(j).
    \end{align*}
    Note the existence of the lower bound constant is justified in Remark~\ref{rem:kappa}. For instance, there are Sigmoid, Tanh, ELU ($\alpha=1$) activation functions. 
\end{itemize}
Finally, $d_{\kappa}$ relies on the choice of activation as well: 
\begin{itemize}
    \item if the derivative is not Lipschitz (e.g., ReLU), $d_{\kappa}=d_{\text{TV}}$; 
    \item if the derivative is Lipschitz, $d_{\kappa}=d_{W_1}$. 
\end{itemize}

\begin{remark}
\label{rem:interpre_relu}
    To simplify the discussion, we focus on $l=d-1$ for the ReLU case, but a similar explanation can be made for $l\leq d-2$. Since the ReLU activation $\kappa(x)=\max\{x,0\}$ is not continuous at $x=0$, we can observe 
    \begin{align*}
        C_{\kappa,1}^{(t)}(d-1)&=\dfrac{\bE[\kappa'_{d-1}(B_{d-1,r}^{(t)}(X_i))]+\bE[\kappa'_{d-1}(B_{d-1,r}^{(t)}(G(Z_i))]}{2M_w(d)M_w(d-1)}, \\ 
        &\bE[\kappa'_{d-1}(B_{d-1,r}^{(t)}(x))]=P(B_{d-1,r}^{(t)}(x)>0),
    \end{align*}
    i.e., $P(B_{d-1,r}^{(t)}(x)>0)$ implies the probability of the $r$th node being activated when the initial input is $x$. Hence, $C_{\kappa,1}^{(t)}(d-1)$ becomes larger as the hidden nodes are more likely to be activated. The second constant $C_{\kappa,2}^{(t)}(d-1)=(\bE[A_{d-2,c}^{(t)}(X_i)]+\bE[A_{d-2,c}^{(t)}(G(Z_i)])/2$ becomes larger in accordance with the size of post-activation node. Being aware of $A_{d-2,c}^{(t)}(x)=\kappa_{d-2}(B_{d-2,c}^{(t)}(x))$, applying a normalization technique, e.g., Batch Normalization, to the pre-activation node would help control the size of $C_{\kappa,2}^{(t)}(d-1)$ to a moderate extent. In the third constant $C_{\kappa,3}^{(t)}(d-1)=|\text{Cov}(\kappa'_{d-1}(B^{(t)}_{d-1,r}(X_i)),A_{d-2,c}^{(t)}(X_i))-\text{Cov}(\kappa'_{d-1}(B_{d-1,r}^{(t)}(G^{(t)}(Z_i))),A_{d-2,c}^{(t)}(G^{(t)}(Z_i))|$, the covariance $\text{Cov}(\kappa'_{d-1}(B^{(t)}_{d-1,r}(x)),A_{d-2,c}^{(t)}(x))$ can be seen to represent the degree of information alignment between the $r$th pre-node in $(d-1)$th hidden layer and the $c$th post-node in the $(d-2)$th layer, e.g., the covariance would be negligible if the $c$th post-node does not contribute much to the $r$th pre-node. In general, $C_{\kappa,3}^{(t)}(d-1)$ tends to vanish as $p_X\approx p_{\Gt(Z)}$. 
\end{remark}

\subsection{Variance reduction within a linear class}
\label{supp:ad_var_linear}
As a more concrete example, we further investigate the variance reduction mechanism under the linear function class ${\cal D}$. Let's consider a linear critic function $D(\Qal,\alpha)=W_1^{\top}[\Qal,\alpha]$ with $W_1 \in \mathbb{R}^{(d_X + 1) \times 1}$ and $\Gt(\alpha Z_1 + (1-\alpha) Z_2, \alpha)\overset{d}{=} \alpha \Gt(Z_1,1) + (1-\alpha) \Gt(Z_2,1)$, which means $\Gt$ is simultaneously converging to the equilibrium for all $\alpha$. Here, we use the interpolated input for $G$ (See Section~\ref{supp:inter_noise}). The below proposition shows when the gradients' variance reduction occurs under the verifiable assumption.
\begin{proposition}
\label{prop:var_linear}
    Suppose $D$ is linear, $n_b=m_b$, and $\Gt(\alpha Z_1 + (1-\alpha) Z_2, \alpha)\overset{d}{=} \alpha \Gt(Z_1,1) + (1-\alpha) \Gt(Z_2,1)$ with $\alpha \sim r \delta_1(\cdot) + (1-r) p_{\alpha^*}(\cdot)$. Then, $\text{tr}\left(\cov\left(\frac{\partial \hat{L}^{\alpha}_b(\Dt,\Gt)}{\partial W_1}\right)\right)$ is equal to 
    \begin{align*}
        \left(\dfrac{2}{3}+\dfrac{1}{3}r\right)\text{tr}\left(\cov\left(\dfrac{\partial \hat{L}_b^1(\Dt,\Gt)}{\partial W_1}\right)\right)+\var(\alpha)\left(\dfrac{1}{n_b}+\dfrac{1}{m_b}\right).
    \end{align*}
\end{proposition}
\noindent This proposition shows that $\text{tr}\left(\cov\left(\frac{\partial \hat{L}_b^{\alpha}(\Dt,\Gt)}{\partial W_1}\right)\right)\leq \text{tr}\left(\cov\left(\frac{\partial \hat{L}_b^1(\Dt,\Gt)}{\partial W_1}\right)\right)$ holds for any $r$ if $\text{tr}(\cov(X_1))+\text{tr}(\cov(\Gt(Z,1)))\geq 3\var(\alpha)$ is satisfied, where the equality only holds with $r=1$. The variance reduction effect tends to be stronger as $r\rightarrow 0$. % and maximally decreases by $\frac{2}{3}$ factor. 
The extra assumption demands that the randomness of $\alpha$ does not have to dominate the randomness from $X_1$ and $\Gt(Z,1)$ while $0\leq\var(\alpha)\leq 1/9$. This assumption can be usually satisfied in deep learning applications; it is a convention to standardize the input space such that $\max(X)=1$ and $\min(X)=-1$ (or $\min(X)=0$) for efficient optimization. Also, the input dimension $d_X$ is usually large, e.g., CIFAR10 (the benchmark data set in Section~\ref{sec:simul} with $d_X=32\times 32\times 3$).

\subsection{Comparison to Mixup}
\label{supp:mixup}
The idea of interpolating data points was first introduced by \citetSupp{zhan:etal:17}. They mainly discussed that the use of the convex combinations, so-called \emph{Mixup}, greatly improves generalization errors and robustness against adversarial testing data within the supervised learning framework. Based on the idea of Mixup, there has been a strand of research designing better ``mixed" data augmentation, mostly focusing on computer vision tasks \citepSupp[e.g., ][]{yun:etal:19, verm:etal:19, hend:etal:19}. It is worth mentioning that the original work of \citetSupp{zhan:etal:17} also briefly discussed applying the Mixup technique for GAN training by introducing linear combinations of real and generated data points. Despite the similarity of Mixup and our convex mixture (\ref{def:convex}), there are fundamental differences: Mixup technique serves as a penalization that aims to regularize and smooth the optimization objective and hence to improve the generalization and robustness; in contrast, our usage of convex combination doesn't change the optimization objective but creates auxiliary intermediate distributions that helps stabilize and accelerate the original GAN training. Finally, the Mixup strategy only applies to a specific type of $d_{\cal D}$, whereas ours is universally applicable to most probability metrics.

Figure~\ref{fig:mixture_comp} compares the Mixup GAN (MixGAN) and our approach (PTGAN) for the toy example where the original training fails (Figure~\ref{fig:mode_collapse}). For MixGAN implementation, we consider two hyperparameters for the label distribution defined by a $\text{Beta}(\alpha,\beta)$ distribution. The figure illustrates the density plot of $G^{(t)}$ over the different training iterations, where the shaded region from one standard deviation is found based on 10 independent runs. The figure implies that PTGAN more quickly captures the entire distribution and begins to represent all unimodal components than MixGAN.

\begin{figure}[ht!]
\centering
\includegraphics[width=1.0\textwidth]{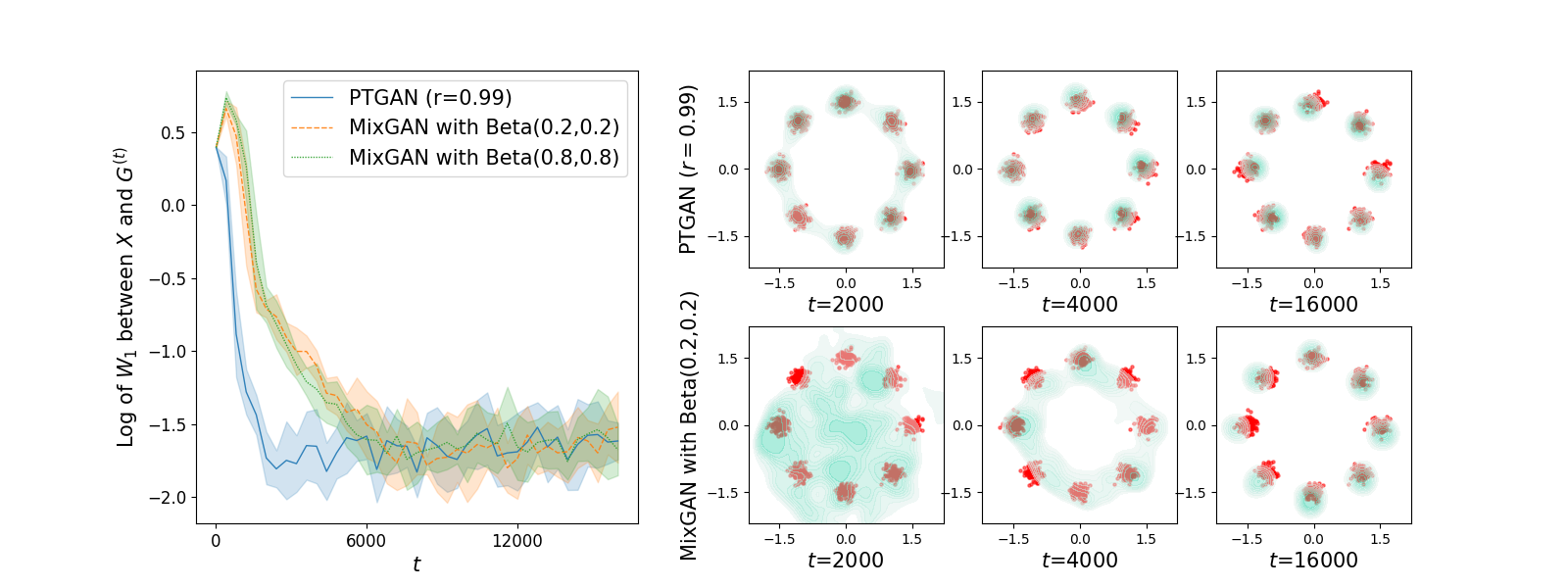}\vspace{-0.00in}
\caption{The left plot shows the logarithm of the 1-Wasserstein distance between the data and generated data for ours (PTGAN) and Mixup GAN (MixGAN). Plots draw the kernel density plots of generated distributions by $\Gt$ over the target distribution (red dots).}
\label{fig:mixture_comp}
\end{figure}

\paragraph{Simulation setup} The toy data used to show Figure~\ref{fig:mode_collapse} is considered. Also, the network architectures of $D$ and $G$ in the figure are used. The Adam optimizer's hyperparameters are set to $\beta_1=0.0$ and $\beta_2=0.9$ with the learning rates for $D$ and $G$ as 0.0001.

\subsection{Extension to a RBM model}
\label{supp:extension}
Our tempering scheme can be technically extended to a deep belief network model. For concise discussion, let's consider a restricted Boltzmann machine (RBM) model, i.e., 
\begin{align*}
        p(x|h)&\propto \exp\left(-E(x,h)\right),\\
        E(x,h)&=-\sum_{i}x_ia_i-\sum_{i,j}x_iW_{ij}h_j-\sum_j h_j b_j,
\end{align*}
where $x=(x_1,\dots,x_p)$ and $h=(h_1,\dots,h_K)$ for all $x_i,h_j\in\{0,1\}$, and $W,b,c$ are parameters. The marginal distribution is described by 
\begin{align*}
    p(x)=\sum_{h\in \{0,1\}^K} p(x|h)p(h), 
\end{align*}
which can be seen as a distribution with $2^K$ components. Therefore, as Figure~\ref{fig:tempered_dist} in the manuscript hints, taking convex interpolation in the input space can also be effective in training a DBM model if $2^K$ modes in the latent space makes $2^K$ distinguished distributions in the space of $x$. 

In this case, by adopting a Gaussian-Bernoulli RBM, we can adapt the real-valued inputs $v=(q_{\alpha},\alpha)\in \mathbb{R}^{p+1}$ where  $q_{\alpha}=\alpha x_1 + (1-\alpha) x_2$ and $\alpha \sim p_{\alpha}$ in the place of $x$, i.e., 
  \begin{align*}
        p(q_{\alpha},\alpha|h)&=p(v|h)\propto \exp\left(-E_{\rm{GB}}(v,h)\right), \\
        E_{\rm{GB}}(v,h)&=\sum_{i}\frac{(v_i-b_i)^2}{2\sigma_i^2} - \sum_{i,j}\frac{v_i}{\sigma_i}W_{ij}h_j-\sum_{j}c_jh_j.
    \end{align*}
Once this RBM model is trained, one can generate samples by fixing $\alpha=1$. However, for successful training and sampling procedures, it may be required to devise additional optimization techniques to harness the interpolation scheme more effectively in the training of the RBM model, such as our coherency penalty.

\newpage

\section{Proof}
\label{supp:proof}

\subsection{Proposition~\ref{prop:nd_wd}}
\label{supp:prop1}

Suppose that two distributions $X\sim p_1$ and $Y\sim p_2$ are defined on the compact ${\cal X}$ with finite second moments in $\mathbb{R}^{d_X}$. Under (A1) and (A2), for any coupling $p_{1,2}$ whose marginal densities $p_{1}$ and $p_2$, $\bE_{p_{1,2}} [D(X)] - \bE_{p_{1,2}} [D(Y)]=\bE_{p_{1}} [D(X)] - \bE_{p_{2}} [D(Y)]$ holds, therefore,  
\begin{align*}
    \expect_{p_1} [D(X)] - \expect_{p_2} [D(Y)] &\leq \int |D(x)-D(y)|p_{12}(x,y)dxdy, \\
    & \leq \prod_{s=1}^d M_w(s) \prod_{u=1}^{d-1} K_{\kappa}(u) \int |x-y| p_{12}(x,y) dxdy, \\ 
    & = \prod_{s=1}^d M_w(s) \prod_{u=1}^{d-1} K_{\kappa}(u) \expect [\lVert X-Y \rVert],
\end{align*}
by the Cauchy-Schwarz inequality, and it implies
\begin{align*}
 d_{\cal D}(p_1,p_2)\leq \prod_{s=1}^d M_w(s) \prod_{u=1}^{d-1} K_{\kappa}(u) \times W_1(p_1, p_2),
\end{align*}
where $W_k(p_1, p_2)$ is the $k$th-order Wasserstein distance between $p_1$ and $p_2$. 

For the lower bound, let's denote by ${\cal L}$ the class of 1-Lipschitz continuous functions. It is well-known that neural networks have the universal approximation property for $L^{\infty}$ norm under the compact domain ${\cal X}$ or for $L^p$ norm \citepSupp{lu:etal:17,park:etal:20}. Thus there is an approximation error $\omega_{\cal D}/2>0$ of ${\cal D}\cap {\cal L}$ to ${\cal L}$ characterized by the structure of ${\cal D}$, i.e., for any function $f\in {\cal L}$, there always exists a network $D\in {\cal D}\cap {\cal L}$, such that $|f(x)-D(x)|\leq \omega_{\cal D}/2$ for all $x\in \mathcal X$. Such an approximation holds due to the universal approximation properties of neural network \citepSupp[e.g.,][]{park:etal:20}. Trivially, this implies that $|d_{{\cal D}\cap {\cal L}}(p_1,p_2) - d_{\cal L}(p_1,p_2)| < \omega_{\cal D}$. Since $W_1(p_1,p_2)= d_{\cal L}(p_1,p_2)$, we obtain the bound, $W_1(p_1,p_2) - \omega_{\cal D}\leq d_{{\cal D}\cap {\cal L}}(p_1,p_2)\leq d_{\cal D}(p_1,p_2)$, concluding 
\begin{align*}
    W_1(p_1,p_2) - \omega_{\cal D}\leq d_{\cal D}(p_1,p_2) \leq \prod_{s=1}^d M_w(s) \prod_{u=1}^{d-1} K_{\kappa}(u) \times W_1(p_1, p_2).
\end{align*}
% Since the support is bounded, it is known that $W_1(p_1,p_2)\leq W_k(p_1,p_2)\leq \mathrm{diam}({\cal X})^{\frac{k-1}{k}}W_1(p_1,p_2)^{\frac{1}{k}}$ for any $k\geq 1$ where $\mathrm{diam}({\cal X})=\sup_{x,y}\lVert x-y \rVert$. 
In the meantime, by the Jensen's inequality, $\lVert \mu_1-\mu_2\rVert=\lVert \bE[X-Y] \rVert \leq \bE[\lVert X-Y\rVert]$, so, for any coupling between $p_1$ and $p_2$, $\lVert \mu_1-\mu_2\rVert \leq W_1(p_1,p_2)$. Since $W_2(p_1,p_2)=\inf_{\pi}(\int \lVert X-Y \rVert^2d\pi)^{1/2}$, we consider an independent coupling to see 
\begin{align*}
    W_1(p_1,p_2)\leq W_2(p_1,p_2)\leq (\bE[\lVert X-Y\rVert^2])^{1/2}&=\sqrt{\mathrm{Tr}(\mathrm{Cov}(X)+\mathrm{Cov}(Y)))+\lVert \mu_1-\mu_2 \rVert ^2}, 
\end{align*}
where $W_1(p_1,p_2)\leq W_2(p_1,p_2)$ holds by the Jensen's inequality. Therefore, for the univariate case, by setting $\sigma_1=\sigma_2=\sigma$ for simplicity, we have 
\begin{align*}
    |\mu_1-\mu_2|-\omega_{\cal D}\leq d_{\cal D}(p_1,p_2)\leq \prod_{s=1}^d M_w(s) \prod_{u=1}^{d-1} K_{\kappa}(u) \sqrt{2\sigma^2+(\mu_1-\mu_2)^2}.
\end{align*}

% \paragraph{Proof for the other bounds}
% Then, by the Jensen's inequality, $(\mu_1-\mu_2)^2=(\bE[X-Y])^2\leq \bE[(X-Y)^2]$. Therefore, for any coupling $\gamma$ between $X$ and $Y$,  
% \begin{align*}
%     (\mu_1-\mu_2)^2\leq \inf_{\gamma} \bE[(X-Y)^2]=W_2^2(p_1,p_2).
% \end{align*}
% Note there exists a tighter bound 
% \begin{align*}
%     (\sigma_1-\sigma_2)^2+(\mu_1-\mu_2)^2\leq \inf_{\gamma} \bE[(X-Y)^2]=W_2^2(p_1,p_2).
% \end{align*}

\subsection{Proposition~\ref{prop:grad_lwbd}}

The argument in the main text is based on the following two lemmas.
\begin{lemma}
\label{lemma:cov_lwb}
    Assume $\lVert w_d^{(t)}\rVert>0$. The norm of the covariance of the $w_d$'s gradient is bounded below
    \begin{align*}
            \left\lVert \text{Cov}\left(\dfrac{\partial\hat{L}_b(\Dt,\Gt)}{\partial  w_d}\right)\right\rVert_{2} \geq \dfrac{\var(\hat{L}_b(\Dt,\Gt))}{\lVert  w_d^{(t)} \rVert^2},
    \end{align*}
    where the $\lVert \cdot\rVert_2$ for the covariance matrix is the induced 2-norm. 
\end{lemma}
\begin{proof}
    Let $\Lhat_i = D(X_i) - D(G(Z_i))$ and  $\hat{L}_{d-1,i}=A_{d-1}(X_i) - A_{d-1}(G(Z_i))$ where $A_{d-1}$ is the output of the $(d-1)$th post-activation layer. Since $\hat{L}_i=w_d^{\top} \hat{L}_{d-1,i}$ and accordingly $\partial \hat{L}_i / \partial w_d=\hat{L}_{d-1,i}$,  
\begin{align*}
    \var(\hat{L}_i) &= w_d^{\top} \text{Cov}(\hat{L}_{d-1,i}) w_d, \\
    \iff & \dfrac{\var(\hat{L}_i)}{n_b} = w_d^{\top} \dfrac{\text{Cov}(\hat{L}_{d-1,i})}{n_b} w_d\leq \lVert w_d \rVert \times \left \lVert \cov\left(\frac{\partial \hat{L}_b(D,G)}{\partial w_d}\right) w_d \right \rVert,\\ 
    \Rightarrow & \dfrac{\var(\hat{L}_b(D,G))}{\lVert w_d \rVert^2}\leq  \left \lVert \cov\left(\frac{\partial \hat{L}_b(D,G)}{\partial w_d}\right) w_d \right \rVert / {\lVert w_d \rVert},
\end{align*}
by the Cauchy–Schwarz inequality, where $\var(\hat{L}_b(D,G))=\var(\hat{L}_i)/n_b$ and $\cov\left(\partial \hat{L}_b(D,G)/\partial w_d\right)=\cov(\hat{L}_{d-1,i})/n_b$ with $n_b=m_b$ under the i.i.d. assumption. Therefore, by applying the definition of the induced norm to the right-hand side, we obtain the result. 
\end{proof}

\begin{lemma}
\label{lemma:var_lwb_i}
    Define $\zetaeps = P(\Lti\leq \epsilon)$,  $\mu_{\calS}=\bE(\Lti| \Lti \leq \epsilon)$, $\sigma_{\calS}^2=\var(\Lti|\Lti\leq \epsilon)$, and $\sigma_{\calS^c}^2=\var(\Lti|\Lti> \epsilon)$. The variance of $\Lti$ is characterized by 
    \begin{align*}
    \var(\Lti)= \zetaeps\sigma_{\calS}^2 +(1-\zetaeps)\sigma^2_{\calS^c}+ \zetaeps(1-\zetaeps)(\bE(\Lti|\Lti\geq \epsilon)-\mu_{\calS})^2. 
    \end{align*}
\end{lemma}
\begin{proof}
    Recall $\Lti=\Dt(X_i)-\bE[\Dt(\Gt(Z))]$ and $\zetaeps = P(\Lti\leq \epsilon)$. Let's denote by $A$ the classification rule such that $A_1:\{ \Lti  \leq \epsilon\}$ and $A_2:\{ \Lti > \epsilon\}$.  We define $\mu_{\calS}=\bE(\Lti|A_1)$ and $\sigma_{\calS}^2=\var(\Lti|A_1)$. Then, by the Jensen's inequality, we have
\begin{align*}
    \bE(\var(\Lti|A)) &= P(A_1)\var(\Lti|A_1) + P(A_2)\var(\Lti|A_2), \\
    &= \zetaeps\var(\Lti|A_1) + (1-\zetaeps)\var(\Lti|A_2),\\
    &= \zetaeps\sigma_{\calS}^2+(1-\zetaeps)\sigma^2_{\calS^c}.
\end{align*}
On the one hand, 
\begin{align*}
    \var(\bE(\Lti|A)) &= \bE(\bE(\Lti|A)^2) - \bE(\bE(\Lti|A))^2,\\
    &= P(A_1)\bE(\Lti|A_1)^2 + P(A_2)\bE(\Lti|A_2)^2 - (P(A_1)\bE(\Lti|A_1)+P(A_2)\bE(\Lti|A_2))^2,\\
    &= P(A_1)\bE(\Lti|A_1)^2 - P(A_1)^2\bE(\Lti|A_1)^2 + P(A_2)\bE(\Lti|A_2)^2 -P(A_2)^2\bE(\Lti|A_2))^2 \\ 
    &~ -2P(A_1)P(A_2)\bE(\Lti|A_1)\bE(\Lti|A_2), \\
    % &=  \\
    % &= \zetaeps(1-\zetaeps)^2(\mu_{\calS}-\bE(\Lti|A_2))^2 + \zetaeps^2(1-\zetaeps)(\mu_{\calS}-\bE(\Lti|A_2))^2,\\
    &=\zetaeps(1-\zetaeps)(\bE(\Lti|A_2)-\mu_{\calS})^2.
\end{align*}
Therefore, 
\begin{align*}
    \var(\Lti)&=\bE(\var(\Lti|A)) + \var(\bE(\Lti|A)), \\
    & =\zetaeps\sigma_{\calS}^2 +(1-\zetaeps)\sigma^2_{\calS^c} + \zetaeps(1-\zetaeps)(\bE(\Lti|A_2)-\mu_{\calS})^2.
\end{align*}
\end{proof}

Suppose $\lVert w_d^{(t)}\rVert>0$. Let $\zetaeps = P(\Lti\leq \epsilon)$,  $\sigma_{\calS}^2=\var(\Lti|\Lti\leq \epsilon)$, $\sigma^2_{\Gt}=\var(\Dt(\Gt(Z_j)))$, and $\bE(\Lti| \Lti \leq \epsilon)=0$ for some $\epsilon>0$. Since the variance of $\hat{L}_b(\Dt,\Gt)$ is bounded below
\begin{align*}
    \var\left(\hat{L}_b(\Dt,\Gt)\right) &= \var\left(\dfrac{1}{n_b}\sum_{i=1}^{n_b} \Lti \right) \\ 
    & ~~ + \var\left(\dfrac{1}{m_b}\sum_{j=1}^{m_b} (\bE(\Dt(G^{(t)}(Z)))-\Dt(G^{(t)}(Z_j)))\right),  \\
    &\geq \dfrac{1}{n_b}\left(\zetaeps\sigma_{\calS}^2 +(1-\zetaeps)\sigma^2_{\calS^c}+ \zetaeps(1-\zetaeps) \bE(\Lti|A_2)^2 \right) + \dfrac{\sigma^2_{\Gt}}{m_b},
\end{align*}
by Lemma~\ref{lemma:var_lwb_i} and, therefore, by Lemma~\ref{lemma:cov_lwb}, we have the statement.

\subsection{Proposition~\ref{prop:reduction}}
Let's consider the 2-mixture example discussed in Section~\ref{sec:multimodal}, i.e., $p_X(x)=\sum_{k=1}^2 p_k(x;\mu_k,\sigma)/2$. Let's denote by $u\sim p_1(x;\mu_1,\sigma)$ and $v\sim p_2(x;\mu_2,\sigma)$ each mixture component. Suppose that $\Qal$ follows a mixture distribution $\Qal \sim \sum_{k=1}^2 p_k^{\alpha}(x;\mu_k^*,\sigma^*)/2$. In this work, we define each $p_k^{\alpha}$ based on 
\begin{align*}
    &\alpha u + (1-\alpha)X_2 \sim p_1^{\alpha} \quad \text{w.p.} \quad 0.5, \\
    &\alpha v + (1-\alpha)X_2 \sim p_2^{\alpha} \quad \text{w.p.} \quad 0.5,
\end{align*}
for $\alpha\sim{\rm Unif}(0.5,1)$. In this construction, $\Qal$ has the same distribution with $\alpha\sim {\rm Unif}(0,1)$. This construction helps illustrate $p_1^{\alpha}$ and $p_2^{\alpha}$ are separated unimodal distributions as shown in Figure~\ref{fig:separability} that visually compares $p_k$ and $p_k^{\alpha}$. 

\begin{figure*}[t]
    \centering
    \includegraphics[width=0.8\linewidth]{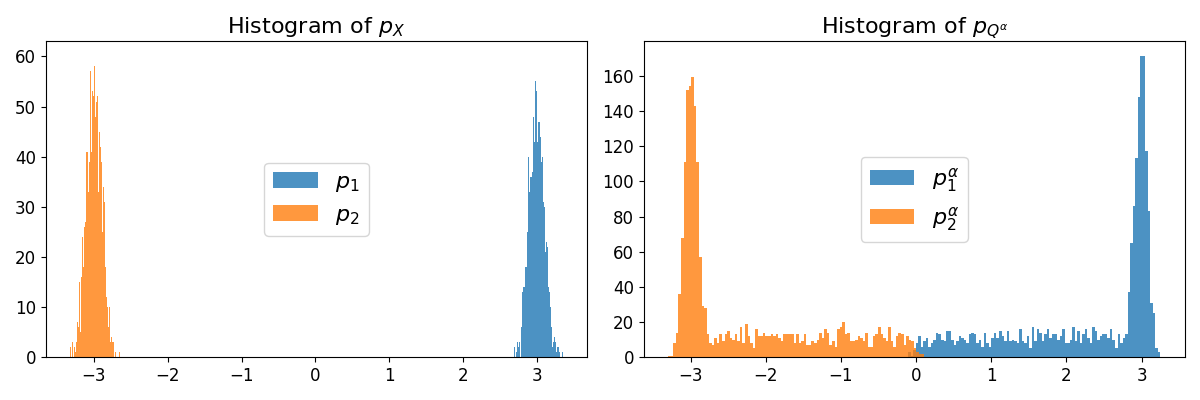}
    \caption{The unimodal components $p_1$ and $p_2$ are generated from normal distribution with $\mu_1=3$, $\mu_2=-3$, and the variance $\sigma^2=0.01$.}
    \label{fig:separability}
\end{figure*}

Based on the decomposition, we obtain 
\begin{align*}
    \mu_1^* = \bE[\alpha u+ (1-\alpha)X_2]=\bE[\alpha] \mu_1 + (1-\bE[\alpha])(\mu_1+\mu_2)/2, \\
    \mu_2^* = \bE[\alpha v+ (1-\alpha)X_2]=\bE[\alpha] \mu_2 + (1-\bE[\alpha])(\mu_1+\mu_2)/2, 
\end{align*}
so $|\mu_1^*-\mu_2^*|=3|\mu_1-\mu_2|/4$
and the variance is 
\begin{align*}
    (\sigma^*)^2 &= \var(\alpha u + (1-\alpha)X_2), \\
    & =\bE[\var(\alpha u + (1-\alpha)X_2|\alpha)] + \var(\bE[\alpha u + (1-\alpha)X_2|\alpha]),\\
    &= \bE[\alpha^2 \var(u) + (1-\alpha)^2 \var(X_2)] + \var(\alpha \bE[u] + (1-\alpha)\bE[X_2]), \\ 
    &= \sigma^2 \bE[\alpha^2] + \var(X_2)\bE[(1-\alpha)^2] + (\mu_1 - \bE[X_2])^2\var(\alpha), \\
    &= \dfrac{3}{4}\sigma^2 + \dfrac{5}{192}(\mu_1-\mu_2)^2,
\end{align*}
where $\bE[\alpha^2]=7/12$, $\bE[(1-\alpha)^2]=1/12$, $\var(\alpha)=1/48$, $\var(X_2)=\sigma^2 + (\mu_1-\mu_2)^2/4$, and $\bE[X_2]=(\mu_1+\mu_2)/2$.

% For $\alpha \sim {\rm Unif}(0,0.5)$, we have $|\mu_1^* - \mu_2^*|=|\mu_1-\mu_2|/4$ and $\sigma^*=\sqrt{2\sigma^2/3+7(\mu_1-\mu_2)^2/192}$. 

% In both cases, we conclude $|\mu_1^*-\mu^*_2|/\sigma^*<|\mu_1-\mu_2|/\sigma$.

\subsection{Theorem~\ref{prop:grad_upbd1}}
\label{supp:prop_thm1}

For readers' convenience, this section is categorized by three subsections: 1) the backpropagation mechanism of GAN (Section~\ref{supp:backprop}), 2) key Lemmas and Remarks (Section~\ref{supp:key_lemmas}), and 3) the main proof (Section~\ref{supp:proof_thm1}). 

\subsubsection{Backpropagation in GAN}
\label{supp:backprop}
Recall that $B_l(x)=W_lA_{l-1}(x)$ is the $l$th pre-activation layer and $A_{l-1}$ is the $(l-1)$th post-activation layer, i.e., $A_l(x)=\kappa_{l}(B_l(x))$, where $W_l\in \mathbb{R}^{N_{l+1}^D\times N_{l}^D}$, $A_{l-1}(x)\in \mathbb{R}^{N_{l}^D}$, and $B_l(x),A_l(x)\in \mathbb{R}^{N_{l+1}^D}$. Set $A_0(x)=x$. Note the activation function $\kappa_l$ applies element-wisely. For simplicity, we implicitly consider $D(x)=A_d(x)=\kappa_d(w_d^{\top}A_{d-1}(x))=\kappa_d(B_d(x))$ where $\kappa_d(x)=x$. Let's denote by $\delta_{X_i}^{l}$ and $\delta_{Z_i}^{l}$ the derivative of $D(X_i)$ and $D(G(Z_i))$ w.r.t. the $l$th pre-activation layer $B_l$ in $D$, i.e., $\delta_{X_i}^{l}=\frac{\pr D(X_i)}{\pr B_{l}}$ and $\delta_{Z_i}^{l}=\frac{\pr D(G(Z_i))}{\pr B_{l}}$. By abusing the notation $\cdot$ for matrix multiplication, we can observe that 
\begin{align*}
    \delta^l_{X_i} = \dfrac{\pr D(X_i)}{\pr B_{l+1}} \cdot \dfrac{\pr B_{l+1}}{\pr A_{l}}\cdot \dfrac{\pr A_l}{\pr B_l} = \left((W_{l+1})^{\top }\delta^{l+1}_{X_i}\right)\odot \kappa_l'(B_l(X_i))\in\mathbb{R}^{N_{l+1}^D},
\end{align*}
where $\odot$ stands for the Hadamard product and $\delta^l_{Z_i}$ presents in the same way. For instance, 
\begin{align*}
    &\delta^d_{X_i} =\dfrac{\pr D(X_i)}{\pr B_d} = 1, \\ 
    &\delta^{d-1}_{X_i}=\dfrac{\pr D(X_i)}{\pr B_d}\cdot\dfrac{\pr B_d}{\pr A_{d-1}}\cdot \dfrac{\pr A_{d-1}}{\pr B_{d-1}}=\delta^d_{X_i}\cdot w_d \odot \kappa_{d-1}'(B_{d-1}(X_i))\in \mathbb{R}^{N^D_d}.
\end{align*}
Based on the above characterization, the gradient of $W_{d-1}$ through the backpropagation can be presented by 
\begin{align*}
    \dfrac{\pr \hat{L}_b}{\pr W_{d-1}}&=\dfrac{1}{n_b}\sum_{i=1}^{n_b}\dfrac{\pr }{\pr W_{d-1}}D(X_i)-\dfrac{\pr }{\pr W_{d-1}}D(G(Z_i)), \\
    &=\dfrac{1}{n_b}\sum_{i=1}^{n_b}\dfrac{\pr D(X_i)}{\pr B_{d-1}}\cdot \dfrac{\pr B_{d-1}}{\pr W_{d-1}}-\dfrac{\pr D(G(Z_i))}{\pr B_{d-1}}\cdot \dfrac{\pr B_{d-1}}{\pr W_{d-1}}, \\
    &=\dfrac{1}{n_b}\sum_{i=1}^{n_b}\delta^{d-1}_{X_i}\cdot A_{d-2}^{\top}(X_i)-\delta^{d-1}_{Z_i}\cdot A_{d-2}^{\top}(G(Z_i)). 
\end{align*}
Therefore, the gradient of $W_l$ is represented by 
\begin{align*}
    \dfrac{\pr \hat{L}_b}{\pr W_l} = \dfrac{1}{n_b}\sum_{i=1}^{n_b} \delta^l_{X_i} \cdot A_{l-1}^{\top }(X_i) - \dfrac{1}{n_b}\sum_{i=1}^{n_b} \delta^l_{Z_i}\cdot A^{\top}_{l-1}(G(Z_i))\in\mathbb{R}^{N_{l+1}^D\times N_{l}^D},
\end{align*}
where $\delta^l \in \mathbb {R}^{N^D_{l+1}}$ and $A_{l-1}(x)\in\mathbb{R}^{N_{l}^D}$. This also implies that the $(r,c)$th parameter of $W_l$ is represented by 
\begin{align}
\label{eqn:backprop_bd1}
    \dfrac{\pr \hat{L}_b}{\pr W_{l,r,c}} = \dfrac{1}{n_b}\sum_{i=1}^{n_b} \delta^l_{X_i,r} A_{l-1,c}(X_i) - \dfrac{1}{n_b}\sum_{i=1}^{n_b} \delta^l_{Z_i,r}A_{l-1,c}(G(Z_i)),
\end{align}
where the additional index $r$ and $c$ in $\delta^l$ and $A_{l-1}$ stand for the $r$th and $c$th element in the vector, respectively.

\subsubsection{Key Lemmas and Remarks}
\label{supp:key_lemmas}
\begin{lemma}
\label{lemma:wln}
Under (A1-3), for any $G\in {\cal G}$ and $D\in {\cal D}$, the size of the minibatch loss is bounded by 
    \begin{align*}
        |\Lhat_b (D, G)|  \leq d_{\cal D}(p_X, p_{G(Z)})+ O_p(1/\sqrt{n_{b}}). 
    \end{align*}
\end{lemma}
\begin{proof}
    Since the support and the norm of matrices are bounded, and the activation is Lipschitz, $D(\cdot)$ is bounded. Let $L_i = D(X_i) - D(G(Z_i))$ and denote by $L=\bE[L_i]$. Since $|\Lhat_b - L|\geq||\Lhat_b| - |L||$, 
    \begin{align*}
        P(||\Lhat_b| - |L||\geq C) \leq P(|\Lhat_b - L|\geq C)\leq \dfrac{\var(\Lhat_b)}{C^2}=\dfrac{\var(D(X_i))+\var(D(G(Z_i)))}{n_b C^2},
    \end{align*}
    for some $C >0$, by the Chebyshev's inequality. Let's express $|\Lhat_b|=|L|+ O_p(1/\sqrt{n_{b}})$ by  definition. Since $|L|\leq d_{\cal D}(p_X,p_{G(Z)})$, we obtain the result. 
\end{proof}

For further analysis, we make an extra assumption:
\begin{enumerate}[label={(A\arabic*)}]
    \setcounter{enumi}{4}
    \item For all $l$, the activation function $\kappa_l(x)$ in the $l$th hidden layer is differentiable and its derivative is lower bounded by $C_{\kappa'}(l)$, i.e., $\kappa_l'(x)\geq C_{\kappa'}(l)>0$ for all $x$.
\end{enumerate}
This assumption will be discussed in the later Remark \ref{rem:kappa} titled ``Characterization of $C_{\kappa'}(l)$", especially for the Sigmoid activation.
\begin{lemma}
\label{lemma:dist_uwb}
    Suppose (A1-3) holds and set $n_b=m_b$. For some $G\in {\cal G}$, we define $\hat{L}_{l}=n_b^{-1}\sum_{i=1}^{n_b} \left\{A_{l}(X_i) - A_{l}(G(Z_i))\right\}$ where $A_{l}(x)=\kappa_{l}(W_lA_{l-1}(x))$ with $W_l\in\mathbb{R}^{N_{l+1}^D\times N_{l}^D}$ is the $l$th post-activation layer output of the input $x$. %$\kappa_{l}(y)$ is an elementwise activation function and differentiable at all $y$. 
    Then, the $r$th post-activation node of $\hat{L}_l$, denoted by $\hat{L}_{l,r}=(\hat{L}_l)_r$, satisfies
    \begin{align*}
        |\hat{L}_{l,r}|\leq 
        \begin{cases}
            \dfrac{1}{\prod_{j=l+1}^d M_w(j)}d_{\cal D}(p_X,p_{G(Z)})+O_p(1/\sqrt{n_{b}}) & \mbox{if } \kappa_l \mbox{ is ReLU}, \\
            \dfrac{1}{\prod_{j=l+1}^d M_w(j)C_{\kappa'}(j)}d_{\cal D}(p_X,p_{G(Z)})+O_p(1/\sqrt{n_{b}}) & \mbox{if (A5) holds}, \\
        \end{cases}
    \end{align*} for all $r=1,\dots,N_{l+1}^D$ and  $l=1,\dots,d-1$.  %In particular, if $l=d$, for any activation function, it follows that 
    % \begin{align*}
    % \lvert\hat{L}_{d,r}\rvert \leq  d_{\cal D}(p_X,p_{G(Z)})+O_p(1/\sqrt{n_{b}}).
    % \end{align*}
\end{lemma}
\begin{proof}
Note the $(d-1)$th hidden layer consisting of $N_{d}^D$ nodes is the matrix multiplication of $W_{d-1}$ and the $(d-2)$th hidden layer output passing the elementwise activation $\kappa_{d-1}$. Considering ${\cal D}$ is a symmetric family (i.e., if $D\in {\cal D}$, then $-D \in {\cal D}$), the optimal critic given $G$ is obtained by
\begin{align*}
    \hat{D}^*_t&=\arg_{D\in {\cal D}}\max\hat{L}_b(D,G),\\
    &=\arg_{D\in {\cal D}}\max \left|\frac{1}{n_b}\sum_{i=1}^{n_b} D(X_i) - \frac{1}{n_b}\sum_{j=1}^{n_b} D(G(Z_j))\right|,\\
    &=\arg_{D\in {\cal D}}\max \left|\frac{1}{n_b}\sum_{i=1}^{n_b} \left\{w_d^{\top}A_{d-1}(X_i) - w_d^{\top}A_{d-1}(G(Z_i))\right\}\right|,\\
    &=\arg_{D\in {\cal D}}\max \left|w_d^{\top}\Lhat_{d-1}(D,G)\right|.
\end{align*}
Because of the maximizing procedure of the critic, it follows $|\hat{L}_b(\hat{D}_t^*,G)|\geq |\hat{L}_b(D,G)|$ for any $D\in {\cal D}$, i.e.,  $|\hat{w}_d^{*\top}\Lhat_{d-1}^*|\geq |w_d^{\top}\Lhat_{d-1}|$ where $\hat{L}_{d-1}^*$ and $\hat{L}_{d-1}$ are the $(d-1)$th post-activation layer -- $\hat{w}_d^*$ and $w_d$ are the last weight -- of $\hat{D}_t^*$ and some $D$. 

Now, let's define $w_d={\bf 0}_{-r}=(0,\dots,M_w(d),\dots,0)^{\top}$ where the $r$th element is $M_w(d)$ but 0 otherwise. Then 
\begin{align*}
    |\hat{L}_b(\hat{D}^*_t,G)|\geq |{\bf 0}_{-r}\hat{L}_{d-1}| = M_w(d) |\hat{L}_{d-1,r}|.
\end{align*}
Therefore, by Lemma~\ref{lemma:wln}, for any $r$th node in the $(d-1)$th hidden layer, 
\begin{align*}
    |\hat{L}_{d-1,r}|\leq \dfrac{1}{M_w(d)}d_{\cal D}(p_X,p_{G(Z)})+O_p(1/\sqrt{n_{b}}).
\end{align*}
For the case of $(d-2)$th post-activation layer, we first observe that 
\begin{align*}
    \hat{L}_{d-1}= \dfrac{1}{n_b}\sum_{i=1}^{n_b} \left\{\kappa_{d-1}(W_{d-1} A_{d-2}(X_i))-\kappa_{d-1}(W_{d-1}A_{d-2}(G(Z_i)))\right\},  
\end{align*}
where $A_{d-2}(x)$ is the $(d-2)$th post-activation layer, and for any other $w_d$ and $W_{d-1}$, the following holds, 
 \begin{align}
 \label{eqn:post1}
 |\hat{L}_b(\hat{D}_t^*,\Gt)| \geq \left\lvert w_d^{\top}\dfrac{1}{n_b}\sum_{i=1}^{n_b} \left\{\kappa_{d-1}(W_{d-1} A_{d-2}(X_i))-\kappa_{d-1}(W_{d-1}A_{d-2}(G(Z_i)))\right\}\right\rvert.
\end{align}
Now, let $W_{d-1,1}$ be the first row vector of $W_{d-1}$, and set $W_{d-1,1}=(0,\dots,M_w(d-1),\dots,0)$ whose $r$th component is $M_w(d-1)$ and 0 otherwise, and $W_{d-1,q}={\bf 0}$ for any other rows $q\neq 1$. Also, we set $w_{d}^{\top}=(M_w(d),0,\dots,0)$ to select the $r$th component, then, \eqref{eqn:post1} becomes
\begin{align}
\label{eqn:post2}
    \left\lvert M_w(d)\dfrac{1}{n_b}\sum_{i=1}^{n_b} \left\{\kappa_{d-1}\left(M_w(d-1) A_{d-2,r}(X_i)\right)-\kappa_{d-1}\left(M_w(d-1) A_{d-2,r}(G(Z_i))\right)\right\}\right\rvert.
\end{align}

\paragraph{Case 1: $\kappa_l$ is ReLU}  Suppose $\kappa_{l}$ is ReLU for all $l=1,\dots,d-1$, i.e., $\kappa_{l}(x)=\max\{x,0\}$. Since all elements of $A_{d-2}(x)$ are non-negative, \eqref{eqn:post2} reduces to $\left\lvert M_w(d)M_w(d-1)\hat{L}_{d-2,r}\right\rvert$ where $\hat{L}_{d-2,r}=\sum_{i=1}^{n_b} A_{d-2,r}(X_i)/n_b-A_{d-2,r}(G(Z_i))/n_b$ and $A_{d-2,r}$ is the $r$th component of the post-activation layer. 
Hence, by Lemma~\ref{lemma:wln}, $|\hat{L}_{d-2,r}|$ is bounded above by 
\begin{align*}
    |\hat{L}_{d-2,r}|\leq \dfrac{1}{M_w(d)M_w(d-1)}d_{\cal D}(p_X,p_{G(Z)})+O_p(1/\sqrt{n_{b}}). 
\end{align*}
By repeating this process to all hidden layers, we can derive 
\begin{align}
\label{lemma:dist_uwb_relu}
    |\hat{L}_{l,r}|\leq \dfrac{1}{\prod_{j=l+1}^dM_w(j)}d_{\cal D}(p_X,p_{G(Z)})+O_p(1/\sqrt{n_{b}}), 
\end{align}
for all $1\leq l\leq d-1,r=1,\dots,N_{l+1}^D$. 

\paragraph{Case 2: $\kappa_l$ follows Assumption A5.} We apply the mean value theorem to \eqref{eqn:post2}, i.e., $\kappa_{d-1}(M_w(d-1) A_{d-2,r}(X_i))=\kappa_{d-1}(M_w(d-1) A_{d-2,r}(G(Z_i)))+\kappa_{d-1}'(\xi_{d-1})M_w(d-1)(A_{d-2,r}(X_i)-A_{d-2,r}(G(Z_i)))$ for some $\xi_{d-1}$ between $M_w(d-1)A_{d-2,r}(X_i)$ and $M_w(d-1)A_{d-2,r}(G(Z_i))$. By (A5), there exist a constant $C_{\kappa'}(d-1)$ such that $|\kappa_{d-1}'(\xi_{d-1})|\geq C_{\kappa'}(d-1) >0$, and we observe
\begin{align*}
    &\left\lvert M_w(d)\dfrac{1}{n_b}\sum_{i=1}^{n_b} \left\{\kappa_{d-1}(M_w(d-1) A_{d-2,r}(X_i))-\kappa_{d-1}(M_w(d-1) A_{d-2,r}(G(Z_i)))\right\}\right\rvert, \\ 
    &=\left\lvert M_w(d)\dfrac{1}{n_b}\sum_{i=1}^{n_b} \left\{\kappa_{d-1}'(\xi_{d-1})(M_w(d-1) A_{d-2,r}(X_i)-M_w(d-1)A_{d-2,r}(G(Z_i)))\right\}\right\rvert, \\ 
    &\geq\left\lvert M_w(d)M_w(d-1)\kappa_{d-1}'(\xi_{d-1})\dfrac{1}{n_b}\sum_{i=1}^{n_b} \left\{A_{d-2,r}(X_i)- A_{d-2,r}(G(Z_i))\right\}\right\rvert, \\ 
    &\geq \left\lvert M_w(d)C_{\kappa'}(d-1)M_w(d-1)\hat{L}_{d-2,r}\right\rvert.  
\end{align*}
In the same way, the $r$th element in $\hat{L}_{d-3}$ has 
\begin{align*}
    |\hat{L}_b(\hat{D}^*_t,\Gt)|\geq |M_w(d)M_w(d-1)M_w(d-2)\kappa'_{d-1}(\xi_{d-1})\kappa'_{d-2}(\xi_{d-2})\hat{L}_{d-3,r}|.
\end{align*}
Hence, we conclude, for $1\leq l\leq d-1$ and $r=1,\dots N^D_{l+1}$,  
\begin{align*}
    |\hat{L}_{l,r}|\leq \dfrac{1}{\prod_{j=l+1}^d M_w(j)C_{\kappa'}(j)}d_{\cal D}(p_X,p_{G(Z)})+O_p(1/\sqrt{n_b}),
\end{align*}
where $C_{\kappa'}(d)=1$.
\end{proof}

 Note that for Sigmoid activation, the above result emphasizes that some tricks for normalizing intermediate layers could be useful to stabilize the size of the gradients of weight matrices (i.e., by inducing large positive $C_{\kappa'}$'s). To support that $C_{\kappa'}(j)$ is bounded below based on (A1-3) for Sigmoid activation, we first show the boundness of post-activation nodes.
 
\begin{remark}[Boundness of $A_{l,c}(x)$]
\label{rem:boundA}
We first characterize an upper bound of $A_{l,c}(X_i)$ and $A_{l,c}(G(Z_i))$, the $c$th element in the $l$th post-activation layer of $D$. Since $A_1(x)=\kappa_1(W_1x)$, denoting by $W_{l,c}$ the $c$th row vector of $W_l$, 
\begin{align*}
    |A_{1,c}(x)|&=|\kappa_1(W_{1,c}x)|=|\kappa_1(0)+\kappa_1'(\xi_x)(W_{1,c}x)|, \\
    &\leq |\kappa_1(0)| + K_{\kappa}(1)\lVert W_{1,c} \rVert \lVert x \rVert. 
\end{align*}
Obviously, ReLU can derive the same result since $A_{1,c}(x)=\max\{W_{1,c}x,0\}$. In particular, if $\kappa_1(0)=0$, then 
\begin{align*}
    \lVert A_{1}(x)\rVert\leq K_{\kappa}(1)\lVert x \rVert \sqrt{\sum_{c=1}^{N_2^D}\lVert W_{1,c} \rVert^2 } = K_{\kappa}(1)\lVert W_{1} \rVert_F \lVert x \rVert.
\end{align*}
Likewise, for the $l$th layer, we have 
\begin{align*}
    \lvert A_{l,c}(x)\rvert &=
    \lvert \kappa_{l}(W_{l,c}A_{l-1}(x))\rvert, \\
    &\leq  |\kappa_{l}(0)|+|\kappa'_l(\xi_x)||(W_{l,c}A_{l-1}(x))\rvert|,\\
    &\leq |\kappa_{l}(0)|+K_{\kappa}(l)\lVert W_{l,c} \rVert \lVert A_{l-1}(x) \rVert. 
\end{align*}
Therefore, for ReLU as well, we see 
\begin{align*}
    |A_{l,c}(x)|\leq K_{\kappa}(l)\lVert W_{l,c} \rVert \prod_{j=1}^{l-1}K_{\kappa}(j)\lVert W_{j} \rVert_F\lVert x \rVert, 
\end{align*}
and set 
\begin{align}
\label{eqn:Asize}
C_{w,\kappa}(l,c)=K_{\kappa}(l)\lVert W_{l,c} \rVert \prod_{j=1}^{l-1}K_{\kappa}(j)\lVert W_{j} \rVert_F,    
\end{align}
for $2\leq l \leq d-1$, and set $C_{w,\kappa}(1,c)=\lVert W_{1,c} \rVert K_{\kappa}(1)$.
\end{remark}
% which further leads to
% \begin{align*}
%     \lvert A_{2,c}(x)\rvert &\leq M_w(1)M_w(2)K_{\kappa}(1)K_{\kappa}(2)\sqrt{N_2^D}\lVert x \rVert,\\
%     \lvert A_{3,c}(x)\rvert &\leq M_w(1)M_w(2)M_w(3)K_{\kappa}(1)K_{\kappa}(2)K_{\kappa}(3)\sqrt{N_2^D}\sqrt{N_3^D}\lVert x \rVert,\\
%     &\vdots
%     % \lvert A_{l,c}(x)\rvert &\leq \prod_{j=1}^l M_w(j)K_{\kappa}(j) \prod_{s=2}^l\sqrt{N^D_s}\lVert x \rVert. 
% \end{align*}
% Therefore, the bound of $|A_{l,c}(X_i)|\leq C_{w,\kappa,X}(l)\times B_X$ appears where
% \begin{align}
% \label{eqn:Asize}
% C_{w,\kappa,X}(l)=\prod_{j=1}^l M_w(j)K_{\kappa}(j) \prod_{s=2}^l\sqrt{N^D_s},    
% \end{align}
% for $l\geq 2$, ans set $C_{w,\kappa,X}(1)=M_w(1)K_{\kappa}(1)$ and $C_{w,\kappa,X}(0)=1$. Note $\kappa_d(x)=x$. 

% For the composition part $\lvert A_{l,c}(G(Z))\rvert \leq C_{w,\kappa,X}(l)\lVert G(Z) \rVert$, it suffices to find the bound of $\lVert G(Z) \rVert = \lVert v_g^{\top} A_{g-1}(G(Z)) \rVert \leq M_g(v) \lVert A_{g-1}(G(Z)) \rVert$ where $G(Z)=\psi_g(v_g^{\top}A_{g-1}(G(Z))$ with the identity $\psi_g(x)=x$. 
% By following the same argument, $\lVert G(Z) \rVert \leq C_{v,\psi,Z}$ where 
% \begin{align*}
%     C_{v,\psi,Z}=\prod_{j=1}^g M_w(j) K_{\psi}(j) \prod_{s=2}^g \sqrt{N^G_s} B_Z.
% \end{align*}
% In practice, it is common to use a bounded activation function such as Sigmoid and Tanh for $\psi_g$ that produces a similar range w.r.t. the normalized ${\cal X}$. In this case, the bounds of $\lvert A_{l,c}(X)\rvert$ and $\lvert A_{l,c}(G(Z))\rvert$ can be seen to be similar. 

\begin{remark}[Characterization of $C_{\kappa'}(l)$]\label{rem:kappa} Now we can bound the impact of $\kappa_l'(\xi_l)$ for differentiable activation functions where $\xi_l$ is between $M_w(l)A_{l-1,r}(X)$ and $M_w(l)A_{l-1,r}(G(Z))$. For simplicity, we assume $\kappa(0)=0$, which as shown in the previous paragraph derives  
\begin{align*}
    |M_w(l)A_{l-1,r}(X_i)|&\leq \prod_{j=1}^l \lVert W_{j} \rVert_F \prod_{s=1}^{l-1}K_{\kappa}(s) \lVert X_i \rVert \leq \prod_{j=1}^l M_w(j) \prod_{s=1}^{l-1}K_{\kappa}(s) B_X,
\end{align*}
and 
\begin{align*}
    |M_w(l)A_{l-1,r}(G(Z_i))|&\leq \prod_{j=1}^l \lVert W_{j} \rVert_F \prod_{s=1}^{l-1}K_{\kappa}(s) \lVert G(Z_i) \rVert \leq \prod_{j=1}^l M_w(j) \prod_{s=1}^{l-1}K_{\kappa}(s) \prod_{k=1}^g M_v(k)B_Z. 
\end{align*}
Thus we can define 
 \begin{align}
 \label{eqn:C_kappa_prime}
     C_{\kappa'}(l)=|\inf_{x\in {\cal X}_{l}} \kappa_l'(x)|, \quad l=1,\dots,d-1,
 \end{align}
 where 
 \begin{align*}
     {\cal X}_{l}= \prod_{j=1}^l M_w(j) \prod_{s=1}^{l-1}K_{\kappa}(s)\times \left[-\max\left\{B_X,\prod_{j=1}^g M_v(j)B_Z\right\}, \max \left\{B_X,\prod_{j=1}^g M_v(j)B_Z\right\}\right].\,
 \end{align*}
 and $C_{\kappa'}(l)$ is then strictly positive.
 % If $\kappa(x)$ is non-decreasing, then 
 % \begin{align*}
 %     C_{\kappa'}(l)=\kappa_l'\left(-\prod_{j=1}^l M_w(j) \prod_{s=1}^{l-1}K_{\kappa}(s)\max \left\{B_X,\prod_{j=1}^g M_v(j)B_Z\right\}\right).
 % \end{align*}
 \end{remark}

\begin{lemma}
\label{lem:lip_prod}
Let $\{f_i\}_{i=1}^n$ be a finite collection of functions where each $f_i: {\cal X} \to \mathbb{R}$ for some domain ${\cal X}\subseteq \mathbb{R}$. Suppose that for each $i \in \{1, \dots, n\}$: 1) the function $f_i$ is $L_i$-Lipschitz and 2) bounded by $|f_i(x)| \le M_i$ for all $x \in {\cal X}$. Then the product function $g(x) = \prod_{i=1}^{n} f_i(x)$ is also bounded and Lipschitz continuous on ${\cal X}$. Specifically, $g(x)$ is bounded by $M_g = \prod_{i=1}^{n} M_i$ and has a Lipschitz constant $L_g$ satisfying the inequality:
$$L_g \le \sum_{i=1}^{n} \left( L_i \prod_{j=1, j \ne i}^{n} M_j \right)$$
\end{lemma}

\begin{proof}
First, it is trivial to see that
$|g(x)| = \prod_{i=1}^{n} |f_i(x)|\le \prod_{i=1}^{n} M_i = M_g$ for any $x$, so
$|g(x)| \le \prod_{i=1}^{n} M_i = M_g$ since each $|f_i(x)| \le M_i$. Next, we establish the Lipschitz continuity of $g(x)$. For any $x, y \in {\cal X}$, we can characterize the absolute difference $|g(x) - g(y)|$ as follows, 
\begin{align*}
g(x) - g(y) = \prod_{i=1}^{n} f_i(x) - \prod_{i=1}^{n} f_i(y) = \sum_{i=1}^{n} \left( \left(\prod_{j=1}^{i-1} f_j(y)\right) \left(\prod_{j=i+1}^{n} f_j(x)\right)(f_i(x) - f_i(y))  \right).    
\end{align*}
For instance, in the case of $n=3$, denoting by $f_i(x)=a_i$ and $f_i(y)=b_i$, since $a_2a_3-b_2b_3=a_2a_3-b_2b_3\pm a_3b_2=a_3(a_2-b_2)+b_2(a_3-b_3)$, we see 
\begin{align*}
    a_1a_2a_3-b_1b_2b_3=a_1a_2a_3-b_1b_2b_3\pm b_1a_2a_3 &= a_2a_3(a_1-b_1)+b_1(a_2a_3-b_2b_3), \\
    &=a_2a_3(a_1-b_1) + b_1b_2(a_3-b_3)+a_3b_1(a_2-b_2).
\end{align*}
Therefore, its size is upper bounded by
\begin{align*}
    |g(x) - g(y)| &\le \sum_{i=1}^{n} \left| \left(\prod_{j=1}^{i-1} f_j(y)\right) \left(\prod_{j=i+1}^{n} f_j(x)\right)(f_i(x) - f_i(y))  \right|, \\
     &\le \sum_{i=1}^{n} \left( \left(\prod_{j=1}^{i-1} |f_j(y)|\right) \left(\prod_{j=i+1}^{n} |f_j(x)|\right)|f_i(x) - f_i(y)|  \right). 
\end{align*}
Since each function $f_i$ is bounded by $M_i$ and is $L_i$-Lipschitz, 
\begin{align*}
    |g(x) - g(y)| \le \sum_{i=1}^{n} \left( \left(\prod_{j=1}^{i-1} M_j\right) \left(\prod_{j=i+1}^{n} M_j\right)L_i |x-y|  \right).
\end{align*}
Factoring out the common term $|x-y|$ from the sum derives
\begin{align*}
    |g(x) - g(y)| \le \left( \sum_{i=1}^{n} L_i \prod_{j=1, j \ne i}^{n} M_j \right) |x-y|.
\end{align*}
This shows that $g(x)$ is Lipschitz continuous with the constant $L_g \le \sum_{i=1}^{n} \left( L_i \prod_{j=1, j \ne i}^{n} M_j \right)$.
\end{proof}

\subsubsection{Proof of Theorem~\ref{prop:grad_upbd1}}
\label{supp:proof_thm1}
Now we derive Theorem~\ref{prop:grad_upbd1} based on the previous lemmas. Recall that ${\cal D}= \{D(x) = w^{\top}_d \kappa_{d-1}(W_{d-1}\kappa_{d-2}(\cdots W_1 x)) :  \bw=(W_1,\dots, W_{d-1}, w_d)\in \bW\}$ where $w_d\in \mathbb{R}^{N^D_d\times 1}$ and $W_l \in \mathbb{R}^{N^D_{l+1}\times N^D_l}$ for $l=1,\dots,d-1$. Note $W_{l}$ consists of the $N_{l+1}^D$ number of row vectors, i.e., $W_{l,r} \in \mathbb{R}^{1\times N^D_l}$ for $r=1,\dots,N_{l+1}^D$. Let $A_{l}(x)$ and $B_l(x)$ be the $l$th post-/pre-activation layer, respectively, i.e., $A_l(x)=\kappa_l(B_l(x))$ and $B_l(x)=W_lA_{l-1}(x)$. Under (A1-3), we derive the results where $\kappa(x)$ is differentiable and satisfies (A5) (e.g., Sigmoid, Tanh, ELU activation functions) or $\kappa'(x)$ is bounded (e.g. ReLU). %To accommodate activation functions that may return negative values (e.g., Tanh), a proof sketch is also provided for the case where activation functions are bounded from below beyond Assumptions A1-A3 and A5. Note that the main text only presents the result with ReLU activation since it has a simpler form. \jw{It is remarkable that no matter what popular activation functions are used, the neural distance contributes to determining the size of the gradient.}

Recall that $\delta^{d-1}_{x}=1\times w_d \odot \kappa_{d-1}'(B_{d-1}(x))$ and that $\delta^l_{x}=(W_{l+1}^{\top}\delta^{l+1}_{x})\odot \kappa_{l}'(B_l(x))$ where $B_l(x)\in {\mathbb R}^{N_{l+1}^D}$ (Section~\ref{supp:backprop}). To simplify the discussion, we define a diagonal matrix $T_l(x)=\diag(\kappa_{l}'(B_{l,1}(x)),\dots ,\kappa_l'(B_{l,N_{l+1}^D}(x)))\in \mathbb{R}^{{N_{l+1}^D}\times {N_{l+1}^D}}$. Then it appears that 
\begin{align*}
    \delta^l_{x}=(W_{l+1}^{\top}\delta^{l+1}_{x})\odot \kappa_{l}'(B_l(x))=T_l(x)\cdot W^{\top}_{l+1}\cdot  \delta_{x}^{l+1},
\end{align*}
which means that $\delta^l_{x,r} = T_{l,r,r}(x)\times (W_{l+1}^{\top}\delta_x^{l+1})_r$ where $({\bf x})_r$ is the $r$th component of the vector ${\bf x}$. Through the recursive structure of $\delta^l_x$, we observe 
\begin{align*}
    \delta^{d-1}_x&=T_{d-1}(x)\cdot w_d, \\
    \delta^{d-2}_x&=T_{d-2}(x)\cdot W_{d-1}^{\top} \cdot T_{d-1}(x)\cdot w_d,\\ 
    \delta^{d-3}_x&=T_{d-3}(x)\cdot W_{d-2}^{\top} \cdot T_{d-2}(x)\cdot W_{d-1}^{\top}\cdot T_{d-1}(x)\cdot w_d,\\ 
    &\vdots
\end{align*}
The $r$th component presents $\delta^{l}_{x,r}=T_{l,r,r}(x)\times W_{l+1,r}^{\top}\cdot v$ where $W_{l+1,r}^{\top}$ is the $r$th row vector of $W^{\top}_{l+1}$, and $v$ is the output of the remaining products of the matrices and vector. In the next paragraphs, we will characterize the size of the gradient of $w_d$, $W_{d-1}$, and $W_{d-2}$, which is then generalized to the $l$th layer. 

\paragraph{Proof for $w_d$} The gradient of $w_{d,r}$ is 
\begin{align*}
    \dfrac{\pr \hat{L}_b}{\pr w_{d,r}} = \dfrac{1}{n_b}\sum_{i=1}^{n_b} A_{d-1,r}(X_i)-A_{d-1,r}(G(Z_i)),
\end{align*}
so by Lemma~\ref{lemma:dist_uwb}, 
\begin{align*}
    \left\lvert \dfrac{\pr \hat{L}_b}{\pr w_{d,r}}  \right\rvert \leq \dfrac{1}{M_w(d)}d_{\cal D}(p_X,p_{G(Z)})+O_p(1/\sqrt{n_b}).
\end{align*}

\paragraph{Proof for $W_{d-1}$} Let's recall the gradient of the $(r,c)$th parameter in $W_{d-1}$ is 
\begin{align*}
    \dfrac{\pr \hat{L}_b}{\pr W_{d-1,r,c}}&=\dfrac{1}{n_b}\sum_{i=1}^{n_b}\delta^{d-1}_{X_i,r}\cdot A_{d-2,c}(X_i)-\delta^{d-1}_{Z_i,r}\cdot A_{d-2,c}(G(Z_i)),\\
    &=\dfrac{w_{d,r}}{n_b}\sum_{i=1}^{n_b} \kappa'_{d-1,r}(X_i)\times  A_{d-2,c}(X_i)-\kappa'_{d-1,r}(G(Z_i))\times  A_{d-2,c}(G(Z_i)).
\end{align*}
where $\kappa'_{d-1,r}(X_i)=\kappa'_{d-1}(B_{d-1,r}(X_i))$ and $\kappa'_{d-1}(B_{d-1,r}(G(Z_i))=\kappa'_{d-1,r}(G(Z_i))$.
Before characterizing the size of this gradient, we first make the following remarks. 

\begin{remark}[Boundness and Lipschitzness of $A_{l,c}(x)$]
Note $A_{l,c}(x)=\kappa_l(W_lA_{l-1}(x))$ is the $c$th post-activation node in the $l$th hidden layer. Under (A1-3), the function $A_{l,c}(x)$ is bounded as discussed in Remark~\ref{rem:boundA} (when $\kappa(0)=0$), i.e., 
\begin{align*}
     |A_{l,c}(x)|&\leq K_{\kappa}(l)\lVert W_{l,c} \rVert \prod_{j=1}^{l-1}K_{\kappa}(j)\lVert W_{j} \rVert_F\lVert x \rVert, \\
     &\leq K_{\kappa}(l)\lVert W_{l,c} \rVert C_{A_{l-1}}\times \max\left\{B_X,B_{G(Z)}\right\},
\end{align*}
where $B_{G(Z)}=\prod_{j=1}^g M_v(j) K_{\psi}(j) B_Z$ and set $C_{A_{l-1}}=\prod_{s=1}^{l-1}K_{\kappa}(s)\lVert W_{s}  \rVert_F$. Note $B_X$ and $B_{G(Z)}$ stand for the size of the domain ${\cal X}$ and $\{G(z):z\in {\cal Z}\}$. We also remark that some activation functions, such as Sigmoid or Tanh, are trivially bounded. For Lipschitzness, the Cauchy-Schwarz inequality derives 
\begin{align*}
    |A_{l,c}(x)-A_{l,c}(y)|&=|\kappa_l(W_{l,c}A_{l-1}(x))-\kappa_l(W_{l,c}A_{l-1}(y))|, \\
    &\leq K_{\kappa}(l) \lVert W_{l,c}\rVert \lVert A_{l-1}(x)-A_{l-1}(y) \rVert,\\
    &\leq K_{\kappa}(l)\lVert W_{l,c} \rVert \prod_{s=1}^{l-1}K_{\kappa}(s)\lVert W_{s}  \rVert_F \times \lVert x - y\rVert, \\
    &= K_{\kappa}(l)\lVert W_{l,c} \rVert \times C_{A_{l-1}}\times \lVert x-y\rVert. 
\end{align*}
We define constants $M_{A_{l,c}}$ and $L_{A_{l,c}}$ such that 
\begin{align*}
 &|A_{l,c}(x)|\leq M_{A_{l,c}}, \\ 
 &|A_{l,c}(x)-A_{l,c}(y)|\leq L_{A_{l,c}} \lVert x-y\rVert, 
\end{align*}
for all $x$ and $y$.    
\end{remark}

\begin{remark}[Boundness and Lipschitzness of $\kappa'_{l}(B_{l,r}(x))$]
\label{rem:bd_lip_k}
Note $B_{l,r}(x)=W_{l,r}A_{l-1}(x)$ where $W_{l,r}$ is the $r$th row of $W_l$. If $\kappa_{l}'(x)$ is $K_{\kappa'}(l)$-Lipschitz, it follows that 
\begin{align*}
    |\kappa'_{l}(B_{l,r}(x))-\kappa'_{l}(B_{l,r}(y))|&=|\kappa'_{l}(W_{l,r}A_{l-1}(x))-\kappa'_{l}(W_{l,r}A_{l-1}(y)|\\
    &\leq K_{\kappa'}(l)|W_{l,r}A_{l-1}(x)-W_{l,r}A_{l-1}(y))|\\ 
    &\leq  K_{\kappa'}(l)\lVert W_{l,r} \rVert \lVert A_{l-1}(x)-A_{l-1}(y)\rVert \\
    &\leq K_{\kappa'}(l)\lVert W_{l,r} \rVert \times C_{A_{l-1}}\times \lVert x -y\rVert.
\end{align*}
For the boundeness, $|\kappa'_{l}(x)|\leq K_{\kappa}(l)$ for all $x$. 
We define constants $M_{\kappa'_{l,r}}$ and $L_{\kappa'_{l,r}}$ such that 
\begin{align*}
    &|\kappa'_{l}(B_{l,r}(x))|\leq M_{\kappa'_{l,r}}, \\ 
 &|\kappa'_{l}(B_{l,r}(x))-\kappa'_{l}(B_{l,r}(y))|\leq L_{\kappa'_{l,r}} \lVert x-y\rVert, 
\end{align*}
for all $x,y$, and also for any $l,r$. 
\end{remark}

Now, the backpropagated gradient can be decomposed as 
\begin{align*}
    \dfrac{1}{w_{d,r}}\times \dfrac{\pr \hat{L}_b}{\pr W_{d-1,r,c}}=&\dfrac{1}{n_b}\sum_{i=1}^{n_b}\kappa_{d-1,r}'(X_i)A_{d-2,c}(X_i)- \kappa_{d-1,r}'(G(Z_i))A_{d-2,c}(G(Z_i)) \\
    =&\underbrace{\dfrac{1}{n_b}\sum_{i=1}^{n_b}\kappa_{d-1,r}'(X_i)A_{d-2,c}(X_i)- \kappa_{d-1,r}'(X_i)A_{d-2,c}(G(Z_i))}_{\text{(I)}}\\
    &+\underbrace{\dfrac{1}{n_b}\sum_{i=1}^{n_b}(\kappa_{d-1,r}'(X_i)-\kappa_{d-1,r}'(G(Z_i)))A_{d-2,c}(G(Z_i))}_{\text{(II)}}.
\end{align*}
For (I), 
\begin{align*}
    {\text{(I)}}=&\dfrac{1}{n_b}\sum_{i=1}^{n_b}\bE[\kappa_{d-1,r}'(X_i)]\left\{A_{d-2,c}(X_i)-A_{d-2,c}(G(Z_i))\right\} \\
    &+ \dfrac{1}{n_b}\sum_{i=1}^{n_b}(\kappa'_{d-1,r}(X_i)-\bE[\kappa'_{d-1,r}(X_i))])\left\{A_{d-2,c}(X_i)-A_{d-2,c}(G(Z_i))\right\},
\end{align*}
where the Chebyshev's inequality approximates the second term 
\begin{align*}
&\dfrac{1}{n_b}\sum_{i=1}^{n_b}(\kappa'_{d-1,r}(X_i)-\bE[\kappa'_{d-1,r}(X_i))])\left\{A_{d-2,c}(X_i)\pm\bE[A_{d-2,c}(X_i)]-A_{d-2,c}(G(Z_i))\right\}, \\
    &= \text{Cov}(\kappa'_{d-1,r}(X_i), A_{d-2,c}(X_i))+ O_p(1/\sqrt{n_b}),
\end{align*}
which conceptually represents the degree of alignment between the $r$th and $c$th nodes. 

In the meantime, (II) can be approximated by  
\begin{align*}
    \text{(II)}=&\bE[(\kappa_{d-1,r}'(X_i)-\kappa_{d-1,r}'(G(Z_i)))A_{d-2,c}(G(Z_i))] + O_p(1/\sqrt{n_b})\\
    =&\bE[\kappa_{d-1,r}'(X_i)]\bE[A_{d-2,c}(G(Z_i))]-\bE[\kappa_{d-1,r}'(G(Z_i))A_{d-2,c}(G(Z_i))]+ O_p(1/\sqrt{n_b})\\ 
    =&\bE[\kappa_{d-1,r}'(X_i)]\bE[A_{d-2,c}(G(Z_i))]-\bE[\kappa_{d-1,r}'(G(Z_i))]\bE[A_{d-2,c}(G(Z_i))] \\ &+\bE[\kappa_{d-1,r}'(G(Z_i))]\bE[A_{d-2,c}(G(Z_i))]-\bE[\kappa_{d-1,r}'(G(Z_i))A_{d-2,c}(G(Z_i))]+ O_p(1/\sqrt{n_b})\\
    =&\bE[A_{d-2,c}(G(Z_i))](\bE[\kappa_{d-1,r}'(X_i)]-\bE[\kappa_{d-1,r}'(G(Z_i))])\\
    &-\text{Cov}(\kappa_{d-1,r}'(G(Z_i)),A_{d-2,c}(G(Z_i)))+ O_p(1/\sqrt{n_b}).
\end{align*}
Likewise, we can decompose the gradient as 
\begin{align*}
    \dfrac{1}{w_{d,r}}\times \dfrac{\pr \hat{L}_b}{\pr W_{d-1,r,c}}=&\dfrac{1}{n_b}\sum_{i=1}^{n_b}\kappa_{d-1,r}'(X_i)A_{d-2,c}(X_i)- \kappa_{d-1,r}'(G(Z_i))A_{d-2,c}(G(Z_i)) \\
    =&\underbrace{\dfrac{1}{n_b}\sum_{i=1}^{n_b}\kappa_{d-1,r}'(G(Z_i))A_{d-2,c}(X_i)- \kappa_{d-1,r}'(G(Z_i))A_{d-2,c}(G(Z_i))}_{\text{(I)}}\\
    &+\underbrace{\dfrac{1}{n_b}\sum_{i=1}^{n_b}(\kappa_{d-1,r}'(X_i)-\kappa_{d-1,r}'(G(Z_i)))A_{d-2,c}(X_i)}_{\text{(II)}}.
\end{align*}
Following the previous approach, we have
\begin{align*}
    {\text{(I)}}=%&\dfrac{1}{n_b}\sum_{i=1}^{n_b}\bE[\kappa_{d-1,r}'(G(Z_i))]\left\{A_{d-2,c}(X_i)-A_{d-2,c}(G(Z_i))\right\} \\
    %&+ \dfrac{1}{n_b}\sum_{i=1}^{n_b}(\kappa'_{d-1,r}(G(Z_i))-\bE[\kappa'_{d-1,r}(G(Z_i)))])\left\{A_{d-2,c}(X_i)-A_{d-2,c}(G(Z_i))\right\}, \\ 
    &\dfrac{1}{n_b}\sum_{i=1}^{n_b}\bE[\kappa_{d-1,r}'(G(Z_i))]\left\{A_{d-2,c}(X_i)-A_{d-2,c}(G(Z_i))\right\} \\ 
    &-\text{Cov}(\kappa'_{d-1,r}(G(Z_i)),A_{d-2,c}(G(Z_i)))+O_p(1/\sqrt{n_b}),
\end{align*}
and 
\begin{align*}
    \text{(II)}=%&\bE[(\kappa_{d-1,r}'(X_i)-\kappa_{d-1,r}'(G(Z_i)))A_{d-2,c}(X_i)] + O_p(1/\sqrt{n_b}), \\ 
    %=&\bE[\kappa_{d-1,r}'(X_i)A_{d-2,c}(X_i)]-\bE[\kappa_{d-1,r}'(G(Z_i))]\bE[A_{d-2,c}(X_i)]+O_p(1/\sqrt{n_b}), \\
    %=&\bE[\kappa_{d-1,r}'(X_i)A_{d-2,c}(X_i)]-\bE[\kappa_{d-1,r}'(X_i)]\bE[A_{d-2,c}(X_i)]\\
    %&+\bE[\kappa_{d-1,r}'(X_i)]\bE[A_{d-2,c}(X_i)]-\bE[\kappa_{d-1,r}'(G(Z_i))]\bE[A_{d-2,c}(X_i)]+O_p(1/\sqrt{n_b}),\\
    &\text{Cov}(\kappa_{d-1,r}'(X_i),A_{d-2,c}(X_i))\\ 
    &+\bE[A_{d-2,c}(X_i)](\bE[\kappa'_{d-1,r}(X_i)]-\bE[\kappa'_{d-1,r}(G(Z_i))])+O_p(1/\sqrt{n_b}).
\end{align*}
Therefore, the gradient can be expressed by averaging the two previous representations. 
\begin{align*}
    \dfrac{1}{w_{d,r}}\times\dfrac{\pr \hat{L}_b}{\pr W_{d-1,r,c}}=&\dfrac{\bE[\kappa_{d-1,r}'(X_i)]+\bE[\kappa_{d-1,r}'(G(Z_i))]}{2}\dfrac{1}{n_b}\sum_{i=1}^{n_b}\left\{A_{d-2,c}(X_i)-A_{d-2,c}(G(Z_i))\right\}\\ 
    &+\text{Cov}(\kappa'_{d-1,r}(X_i),A_{d-2,c}(X_i))-\text{Cov}(\kappa'_{d-1,r}(G(Z_i)),A_{d-2,c}(G(Z_i))) \\ 
    &+\dfrac{\bE[A_{d-2,c}(X_i)]+\bE[A_{d-2,c}(G(Z_i))]}{2}(\bE[\kappa'_{d-1,r}(X_i)]-\bE[\kappa'_{d-1,r}(G(Z_i))])\\
    &+O_p(1/\sqrt{n_b}).
\end{align*}
By Lemma~\ref{lemma:dist_uwb}, 
\begin{align*}
    \dfrac{1}{n_b}\sum_{i=1}^{n_b}\left\{A_{d-2,c}(X_i)-A_{d-2,c}(G(Z_i))\right\}\leq \dfrac{1}{\prod_{j=d-1}^d M_w(j)C_{\kappa'}(j)}d_{\cal D}(p_X,p_{\Gt(Z)})+O_p(1/\sqrt{n_b}),
\end{align*}
where $C_{\kappa'}(j)$ disappears if ReLU activation is used.
% the size would be bounded by
% \begin{align*}
%     \dfrac{1}{|w_{d,r}|}\times\left\lvert \dfrac{\pr \hat{L}_b}{\pr W_{d-1,r,c}}\right\rvert\leq &\left\lvert \dfrac{\bE[\kappa_{d-1,r}'(X_i)]+\bE[\kappa_{d-1,r}'(G(Z_i))]}{2}\right\rvert\times \dfrac{1}{\prod_{j=d-1}^d M_w(j)C_{\kappa'}(j)}d_{\cal D}(p_X,p_{\Gt(Z)})\\ 
%     &+\lvert \text{Cov}(\kappa'_{d-1,r}(X_i),A_{d-2,c}(X_i))-\text{Cov}(\kappa'_{d-1,r}(G(Z_i)),A_{d-2,c}(G(Z_i))) \rvert \\ 
%     &+\left\lvert \dfrac{\bE[A_{d-2,c}(X_i)]+\bE[A_{d-2,c}(G(Z_i))]}{2}\right\rvert\lvert \bE[\kappa'_{d-1,r}(X_i)]-\bE[\kappa'_{d-1,r}(G(Z_i))]\rvert \\
%     &+O_p(1/\sqrt{n_b}).
% \end{align*}
The third term $\lvert \bE[\kappa'_{d-1,r}(X_i)]-\bE[\kappa'_{d-1,r}(G(Z_i))]\rvert$ can also be represented by other probability metrics: 
\begin{itemize}
    \item If $\kappa'_{d-1}(x)$ is Lipschitz, then there exists a Lipschitz constant $L_{\kappa'_{d-1,r}}>0$ such that $\kappa'_{d-1,r}(x)$ is $L_{\kappa'_{d-1,r}}$-Lipschitz by Remark~\ref{rem:bd_lip_k}. Therefore, 
    \begin{align*}
     \lvert \bE[\kappa'_{d-1,r}(X_i)]-\bE[\kappa'_{d-1,r}(G(Z_i))]\rvert&\leq  \sup_{g\in \left\{L_{\kappa'_{d-1,r}}\text{-Lipschitz}\right\}}\bE[g(X_i)]-\bE[g(G(Z_i))], \\
     &\leq L_{\kappa'_{d-1,r}}W_1(p_X,p_{G(Z)}).
    \end{align*}
    \item If $\kappa'_{d-1}(x)$ is bounded, i.e., $|\kappa'_{d-1}(x)|\leq K_{\kappa}(d-1)$, 
    \begin{align*}
        \lvert \bE[\kappa'_{d-1,r}(X_i)]-\bE[\kappa'_{d-1,r}(G(Z_i))]\rvert &\leq K_{\kappa}(d-1) \int|p_X(x)-p_{G(Z)}(x)|dx,\\ 
        & = 2K_{\kappa}(d-1)d_{\text{TV}}(p_X,p_{G(Z)}). 
    \end{align*}
\end{itemize}
Let's define the technical constants 
\begin{align*}
    C'_{\kappa,1}(d-1)&=\left\lvert \dfrac{\bE[\kappa_{d-1,r}'(X_i)]+\bE[\kappa_{d-1,r}'(G(Z_i))]}{2}\right \rvert, \\
    C_{\kappa,2}(d-1)&=\left\lvert \dfrac{\bE[A_{d-2,c}(X_i)]+\bE[A_{d-2,c}(G(Z_i))]}{2}\right\rvert\times \begin{cases}
       L_{\kappa'_{d-1,r}},~\text{$\kappa'_{d-1}(x)$ is Lipschitz},\\
       2K_{\kappa}(d-1),~\text{$\kappa'_{d-1}(x)$ is bounded},\\
    \end{cases}\\
    C_{\kappa,3}(d-1)&=\lvert \text{Cov}(\kappa'_{d-1,r}(X_i),A_{d-2,c}(X_i))-\text{Cov}(\kappa'_{d-1,r}(G(Z_i)),A_{d-2,c}(G(Z_i))) \rvert.
\end{align*}
Hence, the size of gradient is eventually characterized by 
\begin{itemize}
    \item When $\kappa'$ is Lipschitz, 
    {\footnotesize
    \begin{align*}
    \left\lvert \dfrac{\pr \hat{L}_b(D,G)}{\pr W_{d-1,r,c}}\right\rvert\leq |w_{d,r}|\left(\dfrac{C'_{\kappa,1}(d-1)}{\prod_{j=d-1}^d M_w(j)C_{\kappa'}(j)}d_{\cal D}(p_X,p_{G(Z)})+C_{\kappa,2}(d-1)W_1(p_X,p_{G(Z)})+C_{\kappa,3}(d-1)\right),
    \end{align*}}
    \item When $\kappa'$ is bounded, 
    {\footnotesize
    \begin{align*}
    \left\lvert \dfrac{\pr \hat{L}_b(D,G)}{\pr W_{d-1,r,c}}\right\rvert\leq |w_{d,r}|\left(\dfrac{C'_{\kappa,1}(d-1)}{\prod_{j=d-1}^d M_w(j)C_{\kappa'}(j)}d_{\cal D}(p_X,p_{G(Z)})+C_{\kappa,2}(d-1)d_{\text{TV}}(p_X,p_{G(Z)})+C_{\kappa,3}(d-1)\right),
    \end{align*}}
\end{itemize}
and set $C_{\kappa,1}(d-1)=C'_{\kappa,1}(d-1)/\prod_{j=d-1}^d M_w(j)C_{\kappa'}(j)$. Note ReLU activation does not need $C_{\kappa'}(j)$. 

\paragraph{Proof for $W_{d-2}$} To characterize the gradient in the $(d-2)$th layer, first recall that $\delta^{d-2}_x=T_{d-2}(x)\cdot W_{d-1}^{\top} \cdot T_{d-1}(x)\cdot w_d$ and 
\begin{align*}
\dfrac{\pr \hat{L}_b}{\pr W_{d-2,r,c}}&=\dfrac{1}{n_b}\sum_{i=1}^{n_b}\delta^{d-2}_{X_i.r}\cdot A_{d-3,c}(X_i)-\delta^{d-2}_{Z_i,r}\cdot A_{d-3,c}(G(Z_i)).
\end{align*}
Note $W_{d-2}\in \mathbb{R}^{N^D_{d-1}\times N^D_{d-2}}$  and $W_{d-1}\in \mathbb{R}^{N^D_{d}\times N^D_{d-1}}$. Then, we have 
\begin{align*}
    \dfrac{\pr \hat{L}_b}{\pr W_{d-2,r,c}}&=\dfrac{1}{n_b}\sum_{i=1}^{n_b}\left\{\delta^{d-2}_{X_i,r}\cdot A_{d-3,c}(X_i)-\delta^{d-2}_{Z_i,r}\cdot A_{d-3,c}(G(Z_i))\right\},\\
    &=\dfrac{1}{n_b}\sum_{i=1}^{n_b}(T_{d-2}(X_i)\cdot W_{d-1}^{\top} \cdot T_{d-1}(X_i)\cdot w_d)_r\times   A_{d-3,c}(X_i)\\
    &-\dfrac{1}{n_b}\sum_{i=1}^{n_b}(T_{d-2}(G(Z_i))\cdot W_{d-1}^{\top} \cdot T_{d-1}(G(Z_i))\cdot w_d)_r \times A_{d-3,c}(G(Z_i)),
\end{align*}
where $({\bf x})_r$ is the $r$th component of the vector ${\bf x}$. It follows that 
\begin{align*}
    (T_{d-2}(x)\cdot W_{d-1}^{\top} \cdot T_{d-1}(x)\cdot w_d)_r = T_{d-2,r,r}^{x} \sum_{k=1}^{N_d^D} W^{\top}_{d-1,r,k}T_{d-1,k,k}^{x}w_{d,k},
\end{align*}
where $T^{x}_{d-2,r,r}$ is the $(r,r)$th component of the diagonal matrix $T_{d-2}(x)$. Therefore, 
{\footnotesize
\begin{align*}
&\dfrac{\pr \hat{L}_b}{\pr W_{d-2,r,c}}=\\
    &\dfrac{1}{n_b}\sum_{i=1}^{n_b}T_{d-2,r,r}^{X_i} \sum_{k=1}^{N_d^D} W^{\top}_{d-1,r,k}T_{d-1,k,k}^{X_i}w_{d,k} A_{d-3,c}(X_i) - 
    \dfrac{1}{n_b}\sum_{i=1}^{n_b}T_{d-2,r,r}^{Z_i} \sum_{k=1}^{N^D_d} W^{\top}_{d-1,r,k}T_{d-1,k,k}^{Z_i}w_{d,k} A_{d-3,c}(G(Z_i)),
\end{align*}}
and due to the linearity of the summation, 
{\footnotesize
\begin{align*}
    \dfrac{\pr \hat{L}_b}{\pr W_{d-2,r,c}}&=\sum_{k=1}^{N_d^D}W_{d-1,k,r}w_{d,k}\underbrace{\left(\dfrac{1}{n_b}\sum_{i=1}^{n_b}\left\{T_{d-2,r,r}^{X_i}T_{d-1,k,k}^{X_i}A_{d-3,c}(X_i)-T_{d-2,r,r}^{Z_i}T_{d-1,k,k}^{Z_i}A_{d-3,c}(G(Z_i))\right\}\right)}_{\text{(I)}},
\end{align*}}
where $W_{d-1,r,k}^{\top}=W_{d-1,k,r}$. Let $p_{d-2,r,k}(x)=T_{d-2,r,r}^{x}T_{d-1,k,k}^{x}$. By Lemma~\ref{lem:lip_prod}, $p_{d-2,r,k}(x)$ is $L_{p_{d-2,r,k}}$-Lipschitz and bounded by $K_{\kappa}(d-2)K_{\kappa}(d-1)$. By following the proof procedure for the case of $W_{d-1}$, we see 
\begin{align*}
    \text{(I)}=&\dfrac{\bE[p_{d-2,r,k}(X_i)]+\bE[p_{d-2,r,k}(G(Z_i))]}{2}\dfrac{1}{n_b}\sum_{i=1}^{n_b}\left\{A_{d-3,c}(X_i)-A_{d-3,c}(G(Z_i))\right\}\\ 
    &+\text{Cov}(p_{d-2,r,k}(X_i),A_{d-3,c}(X_i))-\text{Cov}(p_{d-2,r,k}(G(Z_i)),A_{d-3,c}(G(Z_i))) \\ 
    &+\dfrac{\bE[A_{d-3,c}(X_i)]+\bE[A_{d-3,c}(G(Z_i))]}{2}(\bE[p_{d-2,r,k}(X_i)]-\bE[p_{d-2,r,k}(G(Z_i))])\\
    &+O_p(1/\sqrt{n_b}).
\end{align*}
Therefore, if $\kappa'$ is Lipschitz, 
\begin{align*}
    |\text{(I)}|\leq &C_{\kappa,1}'(d-2)\dfrac{1}{\prod_{j=d-2}^dM_w(j)C_{\kappa'}(j)}d_{\cal D}(p_X,p_{G(Z)})+C_{\kappa,2}(d-2)W_1(p_X,p_{G(Z)})\\
    &+C_{\kappa,3}(d-2)+O_p(1/\sqrt{n_b}),
\end{align*}
and, if $\kappa'$ is ReLU, 
\begin{align*}
    |\text{(I)}|\leq &C_{\kappa,1}'(d-2)\dfrac{1}{\prod_{j=d-2}^dM_w(j)}d_{\cal D}(p_X,p_{G(Z)})+C_{\kappa,2}(d-2)d_{\text{TV}}(p_X,p_{G(Z)})\\
    &+C_{\kappa,3}(d-2)+O_p(1/\sqrt{n_b}),
\end{align*}
where 
\begin{align*}
    C_{\kappa,1}'(d-2)&=\max_{k}\left\lvert \dfrac{\bE[p_{d-2,r,k}(X_i)]+\bE[p_{d-2,r,k}(G(Z_i))]}{2}\right\rvert, \\
    C_{\kappa,2}(d-2)&= \left\lvert \dfrac{\bE[A_{d-3,c}(X_i)]+\bE[A_{d-3,c}(G(Z_i))]}{2} \right\rvert\times \begin{cases}
        \max_k L_{p_{d-2,r,k}},~\text{if $\kappa'$ is Lipschitz}, \\ 
        \prod_{j=d-2}^{d-1}K_{\kappa}(j),~\text{if $\kappa'$ is bounded},\\
    \end{cases}\\
    C_{\kappa,3}(d-2)&=\max_k|\text{Cov}(p_{d-2,r,k}(X_i),A_{d-3,c}(X_i))-\text{Cov}(p_{d-2,r,k}(G(Z_i)),A_{d-3,c}(G(Z_i)))|.
\end{align*}
Thus, the size of the gradient is upper bounded 
\begin{align*}
    \left \lvert \dfrac{\pr \hat{L}_b}{\pr W_{d-2,r,c}}\right\rvert\leq \sum_{k=1}^{N_d^D}|W_{d-1,k,r}w_{d,k}||(\text{I})|,
\end{align*}
and 
\begin{align*}
    \sum_{k=1}^{N_d^D}|W_{d-1,k,r}w_{d,k}|&\leq \left(\sum_{k=1}^{N_d^D}|W_{d-1,k,r}|^2\right)^{1/2} \left(\sum_{k=1}^{N_d^D}|w_{d,k}|^2\right)^{1/2}, \\
    &= \lVert W_{d-1,\cdot,r} \rVert \times \lVert w_{d} \rVert. 
\end{align*}
In consequence, Let $C_{\bw}(d-2)=\lVert W_{d-1,\cdot,r} \rVert \times \lVert w_{d} \rVert$. 
\begin{itemize}
    \item If $\kappa'$ is Lipschitz, 
\begin{align*}
    \dfrac{1}{C_{\bw}(d-2)}\left\lvert \dfrac{\pr \hat{L}_b(D,G)}{\pr W_{d-2,r,c}}\right\rvert \leq &\dfrac{C_{\kappa,1}'(d-2)}{\prod_{j=d-2}^dM_w(j)C_{\kappa'}(j)}d_{\cal D}(p_X,p_{G(Z)})+C_{\kappa,2}(d-2)W_1(p_X,p_{G(Z)})\\
    &+C_{\kappa,3}(d-2)+O_p(1/\sqrt{n_b}),
\end{align*}
and set $C_{\kappa,1}(d-2)=\frac{C_{\kappa,1}'(d-2)}{\prod_{j=d-2}^dM_w(j)C_{\kappa'}(j)}$.
\item If $\kappa'$ is ReLU, 
\begin{align*}
    \dfrac{1}{C_{\bw}(d-2)}\left\lvert \dfrac{\pr \hat{L}_b(D,G)}{\pr W_{d-2,r,c}}\right\rvert \leq &\dfrac{C_{\kappa,1}'(d-2)}{\prod_{j=d-2}^dM_w(j)}d_{\cal D}(p_X,p_{G(Z)})+C_{\kappa,2}(d-2)d_{\text{TV}}(p_X,p_{G(Z)})\\
    &+C_{\kappa,3}(d-2)+O_p(1/\sqrt{n_b}), 
\end{align*}
and set $C_{\kappa,1}(d-2)=\frac{C_{\kappa,1}'(d-2)}{\prod_{j=d-2}^dM_w(j)}$.
\end{itemize}

\paragraph{Proof for $W_{l}$} Similarly, we can apply the proof technique to any hidden layer. Recall that $\delta^l_x$ is 
\begin{align*}
    \delta^{l}_x&=T_{l}(x)\cdot W_{l+1}^{\top} \cdot T^{}_{l+1}(x)\cdot W_{l+2}^{\top}\cdot T_{l+2}(x)\cdots T_{d-1}(x)\cdot w_d,
\end{align*}
where either $X_i$ or $G(Z_i)$ is put for $x$. 
For $l\leq d-2$, we observe 
\begin{align*}
    \delta^{l}_{x,r} A_{l-1,c}(x)=T_{l,r,r}^{x}\sum_{k_{l+2}=1}^{N_{l+2}^D}\cdots \sum_{k_d=1}^{N^D_d} \left[\left(\prod_{j=l+1}^{d-1}W_{j,k_{j},k_{j+1}}^{\top}T^{x}_{j,k_{j+1},k_{j+1}}\right)w_{d,k_d}\right]A_{l-1,c}(x),
\end{align*}
where $k_{l+1}=r$ is fixed and 
\begin{align*}
    &\dfrac{\pr \hat{L}_b}{\pr W_{l,r,c}} \notag \\
    &=\sum_{k_{l+2}=1}^{N_{l+2}^D}\cdots \sum_{k_d=1}^{N_{d}^D}\prod_{j=l+1}^{d-1}W_{j,k_{j},k_{j+1}}^{\top} w_{d,k_d}\notag \\
    &\times \underbrace{\left(\dfrac{1}{n_b}\sum_{i=1}^{n_b}\left\{T_{l,r,r}^{X_i} \prod_{j=l+1}^{d-1}T^{X_i}_{j,k_{j+1},k_{j+1}}A_{l-1,c}(X_i)-T_{l,r,r}^{Z_i} \prod_{j=l+1}^{d-1}T^{Z_i}_{j,k_{j+1},k_{j+1}}A_{l-1,c}(G(Z_i))\right\}\right)}_{\text{(I)}}.
\end{align*}
Note we can say $p_{l,r,k_{l\sim d}}(x)=T_{l,r,r}^{x}\prod_{j=l+1}^{d-1}T^{x}_{j,k_{j+1},k_{j+1}}$ for some $k_{l\sim d}=(r,k_{l+2},\dots,k_{d})$.

By Lemma~\ref{lem:lip_prod}, there exists a Lipschitz constant $L_{l,r,k_{l\sim d}}$ for $p_{l,r,k_{l\sim d}}(x)$, and it is $\prod_{j=l}^{d-1}K_{\kappa}(j)$ bounded. The size of the gradient is upper bounded by
\begin{align*}
    \left\lvert \dfrac{\pr \hat{L}_b}{\pr W_{l,r,c}} \right\rvert \leq \underbrace{\sum_{k_{l+2}=1}^{N_{l+2}^D}\cdots \sum_{k_d=1}^{N_{d}^D}\prod_{j=l+1}^{d-1}\lvert W_{j,k_{j},k_{j+1}}^{\top}\rvert  \lvert w_{d,k_d}\rvert}_{\text{(II)}} \times |\text{(I)}|.
\end{align*}
Following the previous argument, (I) is characterized as
\begin{itemize}
    \item if $\kappa'$ is Lipschitz, 
\begin{align*}
    |\text{(I)}|\leq &\dfrac{C_{\kappa,1}'(l)}{\prod_{j=l}^dM_w(j)C_{\kappa'}(j)}d_{\cal D}(p_X,p_{G(Z)})+C_{\kappa,2}(l)W_1(p_X,p_{G(Z)})\\
    &+C_{\kappa,3}(l)+O_p(1/\sqrt{n_b}),
\end{align*}
\item if $\kappa'$ is ReLU, 
\begin{align*}
    |\text{(I)}|\leq &\dfrac{C_{\kappa,1}'(l)}{\prod_{j=l}^dM_w(j)}d_{\cal D}(p_X,p_{G(Z)})+C_{\kappa,2}(l)d_{\text{TV}}(p_X,p_{G(Z)})\\
    &+C_{\kappa,3}(l)+O_p(1/\sqrt{n_b}),
\end{align*}
\end{itemize}
where 
\begin{align*}
    C_{\kappa,1}'(l)&=\max_{k_{l\sim d}}\left\lvert \dfrac{\bE[p_{l,r,k_{l\sim d}}(X_i)]+\bE[p_{l,r,k_{l\sim d}}(G(Z_i))]}{2}\right\rvert, \\
    C_{\kappa,2}(l)&= \left\lvert \dfrac{\bE[A_{l-1,c}(X_i)]+\bE[A_{l-1,c}(G(Z_i))]}{2} \right\rvert\times \begin{cases}
        \max_{k_{l\sim d}} L_{l,r,k_{l\sim d}},~\text{if $\kappa'$ is Lipschitz}, \\ 
        \prod_{j=l}^{d-1}K_{\kappa}(j),~\text{if $\kappa'$ is bounded},\\
    \end{cases}\\
    C_{\kappa,3}(l)&=\max_{k_{l\sim d}}|\text{Cov}(p_{l,r,k_{l\sim d}}(X_i),A_{l-1,c}(X_i))-\text{Cov}(p_{l,r,k_{l\sim d}}(G(Z_i)),A_{l-1,c}(G(Z_i)))|.
\end{align*}

Meanwhile, the Cauchy-Schwarz inequality finds 
\begin{align*}
    \text{(II)}&=\sum_{k_{l+2}=1}^{N_{l+2}^D}\cdots \sum_{k_d=1}^{N_{d}^D}\prod_{j=l+1}^{d-1}\lvert W_{j,k_{j},k_{j+1}}^{\top}\rvert  \lvert w_{d,k_d}\rvert, \\ 
    &= \sum\cdots \sum_{k_{d-1}=1}^{N_{d-1}^D}\prod_{j=l+1}^{d-2}\lvert W_{j,k_{j},k_{j+1}}^{\top}\rvert 
    \sum_{k_d=1}^{N_d^D}|W^{\top}_{d-1,k_{d-1},k_d}||w_{d,k_d}|, \\ 
    &\leq \sum \cdots  \sum_{k_{d-1}=1}^{N_{d-1}^D} |W^{\top}_{d-2,k_{d-2},k_{d-1}}| \lVert W^{\top}_{d-1,k_{d-1},\cdot} \rVert  \lVert w_d \rVert, \\
    &\leq \sum \cdots \lVert W^{\top}_{d-2,k_{d-2},\cdot} \rVert \lVert W_{d-1} \rVert_F  \lVert w_d \rVert,
\end{align*}
where $W^{\top}_{l,r,\cdot}$ stands for the $r$th row vector in $W^{\top}_l$. By repeating the process, 
\begin{align*}
    \text{(II)}\leq \lVert W_{l+1,r,\cdot}^{\top} \rVert \prod_{j=l+2}^{d-1} \lVert W_{j} \rVert_F \lVert w_d\rVert. 
\end{align*}
In consequence, letting $C_{\bw}(l)=\lVert W_{l+1,\cdot,r} \rVert \prod_{j=l+2}^{d-1} \lVert W_{j} \rVert_F \lVert w_d\rVert$, the bound of $\left\lvert \frac{\pr \hat{L}_b(D,G)}{\pr W_{l,r,c}} \right\rvert$ appears as 
\begin{itemize}
    \item if $\kappa'$ is Lipschitz, 
\begin{align*}
    \dfrac{1}{C_{\bw}(l)}\left\lvert \frac{\pr \hat{L}_b(D,G)}{\pr W_{l,r,c}} \right\rvert\leq &C_{\kappa,1}(l)d_{\cal D}(p_X,p_{G(Z)})+C_{\kappa,2}(l)W_1(p_X,p_{G(Z)})\\&+C_{\kappa,3}(l)+O_p(1/\sqrt{n_b}),\\ 
    C_{\kappa,1}(l)=&\dfrac{C_{\kappa,1}'(l)}{\prod_{j=l}^dM_w(j)C_{\kappa'}(j)},
\end{align*}
\item if $\kappa$ is differentiable and $\kappa'$ is bounded, 
\begin{align*}
    \dfrac{1}{C_{\bw}(l)}\left\lvert \frac{\pr \hat{L}_b(D,G)}{\pr W_{l,r,c}} \right\rvert\leq &C_{\kappa,1}(l)d_{\cal D}(p_X,p_{G(Z)})+C_{\kappa,2}(l)d_{\text{TV}}(p_X,p_{G(Z)})\\&+C_{\kappa,3}(l)+O_p(1/\sqrt{n_b}),\\
    C_{\kappa,1}(l)=&\dfrac{C_{\kappa,1}'(l)}{\prod_{j=l}^dM_w(j)C_{\kappa'}(j)},
\end{align*}
\item if $\kappa'$ is ReLU, 
\begin{align*}
    \dfrac{1}{C_{\bw}(l)}\left\lvert \frac{\pr \hat{L}_b(D,G)}{\pr W_{l,r,c}} \right\rvert\leq &C_{\kappa,1}(l)d_{\cal D}(p_X,p_{G(Z)})+C_{\kappa,2}(l)d_{\text{TV}}(p_X,p_{G(Z)})\\&+C_{\kappa,3}(l)+O_p(1/\sqrt{n_b}),\\
    C_{\kappa,1}(l)=&\dfrac{C_{\kappa,1}'(l)}{\prod_{j=l}^dM_w(j)}.
\end{align*}
\end{itemize}

\subsection{Corollary~\ref{prop:grad_upbd2_2} and Remark~\ref{rem:var_reduction}}

By following the proof scheme of Theorem~\ref{prop:grad_upbd1}, the result in the main text can be obtained. For the condition in Remark~\ref{rem:var_reduction}, the distance $d_{\cal D}$ upper bounds 
\begin{align*}
    d_{\cal D}&(p_{\Qal,\alpha},p_{\Gt(Z,\alpha),\alpha}) 
    \\ 
    &= \sup_{D} \int D(q_a,a) p(q_a|a)p(a)d{q_a}da - \int D(G(z,a),a) p(G(z,a)|a)p(a)d{z}da \\ 
    &\leq r\left(\sup_D \int_{a=1} D(x,1)p(x) dx - \int_{a=1} D(G(z,1),1) p(G(z,1))dz \right)\\ 
    &+(1-r)\left(\sup_D \int_{a \neq 1} D(q_a^*,a^*) p(q_a^*,a^*)dq_a^*da^* - \int_{a \neq 1} D(G(z,a^*),a^*) p(G(z,a^*),a^*)dzda^*\right), \\  
    &= (1-r') d_{\cal D}(p_{X,1},p_{\Gt(Z,1),1})+ r' d_{\cal D}(p_{Q^{\alpha^*},\alpha^*},p_{\Gt(Z,\alpha^*),\alpha^*}), 
\end{align*}
where $r+r'=1$. Hence, the condition in the statement is obtained. It is straightforward to check the conditions for the 1-Wasserstein distance $d_{W_1}$ and the total variation $d_{\text{TV}}$.

\subsection{Proposition~\ref{prop:ae}}
\label{supp:prop:approx}
We first present technical Lemmas required to show the approximation error in the main text. The main proof appears then. 

\begin{lemma}[\citeSupp{yaro:17}]
\label{lem:yaro_square}
The function $f(x)=x^2$ on the segment of $[0,1]$ can be approximated by a feedforward ReLU network $\eta$ having the depth $L$ and width $7$ for $L\geq 2$ such that  
\begin{align}
    \sup_{x\in[0,1]} |f(x)-\eta(x)|\leq 4^{-L}.
\end{align}
\end{lemma}
\begin{proof}
    In this work, depth and width are the number of weight matrices and hidden neurons in a standard feedforward neural network. Let's denote by $\sigma(x)=\max\{x,0\}$ the ReLU activation function. For instance, the identity function $i(x)=\sigma(x)-\sigma(-x)$ has depth 2 and width 2. From \citeSupp{yaro:17}, let's define $f_m(x)=x-\sum_{s=1}^mg_s(x)/2^{2s}$ where $g_s(x)$ is the $s$ compositions of the tooth function $g(x)=2\sigma(x)-4\sigma(x-0.5)+2\sigma(x-1)$ where $g(x)$ has depth 2 and width 3. The author found that 
    \begin{align*}
        \sup_{x\in[0,1]} |f(x)-f_m(x)|\leq 2^{-2(m+1)},
    \end{align*}
    and argued that $f_m$ requires essentially $m+1$ depth, i.e., $m$ depth for $g_m$ and 1 depth for the final affine transformation. Note the non-standard ReLU network $f_m$ has 3 width. However, since such $f_m$ is not a standard feedforward neural network, we construct a standard feedforward neural network $\eta$ such that $\eta(x)=f_m(x)$. It is straightforward to see that this $\eta$ consists of $L=m+1$ depth $(L\geq 2)$ and has $7$ width (3 width for $g(x)$, 2 width for $i(x)$, and 2 width to accumulate $g_s(x)$ through $i(g(x))$), where all hidden neurons have the same width. 
\end{proof}

\begin{lemma}[\citeSupp{yaro:17}]
\label{lem:yaro}
    For any given $M>0$ for $|x|,|y|\leq M$, there exists a feedforward ReLU network $g$ with depth $L+1$ and width $21$ for $L\geq 2$ such that 
    \begin{align*}
        \sup_{|x|,|y|\leq M}|g(x,y)-xy|\leq 4^{-L} 
    \end{align*}
\end{lemma}
\begin{proof}
    Let's first observe that $f(x)=a|x|=a\sigma(x)+a\sigma(-x)$, which is a neural network of depth 2 and width 2. Therefore, we can construct a neural network $\tilde{g}(x,y)=\eta_1(|x+y|/2M)-\eta_2(|x|/2M)-\eta_3(|y|/2M)$ where each $i$th neural network $\eta_i$ with $L$ depth and $7$ width satisfies the approximation capability in Lemma~\ref{lem:yaro_square}. Note a ReLU network with depth 2 and width 6 produces $(x,y)\rightarrow (|x+y|,|x|,|y|)$. Now, suppose that $|\eta_i(x)-x^2|\leq\delta$. Then, based on the expression $xy=\frac{1}{2}((x+y)^2-x^2-y^2)$, we can find
    \begin{align*}
        -3\delta\times 2M^2 \leq  2M^2 \tilde{g}(x,y)-xy\leq 3\delta\times 2M^2.
    \end{align*}
    Therefore, we set $\delta=4^{-L}/6M^2$, and see $g(x,y)=2M^2\tilde{g}(x,y)$ belongs to a class of feedforward ReLU networks having depth $L+1$ and width $21$ for $L\geq 2$, i.e., $7$ width for each $\eta_i$. The additional depth originates from $(x,y)\rightarrow (|x+y|,|x|,|y|)$, and the multiplication of width from the number of $\eta_i$, $i=1,2,3$. 
\end{proof}

\begin{lemma}
\label{lem:approx_interpo}
Let's suppose that  $\tilde{G}_k\in \tilde{{\cal G}}:\mathbb{R}^{d_Z}\rightarrow \mathbb{R}^{d_X}$ belongs to a class of ReLU-based feedforward neural networks satisfying (A1-2) in the main text. There exists a ReLU-based feedforward neural network $G^{\dagger}\in {\cal G}^{\dagger}:\mathbb{R}^{2d_X+1}\rightarrow \mathbb{R}^{d_X}$ having $L+1$ depth and $23d_X$ width with the approximation error
\begin{align*}
    \sup_{z_1,z_2,\alpha}\lVert G^{\dagger}(\tilde{G}_1(z_1),\tilde{G}_2(z_2),\alpha)-(\alpha \tilde{G}_1(z_1)+(1-\alpha)\tilde{G}_2(z_2) )\rVert\leq \sqrt{d_X} 4^{-L}.
\end{align*}
\end{lemma}
\begin{proof}
    Let's denote by the output of the $k$th sub-generator $\tilde{G}_k(z)=(v_1^k,\dots,v_{d_X}^{k})$ where $\tilde{G}_k(z)=V_g^k\sigma(V_{g-1}^k\sigma(\cdots z))$ for $k=1,2$, and we define their difference $(m_1,\dots,m_{d_X})=(v_1^1-v_1^2,\dots,v_{d_X}^1-v_{d_X}^2)$. Then, we apply Lemma~\ref{lem:yaro} to figure out the approximation error for handling mutiplications $\alpha m_i$ for $i=1,\dots,d_X$ by additional neural networks to the output of $\tilde{G}_k$. 
    
    For each $i$th multiplication $\alpha m_i$, we place the ReLU-based neural network with the amount of error $4^{-L}$ in Lemma~\ref{lem:yaro}. Therefore, there should be a feedforward neural network with depth $L+1$ and width $21\times d_X$ that has the total $\sqrt{d_X}\times 4^{-L}$ error for the $d_X$ terms in the $l^2$-norm, producing the vector of $\alpha (m_1,\dots,m_{d_X})$. At the same time, $(v_1^2,\dots,v_{d_X}^2)$ has to be delivered such that the network sums $\alpha (m_1,\dots,m_{d_X})+(v_1^2,\dots,v_{d_X}^2)$, which can be made via the identity map $i(x)=\sigma(x)+\sigma(-x)$ and eventually takes $L+1$ depth and additional $2\times d_X$ width. In consequence, there exists a $L+1$ depth and $23\times d_X$ width feedforward ReLU network $G^{\dagger}$ that approximates $G^{\dagger}(\tilde{G}_1(z_1),\tilde{G}_2(z_2),\alpha)=\alpha \tilde{G}_1(z_1)+(1-\alpha)\tilde{G}_2(z_2)$ with the total error $\sqrt{d_X}\times 4^{-L}$.

    % number of $m_i$, and then the additional network has the depth $L'\leq c_1\log (d_X/\delta)+c_2$ and width $W'=3L'+1$. Hence, for the approximation of $\alpha (\tilde{G}_1(z_1)-\tilde{G}_2(z_2))$ within $\delta$, there should be feedforward neural networks having $3L'd_X$ width and depth $L'\leq c_1\log(2 d_X/\delta)+c_2$. 
\end{proof}

\begin{remark}
\label{remark:interpo}
    For theoretical analysis harnessing Lemma~\ref{lem:approx_interpo}, we can decompose the feedforward ReLU networks $G\in {\cal G}$ with depth $\VD$ and width $\VW$ into two parts $(L_i,W_i)$ for $i=1,2$, each of which has $L_i$ depth and $W_i$ width. Let's denote by $V^{(i)}$ the weight matrix of the $i$th part, and characterize $G$ as follows, 
    \begin{align*}
        G(z_1,z_2,\alpha)&=V_{L_2}^{(2)}\sigma (V_{L_2-1}^{(2)}\sigma(\cdots V_{1}^{(2)}V_{L_1}^{(1)}(\sigma(V_{L_1-1}^{(1)}(\sigma(\cdots V_1^{(1)}[z_1,z_2,\alpha])))))), \\
        &=V_{L_2}^{(2)}\sigma (V_{L_2-1}^{(2)}\sigma(\cdots V^{(1,2)}(\sigma(V_{L_1-1}^{(1)}(\sigma(\cdots V_1^{(1)}[z_1,z_2,\alpha])))))),
    \end{align*}
    where $V^{(1,2)}=V_{1}^{(2)}V_{L_1}^{(1)}$ is a weight matrix, and a class of such $G(z_1,z_2,\alpha)$ includes the generator having the form of the composition $G^{\dagger}(\tilde{G}_1(z_1),\tilde{G}_2(z_2),\alpha)$ in Lemma~\ref{lem:approx_interpo} where $\tilde{G}_1,\tilde{G}_2\in \tilde{{\cal G}}$ share no network parameters and $G^{\dagger}\in{\cal G}^{\dagger}$. More specifically, in the structure of $G$, the first $L_1$ depth network with $W_1$ width will contain $\tilde{G}_1,\tilde{G}_2$ each of which has $\lceil W_1/2 \rceil$ size of width. The remaining $L_2$ depth and $W_2$ width network in $G$ will represent $G^{\dagger}$ that approximates the two intermediate generators as in Lemma~\ref{lem:approx_interpo}.
\end{remark}

For any $G\in {\cal G}$, by definition of the neural distance, it follows that 
\begin{align*}
    d_{\cal D}(p_{\Qal,\alpha},p_{G(Z_1,Z_2,\alpha),\alpha})\leq \bE_{\alpha}[d_{\cal D}(p_{\Qal|\alpha},p_{G(Z_1,Z_2,\alpha)|\alpha})], 
\end{align*}
and accordingly, 
\begin{align*}
    \inf_{G\in {\cal G}} d_{\cal D}(p_{\Qal,\alpha},p_{G(Z_1,Z_2,\alpha),\alpha})\leq \inf_{G\in {\cal G}}\bE_{\alpha}[d_{\cal D}(p_{\Qal|\alpha},p_{G(Z_1,Z_2,\alpha)|\alpha})]. 
\end{align*}
By Lemma~\ref{lem:approx_interpo} and based on our argument in Remark~\ref{remark:interpo}, we can restrict $G$ such that $C_G:\lVert G^{\dagger}(\tilde{G}_1(z_1),\tilde{G}_2(z_2),\alpha)-(\alpha \tilde{G}_1(z_1) + (1-\alpha)\tilde{G}_2(z_2))\rVert \leq \sqrt{d_X} 4^{-L_2}$ for any given $z_1,z_2,\alpha$. By the triangle inequality of $d_{\cal D}(p_X,p_Y)\leq d_{\cal D}(p_X,p_Z)+d_{\cal D}(p_Z,p_Y)$, we observe
\begin{align*}
\inf_{G}\bE_{\alpha}[d_{\cal D}(p_{\Qal|\alpha},p_{G(Z_1,Z_2,\alpha)|\alpha})]&\leq \inf_{G;C_G}\bE_{\alpha}[d_{\cal D}(p_{\Qal|\alpha},p_{G(Z_1,Z_2,\alpha)|\alpha})], \\ 
&\leq \inf_{G^{\dagger},\tilde{G}_1,\tilde{G}_2}\bE_{\alpha}[\underbrace{d_{\cal D}(p_{\alpha X_1 + (1-\alpha )X_2|\alpha},p_{\alpha \tilde{G}_1(Z_1)+(1-\alpha)\tilde{G}_2(Z_2)|\alpha}}_{\text{(I)}})] \\ 
&+\bE_{\alpha}[\underbrace{d_{\cal D}(p_{\alpha \tilde{G}_1(Z_1)+(1-\alpha)\tilde{G}_2(Z_2)|\alpha},p_{G^{\dagger}(\tilde{G}_1(Z_1),\tilde{G}_2(Z_2),\alpha)|\alpha})}_{\text{(II)}}].
\end{align*}
\paragraph{Characterization of (I)} To simplify notation, let $G^{\alpha}=\alpha \tilde{G}_1(Z_1)+(1-\alpha)\tilde{G}_2(Z_2)$, $G_{Z_1}:=\tilde{G}_1(Z_1)$ and $G_{Z_2}:=\tilde{G}_1(Z_2)$. Let's define arbitrary couplings for $(X_1,G_{Z_1})\sim \gamma_1$ and  $(X_2,G_{Z_2})\sim \gamma_2$, and define a product coupling $\gamma=\gamma_1\otimes \gamma_2$, where $\gamma_1$ has the marginals of $X_1$ and $G_{Z_1}$ and $\gamma_2$ does as well. It is obvious to see that 
\begin{align*}
    \bE_{p_{\Qal|\alpha}}[D(\Qal,\alpha)]=\bE_{p_{X_1}\otimes p_{X_2}}[D(\alpha X_1+(1-\alpha)X_2,\alpha)]=\bE_{\gamma}[D(\alpha X_1 +(1-\alpha)X_2,\alpha)],
\end{align*}
and also
\begin{align*}
 \bE_{p_{G^{\alpha}|\alpha}}[D(G^{\alpha},\alpha)]=\bE_{p_{G_{Z_1}}\otimes p_{G_{Z_2}}}[D(\alpha G_{Z_1}+(1-\alpha)G_{Z_2},\alpha)]=\bE_{\gamma}[D(\alpha G_{Z_1}+(1-\alpha)G_{Z_2},\alpha)].
\end{align*}
By using the Cauchy-Schwarz inequality, we derive
\begin{align*}
    \text{(I)}&=\sup_D \int D(\alpha X_1 + (1-\alpha)X_2,\alpha)-D(\alpha G_{Z_1} + (1-\alpha)G_{Z_2},\alpha)d\gamma, \\ 
    &\leq \int \sup_D|D(\alpha X_1 + (1-\alpha)X_2,\alpha)-D(\alpha G_{Z_1} + (1-\alpha)G_{Z_2},\alpha)|d\gamma,
    \\ 
    &\leq \prod_{i=1}^d M_w(i)\prod_{j=1}^{d-1}K_{\kappa}(j) \int \lVert \alpha (X_1- G_{Z_1})+(1-\alpha)(X_2-G_{Z_2}) \rVert d\gamma
    \\ 
    &\leq \prod_{i=1}^d M_w(i)\prod_{j=1}^{d-1}K_{\kappa}(j) \left(\alpha\bE_{\gamma_1}[\lVert X_1 - G_{Z_1} \rVert ] + (1-\alpha)\bE_{\gamma_2}[\lVert X_2-G_{Z_2} \rVert ]\right),
\end{align*}
for any given $\alpha$. Since $\gamma_1$ and $\gamma_2$ are arbitrary, 
\begin{align*}
    \text{(I)} \leq \prod_{i=1}^d M_w(i)\prod_{j=1}^{d-1}K_{\kappa}(j) (\alpha d_{W_1}(p_{X_1},p_{G_1(Z_1)})+(1-\alpha)d_{W_1}(p_{X_2},p_{G_2(Z_2)})),
\end{align*}
where $0\leq \alpha \leq 1$.

\paragraph{Characterization of (II)} We abuse the notation $\gamma$ to denote a coupling whose marginals are in (II). By using the Cauchy-Schwarz inequality as the previous paragraph, 
\begin{align*}
    \text{(II)}&\leq \prod_{i=1}^d M_w(i)\prod_{j=1}^{d-1}K_{\kappa}(j)\int \lVert G^{\dagger}(\tilde{G}_1(z_1),\tilde{G}_2(z_2),\alpha)-(\alpha \tilde{G}_1(z_1) + (1-\alpha)\tilde{G}_2(z_2))\rVert d\gamma,\\
    &\leq  \prod_{i=1}^d M_w(i)\prod_{j=1}^{d-1}K_{\kappa}(j)\sqrt{d_X}4^{-L_2}. 
\end{align*}

Therefore, by combining the above results, we characterize 
\begin{align*}
    &\inf_{G}\bE_{\alpha}[d_{\cal D}(p_{\Qal|\alpha},p_{G(Z_1,Z_2,\alpha)|\alpha})]\leq \\
    &\prod_{i=1}^d M_w(i)\prod_{j=1}^{d-1}K_{\kappa}(j)\left(\inf_{G_1,G_2} \bE[\alpha] d_{W_1}(p_{X_1},p_{\tilde{G}_1(Z_1)})+(1-\bE[\alpha])d_{W_1}(p_{X_2},p_{\tilde{G}_2(Z_2)}) + \sqrt{d_X}4^{-L_2}\right).
\end{align*}
Since $\inf_{\tilde{G}_1} d_{W_1}(p_{X_1},p_{\tilde{G}_1(Z_1)})=\inf_{\tilde{G}_2}d_{W_1}(p_{X_2},p_{\tilde{G}_2(Z_2)})$, the last line is simplified to 
\begin{align*}
    \inf_{G}\bE_{\alpha}[d_{\cal D}(p_{\Qal|\alpha},p_{G(Z_1,Z_2,\alpha)|\alpha})]\leq \prod_{i=1}^d M_w(i)\prod_{j=1}^{d-1}K_{\kappa}(j)\left(\inf_{\tilde{G}_1} d_{W_1}(p_{X_1},p_{\tilde{G}_1(Z_1)}) + \sqrt{d_X}4^{-L_2}\right).
\end{align*}

For the approximation error of the original GAN training, we directly borrow Corollary 5.4 in  \citeSupp{huan;etal;22}. Suppose that the distribution of $Z$ is absolutely continuous on $\mathbb{R}$, and the distribution of $X$ is on $[0,1]^{d_X}$. Then, for the class of standard feedforward ReLU networks $\tilde{{\cal G}}$ with maximal width $\lceil W_1/2\rceil$ and depth $L_1$, it has been investigated that 
\begin{align*}
    \inf_{\tilde{G}_1\in \tilde{{\cal G}}}d_{W_1}(p_{X},p_{\tilde{G}_1(Z)})\leq C_{d_X}(\lceil W_1/2\rceil^2 L_1)^{-1/d_X}
\end{align*} 
for any $\lceil W_1/2\rceil \geq 7d_X+1 $ and $L_1 \geq 2$, where $C_{d_X}$ depends on $d_X$ only. 

Therefore, we characterize the approximation error as 
\begin{align*}
    \inf_{G\in {\cal G}} \bE[d_{W_1}(p_{\Qal|\alpha},p_{G(Z_1,Z_2,\alpha)|\alpha})] \leq
      \underbrace{C_{d_X}(\lceil W_1/2\rceil^2L_1)^{-1/d_X}}_{\text{Approx. Error by GAN}} + \underbrace{\sqrt{d_X} 4^{-L_2}}_{\text{Approx. Error by interpolation}}.
\end{align*}
This reflects that the approximation error due to the linear interpolation decays faster than the distribution approximation. Thus, for sufficiently large $\VD$ such that $L_2$ and $L_1$ increases the same order and $\VW \geq 23 d_X$ (and all the other hidden nodes as well), we conclude that 
\begin{align}
\label{prop:approx_error}
    \inf_{G \in {\cal G}}d_{W_1}(p_{\Qal,\alpha},p_{G(Z_1,Z_2,\alpha),\alpha})\leq \inf_{G\in {\cal G}} \bE[d_{W_1}(p_{\Qal|\alpha},p_{G(Z_1,Z_2,\alpha)|\alpha})] \leq  C_{d_X}'(\lceil \VW/2\rceil^2\VD)^{-1/d_X},
\end{align}
where $C_{d_X}'$ depends on $d_X$ only. Note this result can also be expressed via $Z_{1,2}=(Z_1,Z_2)$, so that $Z_{1,2}$ is in $\mathbb{R}^2$. Therefore, \eqref{prop:approx_error} can also be written as $\inf_{G \in {\cal G}} d_{W_1}(p_{\Qal,\alpha},p_{G(Z,\alpha),\alpha}) \leq C_{d_X}'(\lceil \VW/2\rceil^2\VD)^{-1/d_X}$.
\begin{remark}
    Our approximation error analysis shows that our training scheme based on the convex interpolation may involve extra approximation error compared to the original GAN training. However, this additional error may be negligible compared to the approximation error that stems from finding the push-forwarding neural network. 
\end{remark}

\paragraph{Approximation error with the interpolated reference variables} As discussed in Section~\ref{supp:inter_noise}, the interpolated reference variables $\Zal=\alpha Z_1 + (1-\alpha)Z_2$ can improve the generative performance of $G$. Theoretically, our results also holds when $\Zal$ is used.
Here, we introduce a key idea of the proof that adapts the interpolated input in the approximation error analysis. Suppose $Z\in \mathbb R^d$ is a continuous and bounded reference random vector with $d>1$. All of its margins are iid. Therefore, all margins of $\Zal$ are iid and follow a cumulative density function (CDF) $F_\alpha$.
By universal approximation properties of DNN, there exists a feed-forward neural network of repulsion $f_{\text{NN},i}$ that approximates $F_\alpha$, i.e., $f_{\text{NN},i}(x,\alpha)\approx F_{{\alpha}}(x)$ and $f_{\text{NN},i}(U^{\alpha}_i,\alpha) \approx \mbox{Unif}[0,1]$, where $U^{\alpha}_i$ denotes the $i$th margin of $\Zal$. Then $f_{\text{NN},i}(U^{\alpha}_i,\alpha)$, which approximately follows a uniform distribution, can be regarded as the input of the sub-generator used in Lemma~\ref{lem:approx_interpo}. In summary, the analysis for PTGAN using $\Zal$ consists of three approximation errors from 1) the CDF of $U_{i}^{\alpha}$, 2) the sub-generator's capability, and 3) the interpolation structure, but the second approximation term would dominate in the end.

\subsection{Theorem~\ref{cor:ae}}
\label{proof:ee}
Theorem~\ref{cor:ae} is the direct result of combining Lemmas~\ref{lemma:ee} and~\ref{lem:rade} that appear below. 

\subsubsection{Characterizing the estimation error}
Let's denote by ${\cal R}({\cal F})$ the Rademacher complexity of a generic function class ${\cal F}$. For i.i.d. $X_1,\dots,X_n \sim p_X$, the quantity is defined as ${\cal R}({\cal F})=\expect_{X_1, \epsilon_1,\dots,X_n,\epsilon_n} \left[\sup_{f \in {\cal F}} \left|\frac{1}{n}\sum_{i=1}^n \epsilon_i f(X_i) \right| \right]$ where $\epsilon_1,\dots,\epsilon_n \sim {\rm Unif}\{-1,1\}$ i.i.d. Lemma~\ref{lemma:ee} quantifies this statistical quantity based on the Rademacher complexity ${\cal R}$. The derived bound explicitly relates to the parameters of ${\cal D}$ and ${\cal G}$. For further analysis, we define the composite function class of ${\cal O}=\{D(G(z,a),a): z\in {\cal Z}, a\in [0,1], D\in {\cal D}, G\in{\cal G}\}$. 

\begin{lemma}
\label{lemma:ee}
Under (A1-3), let ${\hat{G}^*} =\arg_{G\in {\cal G}}\min d_{\cal D}(\hat{p}_{Q^{\alpha},\alpha}, \hat{p}_{G(Z,\alpha),\alpha})$. With $1-2\eta$ probability, the estimation error is bounded above by  
\begin{align}
\label{eqn:ee}
    d_{\cal D}&(p_{Q^{\alpha},\alpha}, p_{\hat{G}^*(Z,\alpha),\alpha})- \inf_{G\in {\cal G}}d_{\cal D}(p_{Q^{\alpha},\alpha}, p_{G(Z,\alpha),\alpha}) \notag \\ 
    &\leq 4({\cal R}({\cal D})+{\cal R}({\cal O}))+C_{B_X,B_{\alpha},\bw,\kappa}\sqrt{\dfrac{\log (1/\eta)}{n_e}} + C_{B_{Z},B_{\alpha},\bv,\bw,\psi,\kappa}\sqrt{\dfrac{\log (1/\eta)}{m}},
\end{align}
where $C_{B_X,B_{\alpha},\bw,\kappa}$ and $C_{B_{Z},B_{\alpha},\bv,\bw,\psi,\kappa}$ are specified in the below proof.
% \begin{align*}
%     C_{B_X,B_{\alpha},\bw,\kappa} &= C_{w,\kappa}(d,d-1) \sqrt{B_X^2 + B_{\alpha}^2}\\
%     C_{B_{Z},B_{\alpha},\bv,\bw,\psi,\kappa} &= C_{w,\kappa}(d,d-1) \left(C_{v,\psi}(g,g-1)\sqrt{B_Z^2 + 1}+\sqrt{2}B_{\alpha}\right)
% \end{align*}
\end{lemma}
\begin{proof}
    Following the proof of Theorem 1 in \citeSupp{ji:etal:21}, the estimation error is decomposed as follows
    \begin{align}
        d_{\cal D}(p_{Q^{\alpha},\alpha}, p_{\hat{G}^*(Z,\alpha),\alpha}) - &\inf_{G}d_{\cal D}(p_{Q^{\alpha},\alpha}, p_{G(Z,\alpha),\alpha}) \notag \\
        &= d_{\cal D}(p_{Q^{\alpha},\alpha}, p_{\hat{G}^*(Z,\alpha),\alpha}) - d_{\cal D}(\hat{p}_{Q^{\alpha},\alpha}, p_{\hat{G}^*(Z,\alpha),\alpha})\label{eq:ee1} \\
        &~+ \inf_{G} d_{\cal D}(\hat{p}_{Q^{\alpha},\alpha}, p_{G(Z,\alpha),\alpha}) - \inf_{G}d_{\cal D}(p_{Q^{\alpha},\alpha}, p_{G(Z,\alpha),\alpha})\label{eq:ee2} \\
        &~+ d_{\cal D}(\hat{p}_{Q^{\alpha},\alpha}, p_{\hat{G}^*(Z,\alpha),\alpha}) - \inf_{G}d_{\cal D}(\hat{p}_{\Qal,\alpha}, p_{G(Z,\alpha),\alpha}). \label{eq:ee3}
    \end{align}
    Then \eqref{eq:ee1} and \eqref{eq:ee2} have the  upper bound 
    \begin{align}
    \label{eqn:ee12}
        \eqref{eq:ee1}, \eqref{eq:ee2} \leq \sup_{D} |\expect D(\Qal,\alpha)-\hat{\expect} D(\Qal,\alpha))|,
    \end{align}
    where $\hat{\expect}$ implies the expectation over the empirical mass function. 
    Let's denote $\tilde{G}=\arg_{G}\min d_{\cal D}(\hat{p}_{\Qal,\alpha}, p_{G(Z,\alpha),\alpha})$. Then \eqref{eq:ee3} is bounded above by 
    \begin{align*}
        \eqref{eq:ee3} &= d_{\cal D}(\hat{p}_{\Qal,\alpha}, p_{\hat{G}^*(Z,\alpha),\alpha}) - d_{\cal D}(\hat{p}_{\Qal,\alpha}, \hat{p}_{\hat{G}^*(Z,\alpha),\alpha}) + d_{\cal D}(\hat{p}_{\Qal,\alpha}, \hat{p}_{\hat{G}^*(Z,\alpha),\alpha}) - d_{\cal D}(\hat{p}_{\Qal,\alpha}, p_{\tilde{G}(Z,\alpha),\alpha}), \\
        &\leq d_{\cal D}(\hat{p}_{\Qal,\alpha}, p_{\hat{G}^*(Z,\alpha),\alpha}) - d_{\cal D}(\hat{p}_{\Qal,\alpha}, \hat{p}_{\hat{G}^*(Z,\alpha),\alpha}) + d_{\cal D}(\hat{p}_{\Qal,\alpha}, \hat{p}_{\tilde{G}(Z,\alpha),\alpha}) - d_{\cal D}(\hat{p}_{\Qal,\alpha}, p_{\tilde{G}(Z,\alpha),\alpha}), \\
       &\leq  2 \sup_{D} |\expect D(G(Z,\alpha),\alpha) - \hat{\expect} D(G(Z,\alpha),\alpha)|.
    \end{align*}
    Let $U_1((Q_1^{\alpha_1},\alpha_1),\dots,((Q_n^{\alpha_{n_e}},\alpha_{n_e})))=\sup_{D} |\expect D(\Qal,\alpha)-\hat{\expect}D(\Qal,\alpha)|$. To apply the McDiarmid's inequality, we first check whether or not $U_1$ satisfies the bounded difference condition. We denote by $(\tilde{Q}_j^{\tilde{\alpha}_j},\tilde{\alpha}_j)$ the $j$th random vector independent to $(Q_j^{\alpha_j},\alpha_j)$. Then
    \begin{align*}
|U_1((Q_1^{\alpha_1},\alpha_1),\dots,&(Q_j^{\alpha_j},\alpha_j),\dots, (Q_n^{\alpha_{n_e}},\alpha_{n_e}))-U_1((Q_1^{\alpha_1},\alpha_1),\dots,(\tilde{Q}_j^{\tilde{\alpha}_j},\tilde{\alpha}_j),\dots,(Q_n^{\alpha_{n_e}},\alpha_{n_e}))| \\
 &\leq \dfrac{1}{{n_e}} \sup_{\bw} | D(Q_j^{\alpha_j},\alpha_j) - D(\tilde{Q}_j^{\tilde{\alpha}_j},\tilde{\alpha}_j)|, \\
 &\leq \dfrac{1}{n_e} \prod_{l=1}^d M_w(l) \times \prod_{s=1}^{d-1} K_{\kappa}(s)\times \lVert[Q_j^{\alpha_j} - \tilde{Q}_j^{\tilde{\alpha}_j}, \alpha_j - {\tilde{\alpha}_j}]\rVert,\\
 &\leq \dfrac{1}{n_e} \prod_{l=1}^d M_w(l) \times \prod_{s=1}^{d-1} K_{\kappa}(s)\times \sqrt{2B_X^2 + 2B_{\alpha}^2}=\dfrac{\sqrt{2}}{n_e}C_{B_X,B_{\alpha},\bw,\kappa},
\end{align*}
where the second inequality comes from the Cauchy-Schwarz inequality and Lipschitz activation functions.
Next, the expectation of $U_1$ is 
\begin{align*}
    \expect_{Q,\alpha} U_1 &= \bE_{\Qal,\alpha}\sup_D \bE D(\Qal,\alpha)-\hat{\bE}D(\Qal,\alpha), \\
    &=\bE_{\Qal,\alpha}\sup_D \bE_{\tilde{Q}^{\tilde{\alpha}},\tilde{\alpha}}\left[\frac{1}{n_e}\sum_{j=1}^{n_e}D(\tilde{Q}_j^{\alpha_j},\tilde{\alpha}_j)\right]-\hat{\bE}D(\Qal,\alpha), 
\end{align*}
since $\bE[f(X)]=\bE[\frac{1}{n}\sum_{i=1}^nf(X_i)]$ for i.i.d. random samples. The right-hand side is further bounded above by
\begin{align*}
    &\leq \expect_{Q,\alpha,\tilde{Q},\tilde{\alpha}} \sup_{D} \left \lvert \dfrac{1}{n_e}\sum_{j=1}^{n_e} D(\tilde{Q}_j^{\alpha_j},\tilde{\alpha}_j) - D(Q_j^{\alpha_j},\alpha_j)\right \rvert,\\
    &\leq  \expect_{Q,\alpha,\tilde{Q},\tilde{\alpha},\epsilon} \sup_{D} \left \lvert \dfrac{1}{n_e} \sum_{j=1}^{n_e} \epsilon_j (D(\tilde{Q}_j^{\alpha_j},\tilde{\alpha}_j) - D(Q_j^{\alpha_j},\alpha_j))\right \rvert, \\
    &\leq 2\expect_{Q,\alpha,\epsilon} \sup_{D} \left \lvert \dfrac{1}{n_e}\sum_{j=1}^{n_e} \epsilon_j D(Q_j^{\alpha_j},\alpha_j)\right \rvert=2 {\cal R}({\cal D}).
\end{align*}
Note $\epsilon_j\sim\rm{Unif}\{-1,1\}$ encourages $\epsilon_j (D(\tilde{Q}_j^{\alpha_j},\tilde{\alpha}_j) - D(Q_j^{\alpha_j},\alpha_j))$ to be positive in the sense of taking supremum w.r.t. $D$. Therefore, by the McDiarmid's inequality, \eqref{eqn:ee12} upper bounds
\begin{align*}
    \eqref{eqn:ee12}\leq 2{\cal R}({\cal D}) + C_{B_X,B_{\alpha},\bv,\psi}\sqrt{\dfrac{\log(1/\eta)}{n_e}},
\end{align*}
with $1-\eta$ probability. 
Now, let $U_2(Z_1,\dots,Z_m)=\sup_{D,G} |\expect D(G(Z,\alpha),\alpha)-\hat{\expect}D(G(Z,\alpha),\alpha)|$. $U_2$ satisfies a bounded difference as a result of the Cauchy-Schwarz inequality, i.e., 
\begin{align*}
    |U_2&((Z_1,\alpha_1),\dots, (Z_j,\alpha_j), \dots, (Z_m,\alpha_m)) - U_2((Z_1,\alpha_1),\dots, (\tilde{Z}_j,\tilde{\alpha}_j), \dots, (Z_m,\alpha_m))|\\
    &\leq \dfrac{1}{m} \sup_{D}|D(G(Z_j,\alpha_j),\alpha_j)-D(G(\tilde{Z}_j,\tilde{\alpha}_j),\tilde{\alpha}_j)|, \\
    & \leq \dfrac{1}{m} \prod_{l=1}^d M_w(l) \prod_{s=1}^{d-1} K_{\kappa}(s)\times \lVert [G(Z_j,\alpha_j),\alpha_j] - [G(\tilde{Z}_j,\tilde{\alpha}_j),\tilde{\alpha}_j]\rVert,\\
    & = \dfrac{1}{m} \prod_{l=1}^d M_w(l) \prod_{s=1}^{d-1} K_{\kappa}(s)\times \sqrt{\lVert G(Z_j,\alpha_j) - G(\tilde{Z}_j,\tilde{\alpha}_j)\rVert^2+\lVert \alpha_j -\tilde{\alpha}_j\rVert^2}, \\
    & \leq \dfrac{1}{m} \prod_{l=1}^d M_w(l) \prod_{s=1}^{d-1} K_{\kappa}(s)\times (\lVert G(Z_j,\alpha_j) - G(\tilde{Z}_j,\tilde{\alpha}_j)\rVert+\lVert \alpha_j -\tilde{\alpha}_j\rVert), \\
    & \leq \dfrac{1}{m} \prod_{l=1}^d M_w(l) \prod_{s=1}^{d-1} K_{\kappa}(s)\times \left(\prod_{l=1}^g M_v(l) \prod_{s=1}^{g-1} K_{\psi}(s)\sqrt{\lVert Z_j-\tilde{Z}_j\rVert^2+\lVert \alpha_j -\tilde{\alpha}_j\rVert^2}+2B_{\alpha}\right),\\
    & \leq \dfrac{1}{m} \prod_{l=1}^d M_w(l) \prod_{s=1}^{d-1} K_{\kappa}(s)\times \left(\prod_{l=1}^g M_v(l) \prod_{s=1}^{g-1} K_{\psi}(s)\sqrt{2B_Z^2 + 2B_{\alpha}^2}+2B_{\alpha}\right),\\
    & \leq \dfrac{\sqrt{2}}{m} \prod_{l=1}^d M_w(l) \prod_{s=1}^{d-1} K_{\kappa}(s)\times \left(\prod_{l=1}^g M_v(l) \prod_{s=1}^{g-1} K_{\psi}(s)\sqrt{B_Z^2 + B_{\alpha}^2}+\sqrt{2}B_{\alpha}\right),\\
    & =\dfrac{\sqrt{2}}{m}C_{B_{Z},B_{\alpha},\bv,\bw,\psi,\kappa}.
\end{align*}
Then the expectation of $U_2$ is 
\begin{align*}
    \expect_{Z,\alpha} U_2 &\leq \expect_{Z,\alpha,\tilde{Z},\tilde{\alpha}} \sup_{D} \left\lvert \dfrac{1}{m}\sum_{j=1}^m D(G(\tilde{Z}_j,\tilde{\alpha}_j),\tilde{\alpha}_j) - D(G(Z_j,\alpha_j),\alpha_j)\right\rvert,\\ 
    &\leq  \expect_{Z,\alpha,\tilde{Z},\tilde{\alpha},\epsilon} \sup_{D,G} \left\lvert \dfrac{1}{m} \sum_{j=1}^m \epsilon_j (D(G(\tilde{Z}_j,\tilde{\alpha}_j),\tilde{\alpha}_j) - D(G(Z_j,\alpha_j),\alpha_j))\right\rvert,  \\
    &\leq 2\expect_{Z,\alpha,\tilde{Z},\tilde{\alpha},\epsilon} \sup_{D,G} \left\lvert \dfrac{1}{m} \sum_{j=1}^m \epsilon_j D(G(Z_j,\alpha_j),\alpha_j)\right\rvert=2{\cal R}({\cal O}).
\end{align*}
Therefore, \eqref{eq:ee3} upper bounds 
\begin{align*}
    \eqref{eq:ee3}\leq 2\times \left( 2{\cal R}({\cal O}) +    C_{B_{Z},B_{\alpha},\bv,\bw,\psi,\kappa}\sqrt{\dfrac{\log(1/\eta)}{m}}\right),
\end{align*}
with $1-\eta$ probability by the McDiarmid's inequality. By combining the above upper bounds, we have result in the statement.
\end{proof}

\begin{lemma}
\label{lem:rade}
Under (A1-4), the Rademacher complexities are further bounded
\begin{align*}
    {\cal R}({\cal D}) &\leq 
\dfrac{\sqrt{B_X^2+1}\prod_{l=1}^dM_w(l)\prod_{s=1}^{d-1}K_{\psi}(s)\sqrt{3d}}{\sqrt{n_e}}, \\ 
    {\cal R}({\cal O}) &\leq  \dfrac{\left(1+\prod_{l=1}^d M_w(l)\prod_{s=1}^{d-1}K_{\psi}(s)\right)\prod_{l=1}^g M_v(l)\prod_{s=1}^{g-1}K_{\kappa}(s) (B_Z+1)(\sqrt{(d+g+1)2\log2}+1)}{\sqrt{m}}.
\end{align*}
\end{lemma}

\begin{proof}
In this lemma, the positive homogeneous condition is necessary to characterize the Rademacher complexity of ${\cal D}$ and the composition class induced by $D\circ G$ in terms of the sample size and the characteristics of ${\cal D}$ and ${\cal G}$. By referring to Theorem~1 in \citetSupp{golo:etal:18} and the proof of Corollary~1  in \citeSupp{ji:etal:21}, the Rademacher complexity of ${\cal D}$ upper bounds
\begin{align*}
   {\cal R}({\cal D}) &\leq \dfrac{\bE\left[\sqrt{\sum_{j=1}^{n_e}\lVert[Q_j^{\alpha_j},\alpha_j]\rVert^2}\right]\prod_{l=1}^d M_w(l)\prod_{s=1}^{d-1}K_{\psi}(s)(\sqrt{2d\log 2}+1)}{n_e},\\
   &\leq  \dfrac{\sqrt{B_X^2+1}\prod_{l=1}^dM_w(l)\prod_{s=1}^{d-1}K_{\psi}(s)\sqrt{3d}}{\sqrt{n_e}}, 
\end{align*}
because of $\sqrt{2d\log 2}+1 \leq \sqrt{3d}$ and $\lVert[Q_j^{\alpha_j},\alpha_j]\rVert^2\leq B_X^2+1$.

For the composition function class ${\cal O}=\{D(G(z,\alpha),\alpha):D\in {\cal D}, G\in{\cal G}\}$, the proof has to consider the input $\alpha$ for $D$. Let's denote by the empirical Rademacher complexity $\hat{{\cal R}}({\cal O})=\bE_{\epsilon}\left[\sup_{D,G}\frac{1}{m}\sum_{i=1}^m\epsilon_i D(G(Z_i,\alpha_i),\alpha_i)\right]$. By following the proof of Theorem 1 in \citeSupp{golo:etal:18}, 
\begin{align*}
    m \hat{{\cal R}}({\cal O}) &= \bE_{\epsilon} \sup_{w_d,\cdots,V_1} \sum_{i=1}^m \epsilon_i w_d\psi_{d-1}(W_{d-1}(\cdots([G(Z_i,\alpha_i),\alpha_i])))\\
    &\leq \frac{1}{\lambda}\log \bE_{\epsilon} \sup \exp \left(\lambda \sum_{i=1}^m \epsilon_i w_d\psi_{d-1}(W_{d-1}\cdots)\right)\\
    &\leq \frac{1}{\lambda}\log \bE_{\epsilon} \sup \exp \left(\lambda \lVert w_d \rVert \left\lVert \sum_{i=1}^m \epsilon_i \psi_{d-1}(W_{d-1}\cdots)\right\rVert \right)\\
    &\leq \frac{1}{\lambda}\log \left(2\cdot \bE_{\epsilon} \sup \exp \left(\lambda M_w(d)K_{\psi}(d-1)\left\lVert \sum_{i=1}^m \epsilon_i W_{d-1}(\psi_{d-2}(\cdots))\right\rVert \right)\right) 
\end{align*}
where the last inequality comes from Lemma~1 in \citeSupp{golo:etal:18}. Let $C_D=\prod_{l=1}^dM_w(l)\prod_{s=1}^{d-1}K_{\psi}(s)$.  By the same peeling-off argument, the last line is bounded above 
\begin{align*}
    &\leq \frac{1}{\lambda}\log \left(2^d\cdot \bE_{\epsilon} \sup_G \exp \left(\lambda C_D\left\lVert \sum_{i=1}^m \epsilon_i [G(Z_i,\alpha_i),\alpha_i]\right\rVert \right)\right) \\
    &\leq \frac{1}{\lambda}\log \left(2^d \bE_{\epsilon} \sup_G \exp \left(\lambda C_D\left\lVert \sum_{i=1}^m \epsilon_i G(Z_i,\alpha_i)\right\rVert+\lambda C_D\left\lvert \sum_{i=1}^m \epsilon_i\alpha_i \right\rvert \right)\right)\\
    &\leq \frac{1}{\lambda}\log \left(2^d \bE_{\epsilon} \sup_G \exp \left(p\frac{\lambda C_D}{p}\left\lVert \sum_{i=1}^m \epsilon_i G(Z_i,\alpha_i)\right\rVert+(1-p)\frac{\lambda C_D}{1-p}\left\lvert \sum_{i=1}^m \epsilon_i\alpha_i \right\rvert \right)\right),
\end{align*}
for some $0<p<1$. Since $\exp(x)$ is convex, the last line is further bounded by
\begin{align*}
    \leq \frac{1}{\lambda}\log \left(2^d\cdot\underbrace{\bE_{\epsilon} \sup_G p\exp \left(\frac{\lambda C_D}{p}\left\lVert \sum_{i=1}^m \epsilon_i G(Z_i,\alpha_i)\right\rVert\right)}_{\text{(I)}}+2^d\cdot\bE_{\epsilon}(1-p)\exp\left(\frac{\lambda C_D}{1-p}\left\lvert \sum_{i=1}^m \epsilon_i\alpha_i \right\rvert \right)\right),
\end{align*}
where, by the same peeling-off argument, 
\begin{align*}
    \text{(I)}\leq 2^g p\bE_{\epsilon} \exp\left(\frac{\lambda C_D}{p}\prod_{l=1}^g M_v(l)\prod_{s=1}^{g-1}K_{\kappa}(s)\left\lVert\sum_{i=1}^m \epsilon_i[Z_i,\alpha_i]\right\rVert\right).
\end{align*}
Let $C_G=\prod_{l=1}^g M_v(l)\prod_{s=1}^{g-1}K_{\kappa}(s)$ and write
\begin{align*}
    m \hat{\cal R}({\cal O})\leq \dfrac{1}{\lambda}\log\left(2^{d+g} p\bE_{\epsilon}\exp\left(\frac{\lambda C_D C_G}{p}\left\lVert\sum_{i=1}^m \epsilon_i[Z_i,\alpha_i]\right\rVert\right)+2^d(1-p)\bE_{\epsilon }\exp\left(\frac{\lambda C_D}{1-p}\left\lvert \sum_{i=1}^m \epsilon_i\alpha_i \right\rvert\right)\right).
\end{align*}
By setting $p=\frac{\lambda C_D C_G}{\lambda C_G + \lambda C_DC_G}$ and $k=\lambda C_G +\lambda C_DC_G$, then the last line is bounded above by 
\begin{align*}
    \leq \dfrac{1}{\lambda }\log\left(2^{d+g} \bE_{\epsilon}\exp\left(k\left\lVert\sum_{i=1}^m \epsilon_i[Z_i,\alpha_i]\right\rVert\right)+2^{d+g}\bE_{\epsilon }\exp\left(k\left\lvert \sum_{i=1}^m \epsilon_i\alpha_i \right\rvert\right)\right).
\end{align*}
Since $\left\lVert\sum_{i=1}^m \epsilon_i[Z_i,\alpha_i]\right\rVert\leq \lVert \sum_{i=1}^m\epsilon _i Z_i\rVert + \lvert\sum_{i=1}^m\epsilon_i\alpha_i\rvert $, the last line is simplified to 
\begin{align}
\label{eqn:RO}
    &\leq \dfrac{1}{\lambda }\log\left(2^{d+g+1} \bE_{\epsilon}\exp\left(\lambda(C_G+C_DC_G)\left(\left\lVert\sum_{i=1}^m \epsilon_i Z_i\right\rVert+\left\lvert \sum_{i=1}^m \epsilon_i\alpha_i\right\rvert \right)\right)\right)\\
    &=\dfrac{1}{\lambda}\log\left(2^{d+g+1}\bE_{\epsilon}\exp(\lambda U)\right)\notag,
\end{align}
where $U=C_* (\left\lVert\sum_{i=1}^m \epsilon_i Z_i\right\rVert+\left\lvert \sum_{i=1}^m \epsilon_i\alpha_i\right\rvert)$ with $C_*=C_G+C_DC_G$. 

Now we observe that $\bE[U]\leq C_*\left(\sqrt{\bE_{\epsilon }\lVert\sum_{i=1}^m \epsilon_i Z_i\rVert^2}+\sqrt{\bE_{\epsilon}\lvert \sum_{i=1}^m \epsilon_i\alpha_i\rvert^2}\right)$ by the linearity of the expectation and the Jensen's inequality. It is straightforward to show $\bE_{\epsilon }\lVert\sum_{i=1}^m \epsilon_i Z_i\rVert^2=\lVert\sum_{i=1}^mZ_i\rVert^2\leq mB_Z^2$ and also $\bE_{\epsilon}[\lvert\sum_{i=1}^m\epsilon_i\alpha_i\rvert^2]\leq m$ since $\epsilon_i \sim {\rm Unif}\{1,-1\}$ i.i.d. and the support of $Z$ and $\alpha$ is bounded. Moreover, we observe that 
\begin{align*}
&U(\epsilon_1,\dots,\epsilon_i,\dots,\epsilon_m)-U(\epsilon_1,\dots,-\epsilon_i,\dots,\epsilon_m)\\
&\leq C_*\left(\left\lVert \sum_{i=1}^m \epsilon_i Z_i - \sum_{i=1}^m\epsilon_i' Z_i\right \rVert +\left\lvert \sum_{i=1}^m\epsilon_i\alpha_i - \sum_{i=1}^m \epsilon_i'\alpha_i\right\rvert\right)\\ 
&\leq 2C_*\left(\left\lVert Z_i \right \rVert + \left\lvert \alpha_i \right \rvert \right),
\end{align*}
by $\lVert x\rVert- \lVert y\rVert\leq \lVert x-y\rVert$. Due to this bounded difference condition, $U-\bE_{\epsilon}(U)$ is a sub-Gaussian, and therefore it satisfies 
\begin{align*}
    \dfrac{1}{\lambda}\log \left(\bE_{\epsilon }\exp \left(\lambda(U-\bE_{\epsilon}(U)\right)\right)\leq \dfrac{1}{\lambda}\dfrac{\lambda^2 \sigma^2_U}{2}.
\end{align*}
where $\sigma^2_U=C_*^2\sum_{i=1}^m\left(\lVert Z_i\rVert +\lvert \alpha_i\rvert\right)^2$. Therefore, 
\begin{align*}
    \dfrac{1}{\lambda}\log\left(2^{d+g+1}\bE_{\epsilon}\exp(\lambda U)\right)&=\dfrac{(d+g+1)\log 2}{\lambda}+\dfrac{1}{\lambda}\log\left(\bE_{\epsilon}\exp(\lambda(U-\bE_{\epsilon}(U))\right)+\bE_{\epsilon}(U), \\
    &\leq \dfrac{(d+g+1)\log 2}{\lambda}+\dfrac{\lambda C_*^2\sum_{i=1}^m\left(\lVert Z_i\rVert +\lvert \alpha_i\rvert\right)^2}{2} + \sqrt{m}C_*(B_Z+1).
\end{align*}
Now, set $\lambda=\frac{\sqrt{(d+g+1)2\log 2}}{C_*\sqrt{\sum_{i=1}^m\left(\lVert Z_i\rVert +\lvert \alpha_i\rvert\right)^2}}$, and therefore,
\begin{align*}
m \hat{{\cal R}}({\cal O})\leq \sqrt{m}C_* (B_Z+1)\left(\sqrt{(d+g+1)2\log 2}+1\right),
\end{align*}
so
\begin{align*}
    {\cal R}({\cal O})=\bE\left[\hat{{\cal R}}({\cal O})\right]\leq \dfrac{C_* (B_Z+1)\sqrt{(d+g+1)2\log2}+1}{\sqrt{m}}.
\end{align*}
where $C_*=\left(1+\prod_{l=1}^d M_w(l)\prod_{s=1}^{d-1}K_{\psi}(s)\right)\prod_{l=1}^g M_v(l)\prod_{s=1}^{g-1}K_{\kappa}(s)$.

% {\color{red}
% For ${\cal O}$, we have 
% \begin{align*}
%     {\cal R}({\cal O}) &\leq \dfrac{\lVert [G(Z_j,\alpha_j),\alpha_j]\rVert\prod_{l=1}^dM_w(l)\prod_{s=1}^{d-1}K_{\psi}(s)\sqrt{3d}}{\sqrt{m}},\\
%     &\leq \dfrac{\lVert G(Z_j,\alpha_j)\rVert\prod_{l=1}^dM_w(l)\prod_{s=1}^{d-1}K_{\psi}(s)\sqrt{3d}}{\sqrt{m}} + \dfrac{\lVert \alpha_j \rVert \prod_{l=1}^dM_w(l)\prod_{s=1}^{d-1}K_{\psi}(s)\sqrt{3d}}{\sqrt{m}},\\
%     &\leq \dfrac{\sqrt{B_Z^2+B_{\alpha}^2}\prod_{l=1}^g M_{v}(l)\prod_{s=1}^{g-1}K_{\kappa}(s)\times \prod_{l=1}^dM_w(l)\prod_{s=1}^{d-1}K_{\psi}(s)\sqrt{3(g+d-1)}}{\sqrt{m}} \\ 
%     & ~~~ + \dfrac{B_{\alpha}\prod_{l=1}^dM_w(l)\prod_{s=1}^{d-1}K_{\psi}(s)\sqrt{3d}}{\sqrt{m}}, \\ 
%     & = \dfrac{\prod_{l=1}^dM_w(l)\prod_{s=1}^{d-1}K_{\psi}(s)}{\sqrt{m}}\left( \sqrt{B_Z^2+1}\prod_{l=1}^g M_{v}(l)\prod_{s=1}^{g-1}K_{\kappa}(s) \sqrt{3(g+d-1)} + 1\times \sqrt{3d} \right),
% \end{align*}
% since $B_{\alpha}=1$. Note the term $(g+d-1)$ conceptually stands for the depth of the composite neural network, so it stems from the part $\lVert W_1 V_g\rVert_F \leq M_w(1)M_v(g)$.
% }

\begin{remark}
\label{remark:dropa4}
The positive homogeneous condition can be alleviated to include other nonlinear-type activation functions, e.g., Tanh, shifted Sigmoid, etc. \citeSupp{golo:etal:18} showed that the Rademacher complexity in Lemma~\ref{lem:rade} can also be characterized with the Lipschitz activation function $\sigma(x)$ satisfying $\sigma(0)=0$ if the maximal 1-norm of the rows of the weight matrices (i.e., $\lVert W_i\rVert_{1,\infty}=\max_j \lVert \bw_{i,j} \rVert_1\leq M_w(i)$ and also for $V_i$) are bounded. The following paragraph explains how to characterize ${\cal R}({\cal O})$ in Lemma~\ref{lem:rade} with the maximal 1-norm condition more specifically. %This work sticks to the assumption A1 in the main text to compare our analysis with the existing literature \citepSupp{ji:etal:21}; refer to the discussion in the last paragraph of Section~\ref{sec:parallel}.
\end{remark}

\paragraph{Under the maximal 1-norm condition} Deriving Lemma~\ref{lemma:ee} and characterizing ${\cal R}({\cal D})$ under the maximal 1-norm condition are straightforward based on the following basic property. For a matrix $A$ and a vector $b$, it follows that $\lVert Ab\rVert_{\infty} = \max_j \lvert {\bf a}_j^{\top}b \rvert=\max_j\sum_{j,i} \lvert a_{j,i} b_i \rvert \leq \max_j\sum_{j,i} \lvert a_{j,i}\rvert  \lVert b \rVert_{\infty}= \lVert A\rVert _{1,\infty}\lVert b\rVert_{\infty}$ where $\lVert (x_1,\dots,x_k)\rVert_{\infty}=\max_i \lvert x_i\rvert$ and ${\bf a}_i$ be the $i$th row vector of $A$. 

Here we provide a proof for the characterization of ${\cal R}({\cal O})$ in detail which includes the concatenation layer in ${\cal O}$ when transiting from $D$ to $G$. Suppose the weight matrices of $W_i$ and $V_i$ satisfy the bounded maximal 1-norm, instead of the Frobenius norm. Following the proof in the above with the modified condition and Lemma~2 in \citeSupp{golo:etal:18}, the step in \eqref{eqn:RO} can be shown as
\begin{align*}
    &\leq \dfrac{1}{\lambda }\log\left(2^{d+g+1} \bE_{\epsilon}\exp\left(\lambda C_*\left(\left\lVert\sum_{i=1}^m \epsilon_i Z_i\right\rVert_{\infty}+\left\lvert \sum_{i=1}^m \epsilon_i\alpha_i\right\rvert \right)\right)\right) \\
    &= \dfrac{1}{\lambda }\log\left(2^{d+g+1} \bE_{\epsilon}\exp\left(\lambda C_*\left(\max_j \left \lvert\sum_{i=1}^m \epsilon_i Z_{i,j}\right\rvert+\left\lvert \sum_{i=1}^m \epsilon_i\alpha_i\right\rvert \right)\right)\right)\\ 
    &\leq  \dfrac{1}{\lambda }\log\left(2^{d+g+1} \sum_{j=1}^{d_Z}\bE_{\epsilon}\exp\left(\lambda C_*\left(\left \lvert\sum_{i=1}^m \epsilon_i Z_{i,j}\right\rvert+\left\lvert \sum_{i=1}^m \epsilon_i\alpha_i\right\rvert \right)\right)\right).
\end{align*}
Now, let $A=\sum_{i=1}^m \epsilon_i Z_{i,j}$ and $B=\sum_{i=1}^m \epsilon_i\alpha_i$. By using the relationship $\exp(|x|)\leq \exp(x)+\exp(-x)$, we first observe that $\exp(\lambda C_*(|A|+|B|))$ is bounded above 
\begin{align*}
    \exp(\lambda C_*|A|)\exp(\lambda C_*|B|)\leq (\exp(\lambda C_*A)+\exp(-\lambda C_*A))(\exp(\lambda C_*B)+\exp(-\lambda C_*B)), 
\end{align*}
and also $\bE_{\epsilon}\exp(\lambda C_*(A+B))=\bE_{\epsilon}\exp(-\lambda C_*(A+B))$ and $\bE_{\epsilon}\exp(\lambda C_*(A-B))=\bE_{\epsilon}\exp(-\lambda C_*(A-B))$. In the meantime, $\bE_{\epsilon}\exp(\lambda C_*(A+B))$ is characterized to
\begin{align*}
    \prod_{i=1}^m\bE_{\epsilon} \exp\left(\lambda C_* \epsilon_i (Z_{i,j}+\alpha_i)\right)&=\prod_{i=1}^m\dfrac{\exp\left(\lambda C_* (Z_{i,j}+\alpha_i)\right)+\exp\left(-\lambda C_* (Z_{i,j}+\alpha_i)\right)}{2}, \\ 
    &\leq \exp\left(\dfrac{\lambda^2C_*^2\sum_{i=1}^m (Z_{i,j}+\alpha_i)^2}{2}\right),
\end{align*}
using the property $(\exp(x)+\exp(-x))/2\leq \exp(x^2/2)$, and also similarly, we have
\begin{align*}
  \bE_{\epsilon}\exp(\lambda C_*(A-B))  \leq \exp\left(\dfrac{\lambda^2C_*^2\sum_{i=1}^m (Z_{i,j}-\alpha_i)^2}{2}\right),
\end{align*}
Therefore, 
\begin{align*}
    &\sum_{j=1}^{d_Z}\bE_{\epsilon}\exp\left(\lambda C_*\left(\left \lvert\sum_{i=1}^m \epsilon_i Z_{i,j}\right\rvert+\left\lvert \sum_{i=1}^m \epsilon_i\alpha_i\right\rvert \right)\right)\\
    &\leq \sum_{j=1}^{d_Z}2\exp\left(\dfrac{\lambda^2C_*^2\sum_{i=1}^m (Z_{i,j}+\alpha_i)^2}{2}\right)+\sum_{j=1}^{d_Z}2\exp\left(\dfrac{\lambda^2C_*^2\sum_{i=1}^m (Z_{i,j}-\alpha_i)^2}{2}\right),\\
    &\leq \sum_{j=1}^{d_Z}4\exp\left(\dfrac{\lambda^2C_*^2\sum_{i=1}^m (|Z_{i,j}|+|\alpha_i|)^2}{2}\right),\\
    &\leq 4 d_Z  \max_{j}\exp\left(\dfrac{\lambda^2C_*^2\sum_{i=1}^m (|Z_{i,j}|+|\alpha_i|)^2}{2}\right).
\end{align*}
Thus, 
\begin{align*}
    m{\hat{\cal R}}({\cal O})\leq \dfrac{(d+g+1)\log2 + \log4d_Z}{\lambda}+\dfrac{\lambda C_*^2}{2}\max_j \sum_{i=1}^m (|Z_{i,j}|+|\alpha_i|)^2,
\end{align*}
By setting $\lambda=\sqrt{\frac{(d+g+1)\log2+\log d_Z}{C_*^2 \max_j \sum_{i=1}^m (|Z_{i,j}|+|\alpha_i|)^2}}$, we observe that the complexity relies on $m^{-1/2}$. 

\end{proof}
 
\subsection{Theorem~\ref{thm:lbd}}
We introduce the Fano's lemma in \citeSupp{ji:etal:21}.
\begin{lemma*}[Fano's Lemma]
For $M\geq 2$, assume that there exists $M$ hypotheses $\theta_0, \dots, \theta_M \in \Theta$ satisfying (i) $d(\theta_i,\theta_j) \geq 2s >0$ for all $0\leq i<j\leq M$; (ii) $\frac{1}{M}\sum_{i=1}^M KL(P_{\theta_i}||P_{\theta_0})\leq \alpha \log M$, $0<\alpha \leq 1/8$, where $d(\cdot,\cdot)$ is a semi-distance and $P_{\theta}$ is a probability measure with respect to the randomness of data $D$. Then, we have 
\begin{align*}
\inf_{\hat{\theta}} \sup_{\theta\in\Theta} P_{D\sim P_{\theta}}\left[ d(\hat{\theta},\theta) \geq s \right] \geq \dfrac{\sqrt{M}}{1+\sqrt{M}}\left(1-2\alpha - \dfrac{2\alpha}{\log M}\right).
\end{align*}
\end{lemma*}

Now, let's consider the following hypothetical distribution 
    \begin{align*}
        p_u(q,\alpha) &= \begin{cases}
                        1/4 - u \delta, & \mbox{if } q=q_1,\alpha=1, \\
                        1/4 + u\delta, & \mbox{if } q=-q_1,\alpha=1, \\
                        1/4 - u \delta, & \mbox{if } q=q_1,\alpha=0, \\
                        1/4 + u\delta, & \mbox{if } q=-q_1,\alpha=0,
                        \end{cases}
    \end{align*}
    where $\lVert q_1\rVert=B_X$ for $q_1, -q_1 \in {\cal Q}$. % i.e., $q_1=(B_X,{\bf 0})$ or $q_1=(-B_X,{\bf 0})$ in $\mathbb{R}^{d_X}$. 
    
    The $(l,k)$th element of $W_{i}$ for $1\leq i<d$ is denoted by $W_{i,l,k}$. The $k$th column vector of the $i$th layer is denoted by $W_{i,\cdot,k}$. The final layer $w_d$ is a $N_d^D \times 1$ vector, and $w_{d,l}$ denotes the $l$th element. We select $\bw^{\dagger} \in \bW$ such that $w^{\dagger}_{d,1}=M_w(d)$, $w^{\dagger}_{d,l}=0$ for $l\neq 1$, $W^{\dagger}_{i,1,1}=M_w(i)$ for $2\leq i\leq d-1$, $W^{\dagger}_{i,l,k}=0$ for $(l,k)\neq (1,1)$, $W^{\dagger}_{1,\cdot,1}=M_w(1)\frac{\tilde{q}}{\lVert \tilde{q}\rVert}$, and $W^{\dagger}_{1,\cdot,l}={\bf 0}$ for $l\neq 1$ where ${\bf 0}$ is a zero vector and $\tilde{q}=(q_1,1)$. Then the value of $D$ at each point is  
    \begin{align*}
        D(q,\alpha) = \begin{cases}
            M_w(d)\left(\kappa_{d-1}\left(\cdots M_w(1)\sqrt{B_X^2 + 1}\right)\right) 
            &\mbox{if } q=q_1,\alpha=1, \\ 
            M_w(d)\left(\kappa_{d-1}\left(\cdots M_w(1)\frac{B_X^2}{\sqrt{B_X^2 + 1}}\right)\right) & \mbox{if } q=q_1,\alpha=0, \\
            M_w(d)\left(\kappa_{d-1}\left(\cdots M_w(1)\frac{1-B_X^2}{\sqrt{B_X^2 + 1}}\right)\right) & \mbox{if } q=-q_1,\alpha=1, \\
             M_w(d)\left(\kappa_{d-1}\left(\cdots M_w(1) \frac{-B_X^2}{\sqrt{B_X^2+1}} \right)\right) & \mbox{if } q=-q_1,\alpha=0. 
        \end{cases}
    \end{align*}

     \noindent For $0\leq i < j \leq 2$, the neural distance $d$ is described as follows, 
    \begin{align*}
        d(p_i,p_j)&=\sup_{D} |\bE_{p_i} D(\Qal,\alpha)-\bE_{p_j} D(\Qal,\alpha)| \\
        &= (j-i)\delta |(D(q_1,1)-D(-q_1,1)) + (D(q_1,0)-D(-q_1,0))|,\\
        &\geq \delta  |(D(q_1,1)+D(q_1,0)) - (D(-q_1,0)+D(-q_1,1))|, 
    \end{align*}
    \noindent On the basis of the distribution, we set $\delta = \log(2)/(80\sqrt{n}) < 0.005$. 
    \begin{align*}
        n \text{KL}(p_i || p_0)&=2n\left(\dfrac{1}{4}-i\delta\right)\log(1-4i\delta)+2n\left(\dfrac{1}{4}+i\delta\right)\log(1+4i\delta), \\ 
        &=\dfrac{n}{2}\log(1-4^2i^2 \delta^2) + 2ni\delta\log\left(1+\dfrac{8i\delta}{1-4i\delta}\right), \\
        & \leq n 4^2 i^2 \delta^2 \left(\dfrac{1}{2}\times \dfrac{1+4i\delta}{1-4i\delta}\right),\\
        & \leq n 4^2 i^2 \delta^2,
    \end{align*}
    so we have 
    \begin{align*}
        \dfrac{1}{2}\sum_{i=1}^2n\text{KL}(p_i||p_0)\leq 80n\delta^2\leq \dfrac{\log (2)}{80}\log(2).
    \end{align*} 
    Hence, by Fano's lemma, we obtain 
    \begin{align*}
        \inf_{\hat{p}_n} \sup_{p_{\Qal,\alpha} \in {\cal P}_{{\cal Q},[0,1]}} P\left[d(p_{\Qal,\alpha},\hat{p}_n) \right] \geq \dfrac{\sqrt{2}}{1+\sqrt{2}}\left(\dfrac{39}{40}-\dfrac{\log(4)}{40}\right) > 0.55.
    \end{align*}

\subsection{Proposition~\ref{prop:var_linear}}

Let's denote $W_1 = [W_{1,1},W_{1,2}]$ with $W_{1,1}\in \mathbb{R}^{p_1}, W_{1,2}\in \mathbb{R}$. The derivative w.r.t. $W_{1,1}$ can be expressed as 
    \begin{align*}
     \dfrac{\partial \hat{L}_b^{\alpha}(\Dt,\Gt)}{\partial W_{1,1}}&=\frac{1}{n_b}\sum_{i=1}^{n_b}Q^{\alpha_i}_i - \frac{1}{m_b} \sum_{j=1}^{m_b}\Gt(Z^{\alpha_j}_j,\alpha_j). 
    \end{align*}
    By the iterative rule of the covariance, 
    \begin{align*}
        \cov(Q^{\alpha_i}_i)&=\bE(\cov(Q^{\alpha_i}_i|\alpha_i)) + \cov(\bE(Q^{\alpha_i}_i|\alpha_i)), \\
        &=\bE((\alpha_i^2 + (1-\alpha_i)^2)\cov(X_1)).
    \end{align*}
    Note $\cov(\bE(Q^{\alpha_i}_i|\alpha_i))=\cov(\alpha_i \bE(X_i) + (1-\alpha_i)\bE(X_j))=0$. Since $\alpha_i\sim r \delta_1(\cdot) + (1-r) p_{\alpha^*}(\cdot)$,
    \begin{align*}
        \bE((\alpha_i^2 + (1-\alpha_i)^2)\cov(X_1))&=(r+(1-r)\bE_{\alpha\sim{\rm Unif}[0,1]}(2\alpha^2-2\alpha+1))\cov(X_1), \\
        &= \left(\dfrac{r}{3}+\dfrac{2}{3}\right)\cov(X_1). 
    \end{align*}
    By the assumption of the generator, we also obtain $\cov(G(Z^{\alpha_i}))=\bE(\alpha_i^2 + (1-\alpha_i)^2 \cov(G(Z_1,1)))$. Therefore, 
    \begin{align}
    \label{var_W11}
        \cov\left(\dfrac{\partial \hat{L}^{\alpha}_b(\Dt,\Gt)}{\partial W_{1,1}}\right)= \left(\dfrac{2}{3}+\dfrac{r}{3}\right)\left(\dfrac{\cov(X_1)}{n_b}+\dfrac{\cov(G(Z_1,1))}{m_b}\right). 
    \end{align}
    For the single parameter $W_{1,2}$, the derivative is 
    \begin{align*}
        \dfrac{\partial \hat{L}^{\alpha}_b(\Dt,\Gt)}{\partial W_{1,2}}=\dfrac{1}{n_b}\sum_{i=1}^{n_b} \alpha_i - \dfrac{1}{m_b}\sum_{j=1}^{m_b} \alpha_j,
    \end{align*}
    and its variance is
    \begin{align*}
        \var\left(\dfrac{\partial \hat{L}^{\alpha}_b(\Dt,\Gt)}{\partial W_{1,2}}\right)=\var(\alpha_i)\left(\dfrac{1}{n_b}+\dfrac{1}{m_b}\right), %\leq \dfrac{1}{9}\left(\dfrac{1}{n_b}+\dfrac{1}{m_b}\right),
    \end{align*}
    where the maximum $0\leq \var(\alpha_i)\leq \frac{1}{9}$ is found at $r=1/3$ and the variance is 0 at $r=1$. On the other hand, the counterpart gradient's variance is 
     \begin{align*}
         \text{tr}\left(\cov\left(\dfrac{\partial \hat{L}_b^1(\Dt,\Gt)}{\partial W_1}\right)\right)=\left(\dfrac{\text{tr}(\cov(X_1))}{n_b}+\dfrac{\text{tr}(\cov(G(Z_1,1)))}{m_b}\right).
     \end{align*}
     Therefore, \eqref{var_W11} implies that 
     \begin{align*}
         \text{tr}\left(\cov\left(\dfrac{\partial \hat{L}^{\alpha}_b(\Dt,\Gt)}{\partial W_{1,1}}\right)\right)&= \left(\dfrac{2}{3}+\dfrac{r}{3}\right)\left(\dfrac{\text{tr}(\cov(X_1))}{n_b}+\dfrac{\text{tr}(\cov(G(Z_1,1)))}{m_b}\right) \\ 
         &=\left(\dfrac{2}{3}+\dfrac{r}{3}\right)\text{tr}\left(\cov\left(\dfrac{\partial \hat{L}_b^1(\Dt,\Gt)}{\partial W_1}\right)\right),
     \end{align*}
    Note the derivative of $W_{1,2}$ for $\hat{L}_b^1$ has no variability. Hence, we have 
    {\small 
    \begin{align*}
        \text{tr}\left(\cov\left(\dfrac{\partial \hat{L}^{\alpha}_b(\Dt,\Gt)}{\partial W_{1}}\right)\right)-\var\left(\dfrac{\partial \hat{L}^{\alpha}_b(\Dt,\Gt)}{\partial W_{1,2}}\right)=\left(\dfrac{2}{3}+\dfrac{r}{3}\right)\text{tr}\left(\cov\left(\dfrac{\partial \hat{L}_b^1(\Dt,\Gt)}{\partial W_1}\right)\right).
    \end{align*}
    }
    We observe
    {\small 
    \begin{align*}
        \text{tr}\left(\cov\left(\dfrac{\partial \hat{L}^{\alpha}_b(\Dt,\Gt)}{\partial W_{1}}\right)\right)-\text{tr}\left(\cov\left(\dfrac{\partial \hat{L}_b^1(\Dt,\Gt)}{\partial W_1}\right)\right) \\
        = \left(\dfrac{r-1}{3}\right)\text{tr}\left(\cov\left(\dfrac{\partial \hat{L}_b^1(\Dt,\Gt)}{\partial W_1}\right)\right)+\var(\alpha_i)\left(\dfrac{1}{n_b}+\dfrac{1}{m_b}\right).
    \end{align*}
    Therefore, the condition $3\var(\alpha_i)\leq \text{tr}(\cov(X_1))+\text{tr}(\cov(G(Z_1,1)))$ is found. 
    }
\newpage

\section{Implementation for PTGAN}
\label{supp:algorithm}

%The optimization for PTGAN involves the target neural distance $d_{\cal D}(p_{\Qal,\alpha}, p_{G(Z, \alpha),\alpha}) = \sup_{D \in {\cal D}}\{{\bf E}_{\Qal,\alpha} [D(\Qal, \alpha)] - {\bf E}_{Z,\alpha}[D(G(Z,\alpha),\alpha)]\}$ and the penalty in \eqref{penalty}. 
This section explains the implementation in detail for PTGAN and FairPTGAN and suggests using interpolated reference noises as well to advance the flexibility of the generator. The implementation of PTGAN is similar to the usual GAN training except for the construction of training samples at every iteration. 

\subsection{Algorithm}

\paragraph{PTGAN} Algorithm~\ref{alg:ptgan} describes handling the discrepancy term $d_{\cal D}(p_{\Qal,\alpha}, p_{G(Z, \alpha),\alpha}) = \sup_{D \in {\cal D}}\{{\bf E}_{\Qal,\alpha} [D(\Qal, \alpha)] - {\bf E}_{Z,\alpha}[D(G(Z,\alpha),\alpha)]\}$ and the coherency penalty \eqref{penalty} within the gradient descent/ascent framework. The algorithm consists of mainly four parts: 1) creating minibatch for $\Qal$ (Algorithm~\ref{alg:minibatch}), 2) constructing the penalty $H$ and the minibatch loss $\hat{L}_b$, 3) taking the gradient-ascent step for $\Dt$, and 4) taking the gradient-descent step for $\Gt$. In this work, we specify $T'=1$ and $\lambda=100$ in all simulation studies. Note, in Algorithm~\ref{alg:minibatch}, $q_{(i)}^{(2)}$, $\nu_{(i)}$, and $\alpha_{(i)}^{(2)}$ are for the penalty term. 

\paragraph{FairPTGAN} The optimization scheme shares Algorithm~\ref{alg:ptgan} except for the minibatch constrution by replacing MC($\{x_i\}_{i=1}^{n}$) (Algorithm~\ref{alg:minibatch}) with MC($\{x_i^{(0)}\}_{i=1}^{n_0},\{x_i^{(1)}\}_{i=1}^{n_1}$) (Algorithm~\ref{alg:fair-minibatch}). Algorithm~\ref{alg:fair-minibatch} shows how to construct a minibatch where observed samples are partitioned in accordance with the binary group label $A\in \{0,1\}$. 

\begin{algorithm}[ht!]
\caption{Minibatch Construction (MC) for \textbf{PTGAN}}\label{alg:minibatch}
\KwData{$\{x_i\}_{i=1}^{n}$. The subscript $(i)$ denotes the $i$th randomly selected sample.} 
\KwResult{$\{q_{(i)}^{(1)}\}_{i=1}^{n_b}$, $\{q_{(i)}^{(2)}\}_{i=1}^{n_b}$, $\{\tilde{q}_{(i)}\}_{i=1}^{n_b}$, $\{\alpha_{(i)}^{(1)}\}_{i=1}^{n_b}$, and $\{\tilde{\alpha}_{(i)}\}_{i=1}^{n_b}$}

Randomly choose $\{x_{(i)}\}_{i=1}^{n_b}$ and $\{x_{(i)}'\}_{i=1}^{n_b}$ from $\{x_i\}_{i=1}^n$ independently; \\
Generate $\{\alpha_{(i)}^{(1)}\}_{i=1}^{n_b}\sim p_{\alpha}$, $\{\alpha_{(i)}^{(2)}\}_{i=1}^{n_b}\sim {\rm Unif}(0,1)$, and $\{\nu_{(i)}\}_{i=1}^{n_b}\sim {\rm Unif}(0,1)$; \\

Create $\{q_{(i)}^{(1)}=\alpha_{(i)}^{(1)} x_{(i)} + (1-\alpha_{(i)}^{(1)}) x_{(i)}'\}_{i=1}^{n_b}$, $\{q_{(i)}^{(2)}=\alpha_{(i)}^{(2)} x_{(i)} + (1-\alpha_{(i)}^{(2)}) x_{(i)}'\}_{i=1}^{n_b}$, $\{\tilde{q}_{(i)}=\nu_{(i)} q_{(i)}^{(1)} + (1-\nu_{(i)}) q_{(i)}^{(2)}\}_{i=1}^{n_b}$, and $\{\tilde{\alpha}_{(i)}=\nu_{(i)}\alpha_{(i)}^{(1)} + (1-\nu_{(i)})\alpha_{(i)}^{(2)}\}_{i=1}^{n_b}$; \\ 
\end{algorithm}

\begin{algorithm}[ht!]
\caption{Minibatch Construction (MC) for \textbf{FairPTGAN}}\label{alg:fair-minibatch}
\KwData{$\{x_i^{(0)}\}_{i=1}^{n_0}$ and $\{x_i^{(1)}\}_{i=1}^{n_1}$ are the sets of either $A=0$ or $A=1$. The subscript $(i)$ denotes the $i$th randomly selected sample. Let $n_b'=n_b/2 < n_0,n_1$.}
\KwResult{$\{q_{(i)}^{(1)}\}_{i=1}^{n_b}$, $\{q_{(i)}^{(2)}\}_{i=1}^{n_b}$, $\{\tilde{q}_{(i)}\}_{i=1}^{n_b}$, $\{\alpha_{(i)}^{(1)}\}_{i=1}^{n_b}$, and $\{\tilde{\alpha}_{(i)}\}_{i=1}^{n_b}$}

Randomly choose $\{x_{(i)}^{(0)}\}_{i=1}^{n_b'}$ from $\{x_i^{(0)}\}_{i=1}^{n_0}$ and $\{{x'}_{(i)}^{(0)}\}_{i=1}^{n_b'}$ from $\{x^{(0)}_i\}_{i=1}^{n_0}$; \\ % \setminus \{x_{(i)}^{(0)}\}_{i=1}^{n_b'}$; \\

Randomly choose $\{x_{(i)}^{(1)}\}_{i=1}^{n_b'}$ from $\{x_i^{(1)}\}_{i=1}^{n_1}$ and $\{{x'}_{(i)}^{(1)}\}_{i=1}^{n_b'}$ from $\{x^{(1)}_i\}_{i=1}^{n_1}$; \\ %\setminus \{x_{(i)}^{(1)}\}_{i=1}^{n_b'}$; \\

% Randomly choose $\{z_{(i)}\}_{i=1}^{n_b}$ and $\{z_{(i)}'\}_{i=1}^{n_b}$ from $p_{Z}$. \\
Generate $\{\alpha_{(i)}^{(1)}\}_{i=1}^{n_b'}\sim p_{\alpha}$, $\{\alpha_{(i)}^{(2)}\}_{i=1}^{n_b'}\sim {\rm Unif}(0,1)$, and $\{\nu_{(i)}\}_{i=1}^{n_b'}\sim {\rm Unif}(0,1)$. \\

Create $\{\check{x}_{(i)}^{(1)}=\alpha_{(i)}^{(1)} x_{(i)}^{(0)} + (1-\alpha_{(i)}^{(1)}) x_{(i)}^{(1)}\}_{i=1}^{n_b'}$ and $\{\check{x}_{(i)}^{(2)}=(1-\alpha_{(i)}^{(1)}) x_{(i)}^{(0)} + \alpha_{(i)}^{(1)} x_{(i)}^{(1)}\}_{i=1}^{n_b'}$; \\ 

Create $\{\hat{x}_{(i)}^{(1)}=\alpha_{(i)}^{(2)} {x'}_{(i)}^{(0)} + (1-\alpha_{(i)}^{(2)}) {x'}_{(i)}^{(1)}\}_{i=1}^{n_b'}$ and $\{\hat{x}_{(i)}^{(2)}=(1-\alpha_{(i)}^{(2)}) {x'}_{(i)}^{(0)} + \alpha_{(i)}^{(2)} {x'}_{(i)}^{(1)}\}_{i=1}^{n_b'}$; \\ 

Produce $\{q_{(i)}^{(1)}\}_{i=1}^{n_b}=\{\check{x}_{(i)}^{(1)}\}_{i=1}^{n_b'}\cup \{\check{x}_{(i)}^{(2)}\}_{i=1}^{n_b'}$, $\{q_{(i)}^{(2)}\}_{i=1}^{n_b}=\{\hat{x}_{(i)}^{(1)}\}_{i=1}^{n_b'}\cup \{\hat{x}_{(i)}^{(2)}\}_{i=1}^{n_b'}$, $\{\tilde{q}_{(i)}=\nu_{(i)} q_{(i)}^{(1)} + (1-\nu_{(i)}) q_{(i)}^{(2)}\}_{i=1}^{n_b}$, %$\{z_{(i)}^{(1)}=\alpha_{(i)}^{(1)} z_{(i)} + (1-\alpha_{(i)}^{(1)}) z_{(i)}'\}_{i=1}^{n_b}$, 
and $\{\tilde{\alpha}_{(i)}=\nu_{(i)}\alpha_{(i)}^{(1)} + (1-\nu_{(i)})\alpha_{(i)}^{(2)}\}_{i=1}^{n_b}$; 
\end{algorithm}

\begin{algorithm}[ht!]
\caption{Parallelly Tempered Generative Adversarial Nets}\label{alg:ptgan}
\KwData{$\{x_i\}_{i=1}^n$ be a set of training data set. Set the training iteration $T$ and for the inner loop $T'$, the minibatch size $n_b=m_b$, $t=0$, the penalty size $\lambda$, the ratio $r$, the learning rate $\gamma_D$ and $\gamma_G$, and initialize $\bw^{(0)}$ and $\bv^{(0)}$.}
\KwResult{$\bv^{(T)}$}
{\small 
\While{$t \leq T$}{
    Set $t' = 0$ and $t = t + 1$; \\ 
    /\% Create Minibatch (Algorithm~\ref{alg:minibatch} or~\ref{alg:fair-minibatch}) \%/ \\
    $\{q_{(i)}^{(1)}\}_{i=1}^{n_b}$, $\{q_{(i)}^{(2)}\}_{i=1}^{n_b}$, $\{\tilde{q}_{(i)}\}_{i=1}^{n_b}$, $\{\alpha_{(i)}^{(1)}\}_{i=1}^{n_b}$, $\{\tilde{\alpha}_{(i)}\}_{i=1}^{n_b}$~=~MC($\{x_i\}_{i=1}^{n}$); \\ 
    
    Generate $\{z_{(i)}\}_{i=1}^{n_b}$ from $p_{Z}$,
    
    \While{$t' < T'$}{
    $t' = t' + 1$; \\ 
    /\% Evaluate the loss and penalty \%/ \\
    $\hat{L}_b(\bw^{(t)},\bv^{(t)}) = \frac{1}{n_{b}}\sum_{i=1}^{n_{b}}D_{\bw^{(t)}}(q_{(i)}^{(1)},\alpha_{(i)}^{(1)})-D_{\bw^{(t)}}(G_{\bv^{(t)}}(z_{(i)},\alpha_{(i)}^{(1)}),\alpha_{(i)}^{(1)})$; \\ 

    $\hat{H}(\bw^{(t)})=\frac{1}{n_b}\sum_{i=1}^{n_b} \left(\nabla_{\tilde{q}_{i}}D_{\bw^{(t)}}(\tilde{q}_{i},\tilde{\alpha}_{(i)})\cdot (q^{(1)}_{(i)}-q^{(2)}_{(i)})\right)^2$;

    /\% Update $\Dt$ \%/ \\
    $\bw^{(t+1)} = \bw^{(t)} + \gamma_D \frac{\partial}{\partial \bw}\left(\hat{L}_b(\bw^{(t)},\bv^{(t)}) - \lambda \hat{H}(\bw^{(t)})\right)$;\\ 
    }

    $\hat{L}_b^G(\bw^{(t+1)},\bv^{(t)})=-\frac{1}{n_{b}}\sum_{i=1}^{n_{b}}D_{\bw^{(t+1)}}(G_{\bv^{(t)}}(z_{(i)},\alpha_{(i)}^{(1)}),\alpha_{(i)}^{(1)})$; \\

    /\% Update $\Gt$ \%/ \\
    $\bv^{(t+1)} = \bv^{(t)} - \gamma_G \frac{\partial}{\partial \bv} \hat{L}_b^G(\bw^{(t+1)},\bv^{(t)})$; \\ 
}}
\end{algorithm}

\subsection{Interpolated reference variables}
\label{supp:inter_noise}
To enhance the flexibility of $G$, we also observe the effects of using the interpolated reference noise $\Zal = \alpha Z_{i} + (1 - \alpha) Z_{j}$ with $Z_{i}, Z_{j} \sim p_Z$ (instead of using $Z\sim p_Z$ as the input of $G$). Especially when the generator is not sufficiently large to learn the complexity of $p_X$, the use of $\Zal$ may be helpful. 

Intuitively, the generator network can be viewed as a transport mapping between input reference noise and the target distribution. Since $(\Zal,\alpha)$ shares similarly convex interpolating structure as $(\Qal,\alpha)$, we expect that transport from $(\Zal,\alpha)$ to $(\Qal,\alpha)$ can be less complex than that from $(Z,\alpha)$ to $(\Qal,\alpha)$. To be more specific, our generator shall satisfy the following relationship $G(\cdot,\alpha)\overset{d}{=} \alpha G(\cdot,1) + (1-\alpha) G(\cdot,0)$ where $\cdot$ represents the respective network input noise. When $\Zal$ is used, the above identity reduces to the linearity property which may be easier to approximate. 
As an extreme example, we consider a linear generator $G(z,\alpha)=\beta z$ where samples of $X$ and $Z$ are from ${\cal X} = \{-1,1\}$ and ${\cal Z} = \{-1,1\}$. Then it is straightforward to see that $\alpha x_1 + (1-\alpha) x_2$ can be reproduced by $G(\alpha z_1 + (1-\alpha) z_2,\alpha)$ but not by $G(z,\alpha)$. 

Figure~\ref{fig:mixture_zalpha} compares the performance when using either $\Zal$ or $Z$ as generator input for the 8-component mixture example (Figure~\ref{fig:mode_collapse}). We compare two 2-depth (i.e., 1 hidden and 1 output layers) generators whose intermediate layer have either $N_2^G=4$ or $N_2^G=256$ under our PTGAN scheme (Algorithm~\ref{alg:ptgan}), where the reference noise $Z=(Z_1,Z_2)\in\mathbb{R}^2$ where $Z_1 \sim {\rm Unif}(-1,1)$ and $Z_2 \sim {\rm Unif}\{-1,1\}$. The figure illustrates the logarithm of the 1-Wasserstein distance between $p_{\Qal}$ and $p_{\Gt(\Zal,\alpha)}$ (or $p_{\Gt(Z,\alpha)}$) for specific $\alpha=0.5$ and $\alpha=0.9$ over the training iterations, showing that the use of $\Zal$ can enhance the performance of the generator. 
\begin{figure}[ht!]
    \centering
    \includegraphics[width=1.0\linewidth]{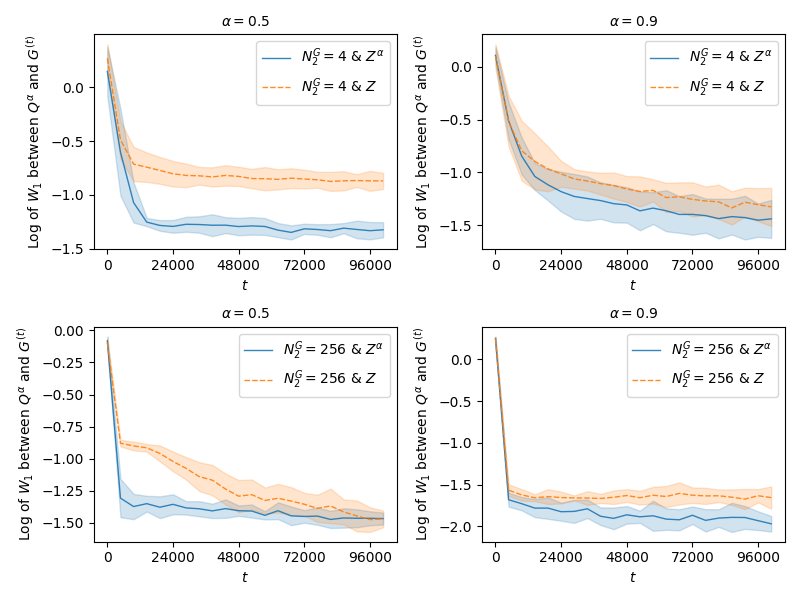}
    \caption{Plots illustrate the logarithm of the 1-Wasserstein distance between $p_{\Qal}$ and $p_{\Gt(\Zal,\alpha)}$ (or $p_{\Gt(Z,\alpha)}$) over training iterations. Each colored area stands for one standard deviation from the average line.}
    \label{fig:mixture_zalpha}
\end{figure}
This work, therefore, basically uses the interpolated reference variables for training of PTGAN and FairPTGAN. Notably, taking convex interpolation between reference variables does not affect our theoretical investigation in Section~\ref{sec:parallel}. To see more details, refer to SM~\ref{supp:proof}. In practice, however, if the generator and the reference noise is sufficiently flexible and large, the performance of $\Zal$ and $Z$ is similar.

\newpage

\section{Simulation}
\label{supp:simulation}

\subsection{Details of Figure~\ref{fig:value_variance}}
\label{appen:fig}

Proposition~\ref{prop:grad_lwbd} and Theorem~\ref{prop:grad_upbd1} imply that the size of weight matrices should be appropriately controlled to calculate the gradients' variance. The same ${\cal D}$ in drawing Figure~\ref{fig:mode_collapse} is used. In every iteration, therefore, $\Dt$ is trained with the penalty of \citeSupp{mesc:etal:18}. Since $\Gt$ is fixed to generate the left mode, it does not need to be updated. After the one-step update of $\Dt$, $\partial \hat{L}_b^i / \partial \bw$ is found for all $i=1,\dots,n_b$ where $\hat{L}_b^i=\Dt(X_i)-\Gt(Z_i)$, so that there are $n_b$ number of gradient for every single element in $\bw^{(t)}$. Then the variance is calculated elementwisely and summed up.

\subsection{Details in Section~\ref{sec:data_gen}}

\subsubsection{Image data generation}
\label{appen:image}

\paragraph{Evaluation metrics} Inception Score \citepSupp[IS,][]{sali:etal:16} and Fr\'echet Inception Distance \citepSupp[FID,][]{heus:etal:17} scores are calculated based on InceptionV3 \citepSupp{szeg:etal:15} which is a pre-trained classifier on {\bf ImageNet}. Let $p(y|x)$ be the classifier and denote by $\tilde{y}(x)$ the intermediate output after the global max-pooling layer in the classifier where labels and images match $y$ and $x$ respectively. The scores are calculated as follows: 
\begin{align*}
    \text{IS}&=\exp\left( \bE_{X\sim p_{G(Z)}}(\text{KL}(p(y|X)||p(y))) \right), \\
    \text{FID}&=\lVert \mu_X - \mu_G\rVert_2^2 + \text{tr}\left(\Sigma_X + \Sigma_{G} -2 \left(\Sigma_X^{\frac{1}{2}}\Sigma_{G} \Sigma_X^{\frac{1}{2}}\right)^{\frac{1}{2}}\right),
\end{align*}
where $\mu_X = \bE(\tilde{y}(X))$,  $\mu_G = \bE(\tilde{y}(G(Z)))$, $\Sigma_X = \cov(\tilde{y}(X))$, and $\Sigma_G = \cov(\tilde{y}(G(Z)))$. For IS, the lower the entropy of $p(y|G(Z))$ is, the higher fidelity images $G$ produces. The marginal probability $p(y)=\int p(y|G(z))p(G(z))p(z)dz$ having higher entropy implies more diversity. Therefore, as the discrepancy of the KL divergence increases, it can be said that the generator achieves higher performance on both the high quality of images and the diversity. On the one hand, FID measures the distance in the latent space under the assumption that the latent feature follows multivariate Gaussian distributions. 

\paragraph{Fine-tuning procedure to calculate IS/FID scores} The pre-trained InceptionV3 model is fine-tuned for {\bf BloodMnist} and {\bf CelebA-HQ} by updating the weight and bias parameters in later layers (after the 249th layer) of the model. After the global pooling layer, a dropout and a linear layer are placed whose size matches the output's dimension of each learning objective. For BloodMnist, the parameters are updated to minimize the cross-entropy based on eight different labels via the Adam optimizer and stopped by the early stopping process. Then, we calculate IS/FID scores using this fine-tuned Inception model. For CelebA-HQ, the model minimizes the sum of cross-entropy losses, each of which measures the discrepancy between a facial attribute and its corresponding probability. This is one simple way to execute multi-label learning to make the model figure out all facial attributes simultaneously. Other learning procedures are the same with BloodMnist. For CelebA-HQ, we report two FID scores where the first is based on the original pre-trained Inception model but the second on the fine-tuned model. Note that reporting the first-type FID is usual in the literature. 

\paragraph{Simulation setup for CIFAR10 and BloodMnist} The network architectures of $D$ and $G$ follow the CNN-based structure (Table~\ref{appen:cnn_spec}) used in the spectral normalization GAN training  \citepSupp{miya:etal:18}. A convolutional layer with $3\times 3$ kernel, 1 stride, and 64 filters is denoted as {[conv: $3\times 3$, 1, 64]}, and a deconvolutional layer is also written in the same way. For ours, the temperature $\alpha_i$ is concatenated to every hidden layer. For $p_x$, CIFAR10 and BloodMnist have 32 and 64. The total number of iterations $T$ is set to both 100k with minibatches having 100 data instances for each data set. Referring to \citeSupp{zhou:etal:19}, the Adam optimizer's hyperparameters \citepSupp{king:ba:14} are set to $\beta_1=0.0$ and $\beta_2=0.9$ with the learning rates for $D$ and $G$ as 0.0001. The spectral normalization layer (SN) is applied only to the original competitor \citepSupp{miya:etal:18}. The penalty parameters for $\lambda_{\text{MP}}$ \citeSupp[MP,][]{zhou:etal:19} and $\lambda_{\text{GP}}$ \citeSupp[GP,][]{gulr:etal:17} are specified as $\lambda_{\text{MP}}=1$ and $\lambda_{\text{GP}}=10$ by referring to their works. For PTGAN, the generator uses the interpolated uniform variables based on the discussion in \ref{supp:inter_noise} while competitors use the uniform distribution. IS/FID scores are measured at 10 different $t$ points that equally space the total number of iterations $T$ since the evaluation of IS and FID is computationally heavy. The best score is determined from those.  
\begin{table}[ht!]
\begin{subtable}[h]{0.45\textwidth}
\centering
\begin{tabular}{c}
\hline
\hline
$X~\in~\mathbb{R}^{p_x \times p_x \times 3}$ \\
\hline
{[conv: 3$\times 3$ , 1, 64]} (SN) lReLU(0.1) \\
{[conv: 4$\times 4$ , 2, 64]} (SN) lReLU(0.1) \\
\hline
{[conv: 3$\times 3$, 1, 128]}  (SN) lReLU(0.1) \\
{[conv: 4$\times 4$, 2, 128]} (SN) lReLU(0.1) \\
\hline
{[conv: 3$\times 3$, 1, 256]} (SN) lReLU(0.1) \\
{[conv: 4$\times 4$, 2, 256]} (SN) lReLU(0.1) \\
\hline
{[conv: 3$\times 3$, 1, 512]} (SN) lReLU(0.1) \\
\hline
dense $\rightarrow$ 1 \\ 
\hline
\hline
\end{tabular}
\caption{Critic}
\end{subtable}
\hfill
\begin{subtable}[h]{0.45\textwidth}
\centering
\begin{tabular}{c}
\hline
\hline
$Z~\in~\mathbb{R}^{128}$ \\
\hline
dense $\rightarrow M_Z \times M_Z \times 512$  \\
\hline
{[deconv: 4$\times 4$, 2, 256]} BN ReLU \\
\hline
{[deconv: 4$\times 4$, 2, 128]} BN ReLU \\
\hline
{[deconv: 4$\times 4$, 2, 64]} BN ReLU \\
\hline
{[deconv: 3$\times 3$, $S_G$, 3]} \\
\hline
reshape $p_x \times p_x \times 3$ \\ 
\hline
\hline
\end{tabular}
\caption{Generator}
\end{subtable}
\caption{Convolutional neural network structures for $D$ and $G$ in CIFAR10 $p_x=32$ and $M_Z=4$ with the stride $S_G=1$ and BloodMnist $p_x=64$ and $M_Z=4$ with $S_G = 2$.}
\label{appen:cnn_spec}
\end{table}

\paragraph{Simulation setup for CelebA-HQ} To adapt $256\times 256 \times 3$, we use a customized CNN-based structure shown in Table~\ref{appen:cnn_spec_celeba}. Considering our computational budget, we set $T$ as 70k and 50 size minibatch in every iteration. Other configurations follow the same setups used in the previous image experiments.
\begin{table}[ht!]
\begin{subtable}[h]{0.45\textwidth}
\centering
\begin{tabular}{c}
\hline
\hline
$X~\in~\mathbb{R}^{256 \times 256 \times 3}$ \\
\hline
{[conv: 4$\times 4$ , 2, 64]} lReLU(0.2) \\
\hline
{[conv: 4$\times 4$, 2, 128]} lReLU(0.2) \\
\hline
{[conv: 4$\times 4$, 2, 256]} lReLU(0.2) \\
\hline
{[conv: 4$\times 4$, 2, 512]} lReLU(0.2) \\
\hline
{[conv: 4$\times 4$, 2, 1024]} lReLU(0.2) \\
\hline
{[conv: 4$\times 4$, 2, 2048]} lReLU(0.2) \\
\hline
dense $\rightarrow$ 1 \\ 
\hline
\hline
\end{tabular}
\caption{Critic}
\end{subtable}
\hfill
\begin{subtable}[h]{0.45\textwidth}
\centering
\begin{tabular}{c}
\hline
\hline
$Z~\in~\mathbb{R}^{256}$ \\
\hline
dense $\rightarrow 4 \times 4 \times 512$  \\
\hline
{Up. [conv: 4$\times 4$, 1, 2048]} BN ReLU \\
\hline
{Up. [conv: 4$\times 4$, 1, 1024]} BN ReLU \\
\hline
{Up. [conv: 4$\times 4$, 1, 512]} BN ReLU \\
\hline
{Up. [conv: 3$\times 3$, 1, 256]} BN ReLU \\
\hline
{Up. [conv: 3$\times 3$, 1, 128]} BN ReLU \\
\hline
{Up. [conv: 3$\times 3$, 1, 3]} \\
\hline
reshape $p_x \times p_x \times 3$ \\ 
\hline
\hline
\end{tabular}
\caption{Generator}
\end{subtable}
\caption{Convolutional neural network structures for $D$ and $G$ in CelebA-HQ. Up. stands for an upsampling layer.}
\label{appen:cnn_spec_celeba}
\end{table}

\paragraph{Other GAN metrics} Our simulation study further considers the popular GAN metrics such as the Jensen-Shannon divergence \citepSupp[JSD,][]{good:etal:14} and the Pearson $\chi^2$ divergence  \citepSupp[PD,][]{mao:etal:17}. In our notations, their loss functions are written as: 
\begin{align*}
    \text{JSD} &= \sup_D \bE(\log D(X)) + \bE (\log (1-D(G(Z)))), \\
    \text{PD} &= \sup_D \dfrac{1}{2}{\bf E}_X\left((D(X)-1)^2\right) + \dfrac{1}{2}{\bf E}_Z\left(D(G(Z))^2\right).
\end{align*}
To see more details, refer to the original works. 

\paragraph{Penalty-based GAN training} The Lipschitz GAN \citepSupp{zhou:etal:19} uses the maximum penalty is defined as $\text{MP}= \lambda_{\text{MP}} \max_i \lVert \nabla_{\tilde{X}_i} D(\tilde{X}_i)\rVert^2$ where $\tilde{X}_i=\nu X_i + (1-\nu)G(Z_i)$ where $\nu$ is randomly drawn from ${\rm Unif}(0,1)$. The Wasserstein GAN with the gradient penalty \citepSupp{gulr:etal:17} uses $\text{GP}= \lambda_{\text{GP}} \bE ((\lVert \nabla_{\tilde{X}} D(\tilde{X}_i)\rVert-1)^2)$ where $\tilde{X}_i$ is the random interpolation as MP. 

\subsubsection{Additional results}
\label{supp:image_additional}
\paragraph{Different choice of hyperparameters} While we use the penalty parameter for MP and GP recommended in their papers, we find extra results with the different parameters of $\lambda_{\text{MP}}$ and $\lambda_{\text{GP}}$. Due to the limited computation resources, the results are only based on the neural distance, and they are shown in Table~\ref{simul:cifar10_extra}. Table~\ref{tab:impact_r} justifies the high value of $r$ because of the bias-variance trade-off. It highlights that the performance is worse when no interpolation points are used. 

\begin{table}[!ht]
\caption{Summary of IS/FID in CIFAR10 and BloodMnist for MP and GP. Standard deviations are averaged across 10 independent implementations. All values are rounded to the third decimal place.}
\label{simul:cifar10_extra}
\vspace{0.1in}
\centering
\footnotesize
\begin{tabular}{c|c|cc|cc}
\hline
&& \multicolumn{2}{c|}{CIFAR10} & \multicolumn{2}{c}{BloodMnist}  \\ 
\hline
$d_{\cal D}$ & Type & IS ($\uparrow$) & FID ($\downarrow$) & IS ($\uparrow$) & FID ($\downarrow$) \\ 
\hline
\hline
\multirow{4}{*}{ND}  
& MP $(\lambda_{\text{MP}}=10)$  & 6.833 (0.090) & 30.048 (0.979) & 4.998 (0.047) & 49.248 (1.007) \\
& MP $(\lambda_{\text{MP}}=100)$ & 6.722 (0.054) & 30.569 (0.457) & 4.939 (0.046) & 50.352 (2.076)  \\
& GP $(\lambda_{\text{GP}}=1)$   & 6.773 (0.145) & 29.903 (0.973) & 5.033 (0.039) &  50.035 (1.282) \\
& GP $(\lambda_{\text{GP}}=100)$ & 6.759 (0.090) & 29.545 (0.592) & 5.023 (0.034) &  48.841 (1.023) \\
\hline
\end{tabular}
%\end{small}
\end{table}

\begin{table}[ht!]
\caption{Comparison by differing the hyperparameter $r$ for the neural distance}
\label{tab:impact_r}
\footnotesize
\centering
\begin{tabular}{c|cc|cc}
\hline
    & \multicolumn{2}{c|}{CIFAR10} & \multicolumn{2}{c}{BloodMnist} \\
    \hline
    & $r=1.0$      & $r=0.9$      & $r=1.0$        & $r=0.9$        \\
\hline
\hline
IS  &    6.885 (0.145)    &   \bf{7.248 (0.067)}    &  4.839 (0.072)  &  \bf{5.071 (0.058)}   \\
FID &   28.551 (2.028)   &    \bf{25.087 (0.962)}   &   56.348 (3.201) & \bf{41.989 (0.897)}   \\
\hline
\end{tabular}
\end{table}

\paragraph{Visual evaluation} Figures~\ref{fig_supp:cifar10_real} and \ref{fig_supp:cifar10_gen} display the original images and generated images in CIFAR10; Figures~\ref{fig_supp:blood_real} and \ref{fig_supp:blood_gen} are for BloodMnist; and Figures~\ref{fig_supp:celeba_real} and \ref{fig_supp:celeba_gen} are for CelebA-HQ. We particularly draw the generated images of PTGAN, SNGAN, and Lipschitz GAN (LGAN) where all methods are trained under the neural distance (ND). From our view, PTGAN and Lipschitz GAN are both partially successful in producing recognizable pictures in CIFAR10 while SNGAN seems not. In BloodMnist, SNGAN even shows mode collapse, i.e., generating similarly looking blood cells. For CelebA-HQ, PTGAN and LGAN have a similar level of visual quality from human perspective but their performance is separated by the Inception model.  
We acknowledge that the CNN-based generator adopted in Figures~\ref{appen:cnn_spec} and \ref{appen:cnn_spec_celeba} may not be large enough to learn the semantic details of all the modalities and may also need many more training iterations with additional training tricks such as scheduling learning rate, weight decay, classifier guided generative modeling, doubling feature maps, etc. However, to adapt to our limited computational resources and also to see the pure effects of the proposed method, this work uses the relatively light network to conduct extensive comparisons.

\begin{figure*}[ht!]
    \centering
    \includegraphics[width=\textwidth]{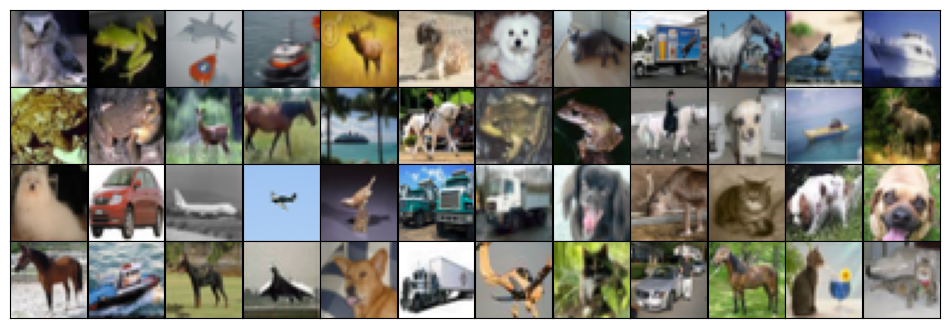}
    \caption{Randomly selected real images of CIFAR10}
\label{fig_supp:cifar10_real}
\end{figure*}

\begin{figure*}[ht!]
\centering
\begin{subfigure}[b]{1.00\textwidth}
    \includegraphics[width=\textwidth]{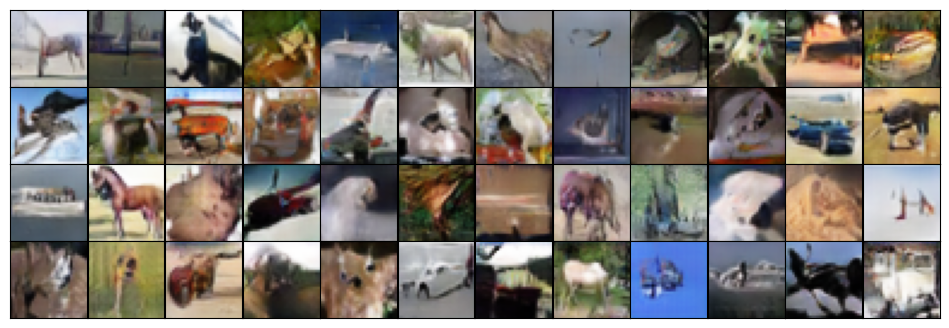}
    \caption{Results of PTGAN trained with ND metric}
\end{subfigure}
\begin{subfigure}[b]{1.00\textwidth}
    \includegraphics[width=\textwidth]{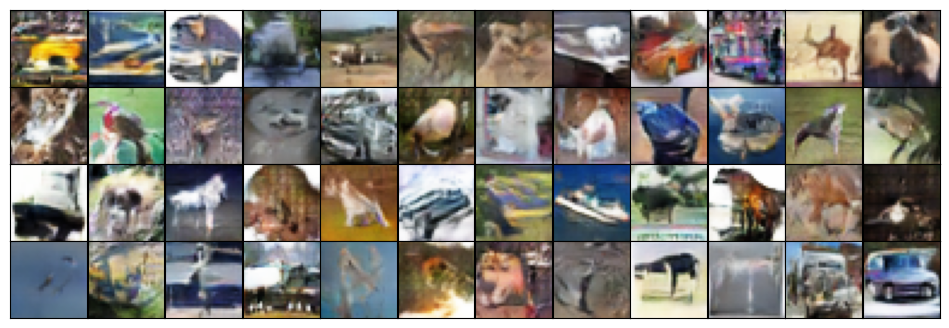}
    \caption{Results of Lipschitz GAN trained with ND metric}
\end{subfigure}
\begin{subfigure}[b]{1.00\textwidth}
    \includegraphics[width=\textwidth]{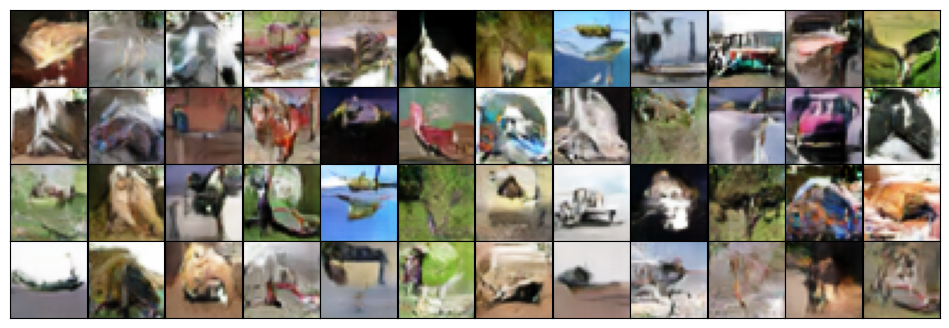}
    \caption{Results of SNGAN trained with ND metric}
\end{subfigure}
\caption{Randomly generated images for CIFAR10}
\label{fig_supp:cifar10_gen}
\end{figure*}

\begin{figure*}[ht!]
    \centering
    \includegraphics[width=\textwidth]{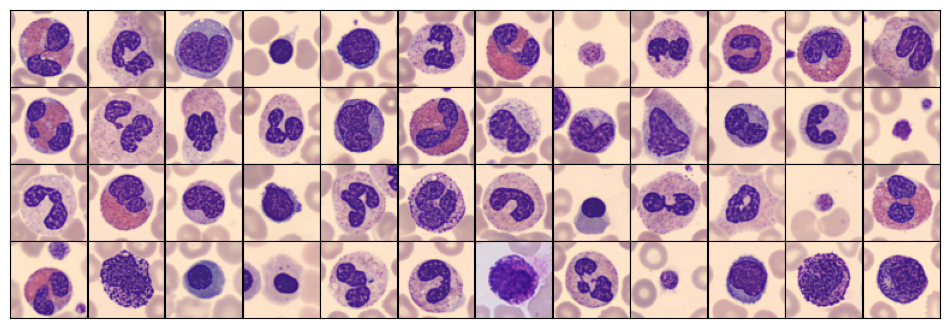}
    \caption{Randomly selected real images of BloodMnist}
    \label{fig_supp:blood_real}
\end{figure*}

\begin{figure*}[ht!]
\centering
\begin{subfigure}[b]{1.00\textwidth}
    \includegraphics[width=\textwidth]{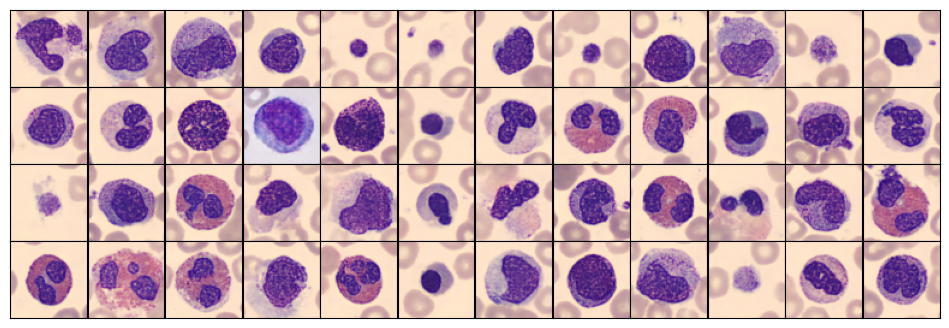}
    \caption{Results of PTGAN trained with ND metric}
\end{subfigure}
\begin{subfigure}[b]{1.00\textwidth}
    \includegraphics[width=\textwidth]{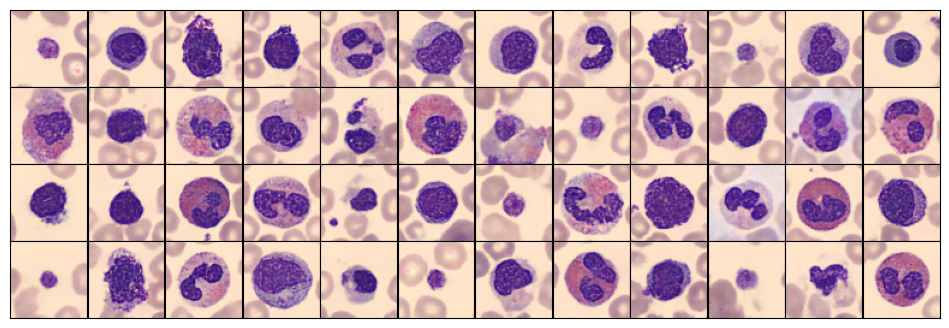}
    \caption{Results of Lipschitz GAN trained with ND metric}
\end{subfigure}
\begin{subfigure}[b]{1.00\textwidth}
    \includegraphics[width=\textwidth]{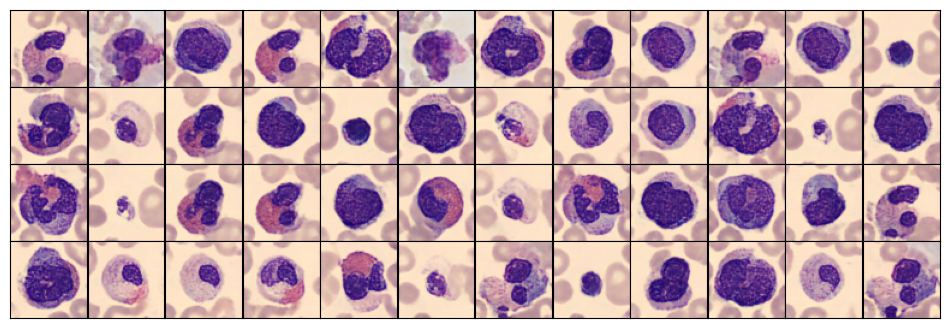}
    \caption{Results of SNGAN trained with ND metric}
\end{subfigure}
\caption{Randomly generated images for BloodMnist}
\label{fig_supp:blood_gen}
\end{figure*}

\begin{figure*}[ht!]
    \centering
    \includegraphics[width=\textwidth]{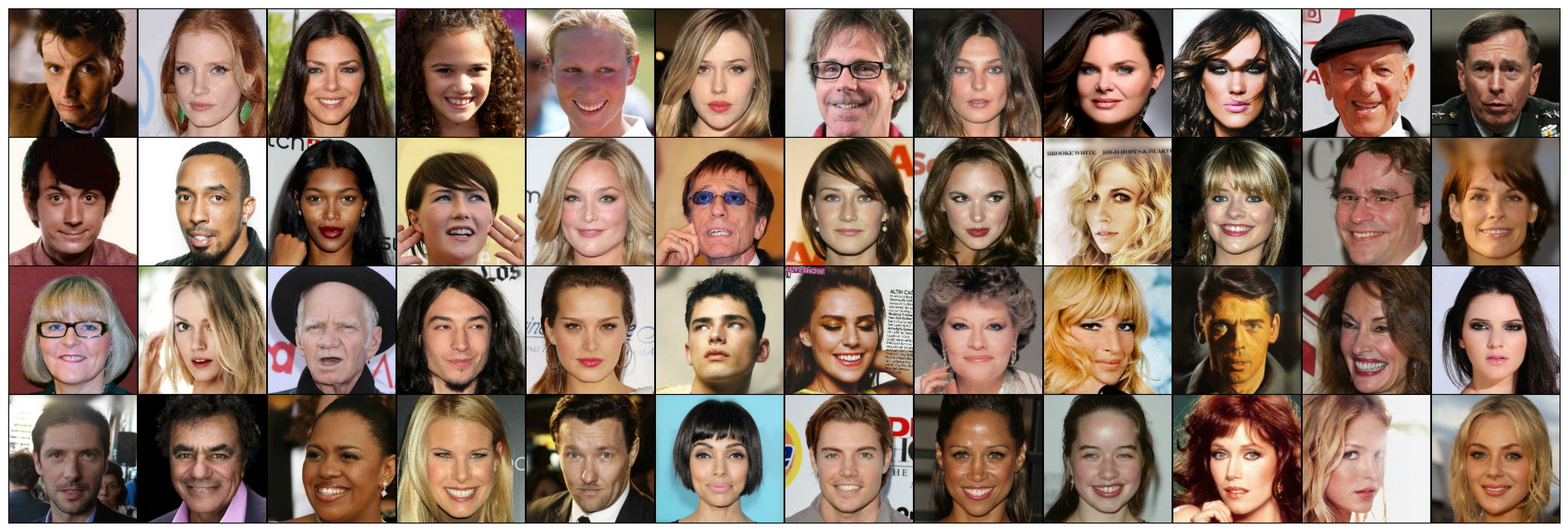}
    \caption{Randomly selected real images of CelebA-HQ}
    \label{fig_supp:celeba_real}
\end{figure*}

\begin{figure*}[ht!]
\centering
\begin{subfigure}[b]{1.00\textwidth}
    \includegraphics[width=\textwidth]{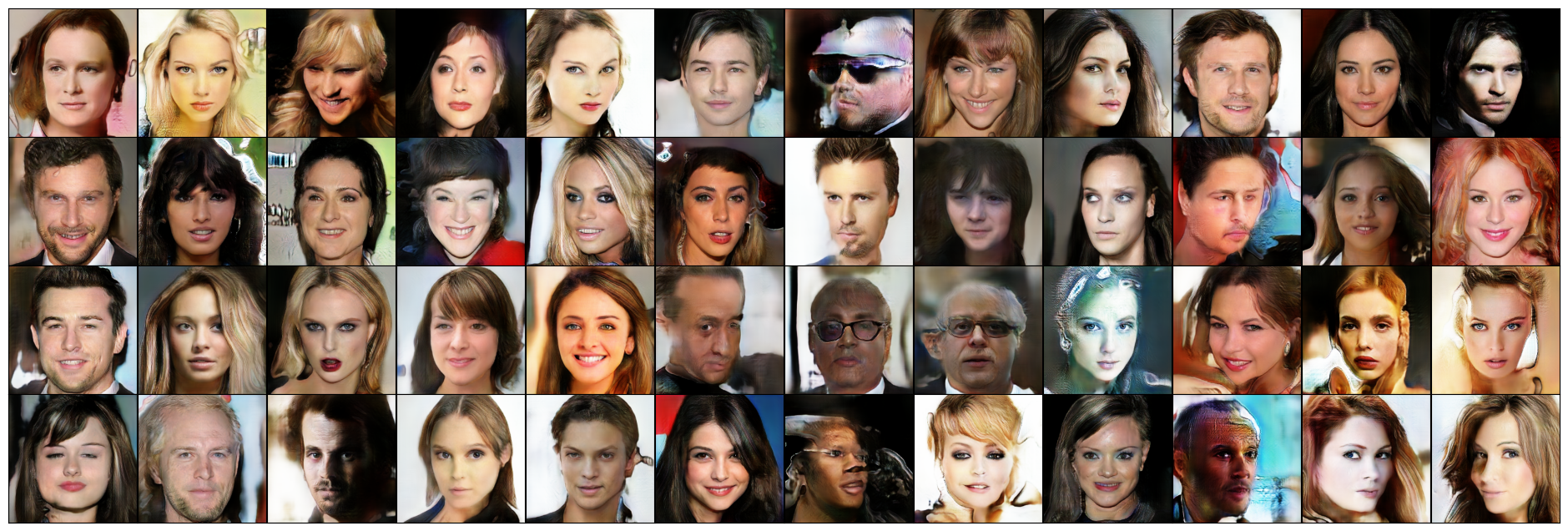}
    \caption{Results of PTGAN trained with ND metric}
\end{subfigure}
\begin{subfigure}[b]{1.00\textwidth}
    \includegraphics[width=\textwidth]{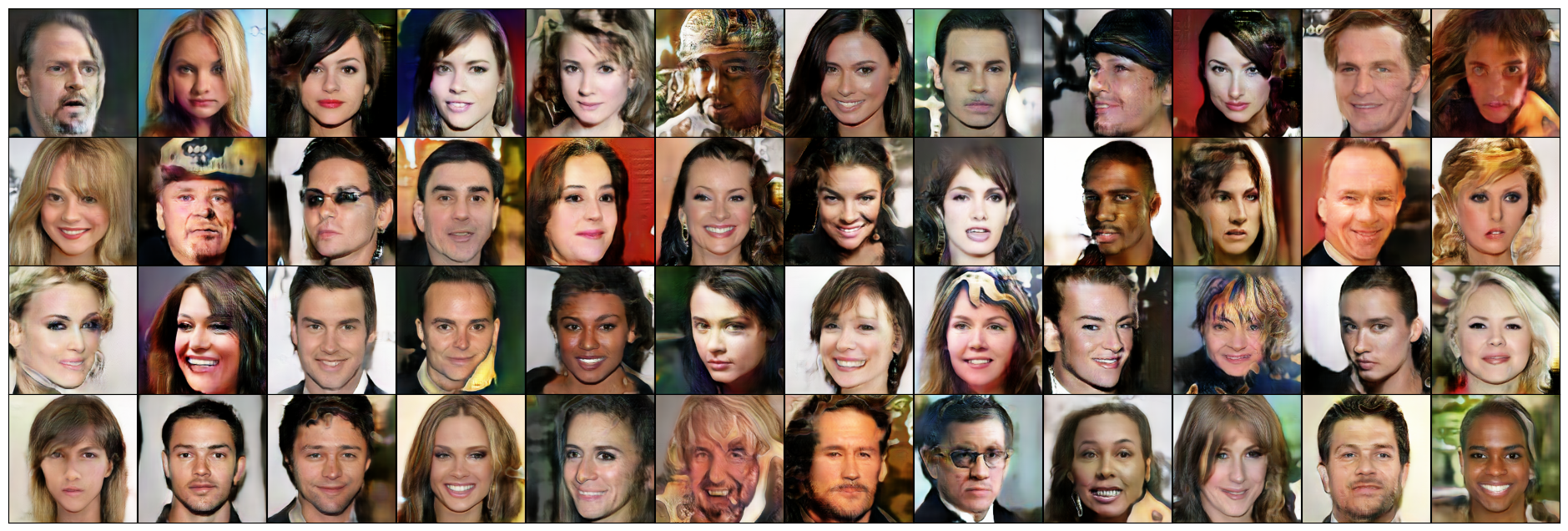}
    \caption{Results of Lipschitz GAN trained with ND metric}
\end{subfigure}
\caption{Randomly generated images for CelebA-HQ}
\label{fig_supp:celeba_gen}
\end{figure*}

\clearpage

\subsubsection{Tabular data generation}
\label{appen:tabgan}

\paragraph{Data description}  All models are tested on the following benchmark data sets: 
\begin{itemize}
    \item {\bf Adult} is for predicting whether an individual's annual income is greater than \$50K or not. The data consists of 32561 individuals with 15 variables, but we exclude `education' and `fnlwgt' by referring to the pre-processing step in \citeSupp{cho:etal:20}. 
    For more details about data, refer to \nolinkurl{https://archive.ics.uci.edu/ml/datasets/adult}.

    \item {\bf Law School Admission} consists of 124557 individuals with 15 variables. Considering the duplication of columns and rows, we select 'LSAT', 'GPA', 'Gender', 'Race', and 'resident'. The task is to predict whether an applicant receives admission. For more details about data, refer to \citeSupp{wigh:98}. 
    
    \item {\bf Credit Card Default} is for predicting whether or not a customer declares default. This data set includes 30000 individuals with 25 variables. We only drop `ID' in the simulation study. For more details about data, refer to \nolinkurl{https://archive.ics.uci.edu/dataset/350/default+of+credit+card+clients}.
\end{itemize}
In all data sets, continuous variables are scaled such that they are within $[-1,1]$. Discrete variables are transformed to one-hot encoding. 

\paragraph{Simulation setup} The network architectures of $D$ and $G$ are based on the dense layers as shown in Table~\ref{appen:dense_spec}. The generator consists of two parts to put different activation functions for continuous and discrete variables. The continuous variables are generated through [common]-[continuous] while each one-hot encoded discrete variable is individually generated through [common]-[discrete]. The final layer in [discrete] employs the Gumbel-softmax function \citepSupp{jang:etal:16} that enables the one-hot encoding procedure to be differentiable. The notations $d_{X}$, $d_{\text{continuous}}$, and $d_{\text{discrete}}$ are generic to denote the dimension of input space, the number of continuous variables, and the total number of discrete variables each of which is one-hot encoded. For each model, we implement 10 experiments with 200 epochs for Adult and Credit Card Default but 40 epochs for Law School Admission while they all have the 100 minibatch size. The Adam optimizer is set to be the same as used in the image generation tasks. For the data sets, the total number of iterations of $T$ is approximately 57k, 35k, and 53k for Adult, Law School Admission, and Credit Card Default. The evaluation of ${\mathtt S}_t$ is made at 50 equally spaced points in $\{0,\dots,T\}$.

\begin{table}[ht]
\begin{subtable}[h]{0.5\textwidth}
\centering
\begin{tabular}{c}
\hline
\hline
$X~\in~\mathbb{R}^{d_{X}}$ \\
\hline
dense 64 ReLU $\times$ 7 \\
\hline
dense $\rightarrow$ 1 \\ 
\hline
\hline
\end{tabular}
\caption{Critic}
\end{subtable}
\hfill
\begin{subtable}[h]{0.5\textwidth}
\centering
\begin{tabular}{c}
\hline
\hline
$Z~\in~\mathbb{R}^{16}$ \\
\hline
{[common]} dense 64 BN ReLU $\times$ 7 \\
\hline
{[continuous]} dense $d_{\text{continuous}}$ \\
\hline
{[discrete]} dense Gumbel-softmax $d_{\text{discrete}}$ \\
\hline
\hline
\end{tabular}
\caption{Generator}
\end{subtable}
\caption{Dense neural network structures for $D$ and $G$.}
\label{appen:dense_spec}
\end{table}

\paragraph{Additional results} We find further results of MP by differing the penalty parameter $\lambda_{\text{MP}}$. Table~\ref{simul:tabgan_mp} also shows that ours defeats the Lipschitz GAN model. For the consistent use of the parameter, the results of $\lambda_{\text{MP}}=1$ appear in the main text.

\begin{table}[ht!]
\caption{Summary of ${\mathtt S}_T$: all scores appearing below are the average of 10 replicated implementations. The standard deviation appears in the parenthesis.}
\label{simul:tabgan_full}
\vspace{0.1in}
\centering
\footnotesize
\begin{tabular}{c|c|c|ccc}
\hline
Data & $d_{\cal D}$ & Type & RF ($\downarrow$)& SVM ($\downarrow$)& LR ($\downarrow$)\\ 
\hline
\hline
\multirow{4}{*}{Adult}
& \multirow{2}{*}{JSD} & PT + CP & \bf{0.022 (0.004)} & \bf{0.037 (0.004)} & \bf{0.028 (0.003)} \\
&                        & MP & 0.059 (0.019) & 0.069 (0.022) & 0.058 (0.019) \\
& \multirow{2}{*}{PD} & PT + CP & \bf{0.023 (0.003)} & \bf{0.039 (0.007)} & \bf{0.026 (0.004)} \\
&                        & MP & 0.047 (0.021) & 0.054 (0.011) & 0.044 (0.011) \\
\hline
\multirow{4}{*}{Law School.}
& \multirow{2}{*}{JSD} & PT + CP & \bf{0.020 (0.014)} & \bf{0.023 (0.009)} & \bf{0.008 (0.006)} \\
&                        & MP & 0.093 (0.022) & 0.101 (0.022) & 0.068 (0.024) \\
& \multirow{2}{*}{PD} & PT + CP & \bf{0.019 (0.007)} & \bf{0.020 (0.004)} & \bf{0.006 (0.001)} \\
&                        & MP & 0.096 (0.017) & 0.099 (0.018) & 0.069 (0.016) \\
\hline
\multirow{4}{*}{Credit Card.}
& \multirow{2}{*}{JSD} & PT + CP& \bf{0.052 (0.009)} & \bf{0.061 (0.017)} & \bf{0.036 (0.008)} \\
&                        & MP   & 0.147 (0.021) & 0.164 (0.038) & 0.146 (0.030) \\
& \multirow{2}{*}{PD} & PT + CP & \bf{0.050 (0.009)} & \bf{0.046 (0.012)} & \bf{0.035 (0.010)} \\
&                        & MP & 0.126 (0.040) & 0.138 (0.047) & 0.122 (0.043) \\
\hline
\end{tabular}
%\end{small}
\end{table}

\begin{table}[ht!]
\caption{Summary of ${\mathtt S}_T$ of MP: all scores appearing below are the average of 10 replicated implementations. The standard deviation appears in the parenthesis.}
\label{simul:tabgan_mp}
\vspace{0.1in}
\centering
\footnotesize
\begin{tabular}{c|c|c|ccc}
\hline
Data & $d_{\cal D}$ & Type & RF & SVM & LR \\ 
\hline
\hline
\multirow{6}{*}{Adult}
& \multirow{2}{*}{JSD} & MP ($\lambda_{\text{MP}} =10$) & 0.028 (0.014) & 0.043 (0.019) & 0.034 (0.015) \\
&                        & MP ($\lambda_{\text{MP}} =100$) & 0.030 (0.016) & 0.044 (0.010) & 0.034 (0.012) \\
& \multirow{2}{*}{PD} & MP ($\lambda_{\text{MP}} =10$) & 0.043 (0.024) & 0.051 (0.018) & 0.041 (0.020) \\
&                        & MP ($\lambda_{\text{MP}} =100$) & 0.035 (0.025) & 0.045 (0.015) & 0.034 (0.013) \\
&\multirow{2}{*}{ND} & MP ($\lambda_{\text{MP}} =10$) & 0.043 (0.025) & 0.047 (0.016) & 0.039 (0.013)  \\
&                      & MP ($\lambda_{\text{MP}} =100$) & 0.025 (0.014) & 0.042 (0.016) & 0.032 (0.012)   \\
\hline
\multirow{6}{*}{Law School.}
& \multirow{2}{*}{JSD} & MP ($\lambda_{\text{MP}} =10$) & 0.092 (0.023) & 0.095 (0.024) & 0.063 (0.025) \\
&                        & MP ($\lambda_{\text{MP}} =100$) & 0.064 (0.032) & 0.065 (0.027) & 0.038 (0.026)  \\
& \multirow{2}{*}{PD} & MP ($\lambda_{\text{MP}} =10$) & 0.079 (0.023) & 0.080 (0.026) & 0.057 (0.029)  \\
&                        & MP ($\lambda_{\text{MP}} =100$) & 0.059 (0.035) & 0.060 (0.033) & 0.037 (0.023) \\
&\multirow{2}{*}{ND} & MP ($\lambda_{\text{MP}} =10$) & 0.079 (0.020) & 0.084 (0.022) & 0.056 (0.018)   \\
&                      & MP ($\lambda_{\text{MP}} =100$) & 0.063 (0.027) & 0.066 (0.030) & 0.039 (0.027)  \\
\hline
\multirow{6}{*}{Credit Card.}
& \multirow{2}{*}{JSD} & MP ($\lambda_{\text{MP}} =10$) & 0.121 (0.051) & 0.126 (0.055) & 0.113 (0.058) \\
&                        & MP ($\lambda_{\text{MP}} =100$) & 0.134 (0.041) & 0.153 (0.035) & 0.132 (0.042) \\
& \multirow{2}{*}{PD} & MP ($\lambda_{\text{MP}} =10$) & 0.121 (0.057) & 0.127 (0.058) & 0.114 (0.061) \\
&                        & MP ($\lambda_{\text{MP}} =100$) & 0.147 (0.020) & 0.170 (0.023) & 0.154 (0.031) \\
&\multirow{2}{*}{ND} & MP ($\lambda_{\text{MP}} =10$) & 0.128 (0.045) & 0.136 (0.042) & 0.121 (0.040)   \\
&                      & MP ($\lambda_{\text{MP}} =100$) & 0.150 (0.021) & 0.174 (0.039) & 0.150 (0.023)  \\
\hline
\end{tabular}
%\end{small}
\end{table}

\subsection{Details in Section~\ref{sec:tab_fairgen}}
\label{appen:tab_fairgan}

\paragraph{Evaluation metric} A Pareto frontier is a set of solutions that are not dominated by other pairs. For example, $(0.7,0.7)$, a pair of AUC and SP, is dominated by $(0.8,0.4)$ but not by $(0.6,0.6)$. To see more details, refer to \citeSupp{emme:etal:18}.

\paragraph{Implementation of FairPTGAN} The proposed FairPTGAN model first yields minibatches from Algorithm~\ref{alg:fair-minibatch} and then implements Algorithm~\ref{alg:ptgan} to learn $D$ and $G$. 

\paragraph{FairWGANGP and GeoRepair} \citeSupp{raja:etal:22} suggested two-step learning procedure: 1) training $\Gt$ up to $T$ iteration using WGANGP \citepSupp{gulr:etal:17} and then 2) regularizing $G^{(T+l)}(Z)$, for $l=1,\dots, T'$,  with the fairness penalty formulated as $\lambda_f |\bE(\tilde{Y}|\tilde{A}=1)-\bE(\tilde{Y}|\tilde{A}=0)|$ where $(\tilde{C},\tilde{A},\tilde{Y}) \sim G^{(T+l)}(Z)$. Thus, $\lambda_f$ controls the trade-off, and it is set to $\lambda_f=10$ by referring to \citeSupp{raja:etal:22}. In \citeSupp{feld:etal:15}, the authors proposed the geometric repair that transforms a univariate covariate $c$ to $(1-\lambda_p) F_a^{-1}(q) + \lambda_p F_A^{-1}(q)$ where $F_a(x)$ is the conditional cumulative distribution of $c$ given $a\in \{0,1\}$ and $F_A^{-1}(q)=\text{median}_{a\in \{0,1\}}F_a^{-1}(q)$ with $q=F_a(c)$. In our study, this pre-processing step is applied to the FairPTGAN model with $\alpha=1$ with 5 equally spaced $\lambda_p \in [0,1]$ considered. 

\paragraph{Simulation setup} The study particularly considers Adult and Law School Admission data sets showing evident discrimination impact on prediction tasks. For Adult, the ``race" variable is specified as a sensitive attribute that is binarized to be white and non-white. Similarly in Law School Admission, the ``White" variable is used as a sensitive attribute while ``Race" is dropped. For a fair comparison, the total number of iterations for both FairPTGAN and FairWGANGP is specified as $T=100$k but FairWGANGP has extra $T/2$ iterations for its second training phase with  $\lambda_f=10$. As mentioned, GeoRepair is implemented to the produced data set by FairPTGAN models with $\alpha=1$. For PTGAN, $r$ is set to 0.2. In all cases, the minibatch size is specified as 200. Other configurations are the same with Section~\ref{appen:tabgan}. To draw smooth Pareto-frontier curves, each run produces 20 independent data sets with the last iterate of the generator, i.e., $G^{(100k)}$ for FairPTGAN and $G^{(150k)}$ for FairWGANGP. Thus, 200 independent sets from the 10 independent runs are used to draw the results.

\paragraph{Additional results} Similar to Table~\ref{tab:fair_comp} in the main text, we draw Table~\ref{tab:fair_comp2} with different thresholds. It is noteworthy that FairPTGAN captures smoother trade-off curves than the two competitors. GeoRepair and FairWGANGP in Table~\ref{tab:fair_comp2} have the same scores, especially in LR with Table~\ref{tab:fair_comp}. 

\begin{table}[ht!]
\caption{Averages of the 10 smallest SP scores whose AUCs are greater than the thresholds ($\geq 0.70$ for Adult and $\geq 0.70$ for Law School). Standard deviations are in the parentheses next to the averages.}
\label{tab:fair_comp2}
\footnotesize
\centering
\begin{tabular}{c|c|ccc}
\hline
Data & Model & RF ($\downarrow$)& SVM ($\downarrow$)& LR ($\downarrow$) \\ 
\hline
\hline
\multirow{3}{*}{Adult}  
& \multirow{1}{*}{FairPTGAN} & \bf{0.008 (0.004)} & \bf{0.015 (0.009)} & \bf{0.058 (0.012)} \\
& \multirow{1}{*}{FairWGANGP}& 0.051 (0.009) & 0.075 (0.006) & 0.080 (0.005) \\
&\multirow{1}{*}{GeoRepair}  & 0.069 (0.007) & 0.039 (0.019) & 0.098 (0.012) \\
\hline
\multirow{3}{*}{Law School.} 
& \multirow{1}{*}{FairPTGAN}  & \bf{0.111 (0.018)} & \bf{0.107 (0.008)} & \bf{0.137 (0.014)} \\
& \multirow{1}{*}{FairWGANGP} & 0.147 (0.007) & 0.120 (0.005) & 0.175 (0.003) \\
&\multirow{1}{*}{GeoRepair}   & 0.119 (0.019) & 0.144 (0.004) & 0.182 (0.003) \\
\hline
\end{tabular}
%\end{small}
\end{table}